%% file: committee.tex
\documentclass[a4paper,11pt]{article}

\usepackage[utf8]{inputenc}
\usepackage[british]{babel}
\usepackage[mono=false]{libertine}
\usepackage[T1]{fontenc}
\usepackage{amsthm}
\usepackage[top=2.5cm, bottom=2.5cm, left=2cm, right=2cm]{geometry}
\usepackage[font=small]{caption}
\usepackage{enumitem}
\usepackage{titlesec}
\usepackage[usenames, dvipsnames]{xcolor}
\makeatletter
\newcommand{\setappendix}{Appendix~\thesection:~~}
\newcommand{\setsection}{\thesection~~}
\titleformat{\section}{\bfseries\LARGE}{%
	\ifnum\pdfstrcmp{\@currenvir}{appendices}=0
	\setappendix
	\else
	\setsection
\fi}{0em}{}
\makeatother
\usepackage[titletoc]{appendix}
\usepackage[pdftex]{graphicx}
\usepackage{epstopdf}
\usepackage{array}
\usepackage{url}
\usepackage{mathtools}
\usepackage{amssymb,amsfonts,amsmath}
\usepackage{dsfont}
\usepackage{stmaryrd}
\usepackage{bm}
\usepackage{footnote}

\usepackage[utf8]{inputenc} 
\usepackage[T1]{fontenc}    
\usepackage{booktabs}       
\usepackage{nicefrac}       
\usepackage{microtype}      
\usepackage{tikz}
\usepackage{subcaption}
\usepackage{algorithmic}
\usepackage{algorithm}
\usepackage{wrapfig}
\usepackage{xr}
\usepackage{xr-hyper}
\usepackage{braket}

\newcommand*{\QED }{\hfill\ensuremath{\square}}

\usepackage[pdfpagemode=UseNone,bookmarksopen=false,colorlinks=true,urlcolor=RoyalBlue,citecolor=magenta,linkcolor=RoyalBlue]{hyperref}
\usepackage{bookmark}

\input{./settings_custom.tex}
\usepackage{setspace}
\usepackage{pgf}
\pgfmathsetseed{1}
\usetikzlibrary{arrows}
\usetikzlibrary{decorations.markings}
\input{tikzlibrarybayesnet.code}

\tikzset{middlearrow/.style={decoration={markings,mark= at position 0.5 with {\arrow{#1}} ,},postaction={decorate}}}
\usetikzlibrary{arrows}

\DeclareTextCommandDefault{\nobreakspace}{\leavevmode\nobreak\ } 

\begin{document}
\title{The committee machine: Computational to statistical gaps \\in
  learning a two-layers neural network}

\author{Benjamin Aubin$^{\star\dagger}$, Antoine Maillard$^{\dagger}$, Jean Barbier$^{\Diamond}$, \\Florent Krzakala$^{\dagger}$, Nicolas Macris$^{\otimes}$ and Lenka Zdeborov{\'a}$^{\star}$}
\date{}
\maketitle
{\let\thefootnote\relax\footnote{
\!\!\!\!\!\!\!\!\!\!$\star$ Institut de Physique Th\'eorique, CNRS \& CEA \& Universit\'e Paris-Saclay, Saclay, France.\\
$\dagger$ Laboratoire de Physique Statistique,
CNRS \& Sorbonnes Universit\'es \& \'Ecole Normale Sup\'erieure, PSL University, Paris, France.\\
$\otimes$ Laboratoire de Th\'eorie des Communications, \'Ecole Polytechnique F\'ed\'erale de Lausanne, Suisse. \\
$\Diamond$ International Center for Theoretical Physics, Trieste, Italy.\\
}}
\setcounter{footnote}{0}

\begin{abstract}
  Heuristic tools from statistical physics have been used in the past 
  to locate the phase transitions and compute the optimal learning and generalization errors in the
  teacher-student scenario in multi-layer neural networks.  
  In this contribution, we provide a rigorous justification of these approaches
   for a two-layers neural network model called the committee machine, under a technical assumption. 
  We also introduce a version of the approximate
  message passing (AMP) algorithm for the committee machine that allows to perform
  optimal learning in polynomial time for a large set of
  parameters. 
  We find that there are regimes in which a low
  generalization error is information-theoretically achievable while
  the AMP algorithm fails to deliver it; strongly suggesting that no
  efficient algorithm exists for those cases, and unveiling
  a large computational gap.
\end{abstract}
{\singlespacing
\hypersetup{linkcolor=black}
\tableofcontents}
\newpage
\input{sections/main/intro.tex}
\input{sections/main/modelandresults.tex}

\input{sections/main/2neurons.tex}

\input{sections/main/moreneurons.tex}

\input{sections/main/proof_new_last.tex}

\input{sections/main/conclusion.tex}

\section*{Acknowledgments}
This work has been supported by the ERC under the European Union’s FP7
Grant Agreement 307087-SPARCS and the European Union's Horizon 2020
Research and Innovation Program 714608-SMiLe, as well as by the French
Agence Nationale de la Recherche under grant ANR-17-CE23-0023-01
PAIL and the Swiss National Foundation grant no 200021E-175541. Additional funding is acknowledged by A.M., F.K. and J.B. from “Chaire de
recherche sur les modèles et sciences des données”, Fondation CFM pour
la Recherche-ENS. We also acknowledge L\'eo Miolane for discussions.

\newpage
\bibliographystyle{unsrt_abbvr}
\bibliography{refs}


\newpage

\appendix
\section*{Supplementary material}
\input{sections/supplementary/details-main-thm.tex}

\input{sections/supplementary/replicas.tex}
\input{sections/supplementary/generalization_error.tex}
\input{sections/supplementary/large_K.tex}
\input{sections/supplementary/linear_network.tex}

\input{sections/supplementary/AMP_derivation.tex}

\input{sections/supplementary/parity_machine.tex}
\end{document}

%% file: settings_custom.tex
\def \({\left(}
\def \){\right)}
\def \[{\left[}
\def \]{\right]}
\newcommand{\nn}{\nonumber \\}
\newcommand{\tbf}[1]{{\textbf{#1}}}

\newcommand{\mcO}{{\mathcal{O}}}

\newcommand{\bh}{{\textbf {h}}}

\newcommand{\cC}{{\mathcal{C}}}

\newcommand{\sign}{\text{ sign}}

\newcommand{\be}{\begin{equation}}
\newcommand{\ee}{\end{equation}}
\newcommand{\beqa}{\begin{eqnarray}}
\newcommand{\eeqa}{\end{eqnarray}}

\newcommand\smallO{
  \mathchoice
    {{\scriptstyle\mathcal{O}}}
    {{\scriptstyle\mathcal{O}}}
    {{\scriptscriptstyle\mathcal{O}}}
    {\scalebox{.7}{$\scriptscriptstyle\mathcal{O}$}}
  }
\newcommand{\bea}{\begin{align}}
\newcommand{\eea}{\end{align}}

\newtheorem{theorem}{Theorem}[section]
\newtheorem{assumption}{Assumption}
\newtheorem{lemma}[theorem]{\textbf{Lemma}}
\newtheorem{thm}[theorem]{\textbf{Theorem}}
\newtheorem{remark}[theorem]{\textbf{Remark}}
\newtheorem{proposition}[theorem]{\textbf{Proposition}}

\DeclareMathAlphabet{\varmathbb}{U}{bbold}{m}{n}
\newcommand{\id}{\mathds{1}}
\newcommand{\EE}{\mathbb{E}}
\newcommand{\bbR}{\mathbb{R}}

\newcommand{\bbN}{\mathbb{N}}

\newcommand{\td}[1]{{\tilde{#1}}}
\newcommand{\mO}{\mathcal{O}}
\newcommand{\underlim}[2]{\underset{#1 \to #2}{\longrightarrow}}

%% file: tikzlibrarybayesnet.code.tex
\usetikzlibrary{shapes}
\usetikzlibrary{fit}
\usetikzlibrary{chains}
\usetikzlibrary{arrows}

\tikzstyle{latent} = [circle,fill=white,draw=black,inner sep=1pt,
minimum size=20pt, font=\fontsize{10}{10}\selectfont, node distance=1]
\tikzstyle{obs} = [latent,fill=gray!25]
\tikzstyle{const} = [rectangle, inner sep=0pt, node distance=1]
\tikzstyle{factor} = [rectangle, fill=black,minimum size=5pt, inner
sep=0pt, node distance=0.4]
\tikzstyle{det} = [latent, diamond]

\tikzstyle{plate} = [draw, rectangle, rounded corners, fit=#1]
\tikzstyle{wrap} = [inner sep=0pt, fit=#1]
\tikzstyle{gate} = [draw, rectangle, dashed, fit=#1]

\tikzstyle{caption} = [font=\footnotesize, node distance=0] %
\tikzstyle{plate caption} = [caption, node distance=0, inner sep=0pt,
below left=5pt and 0pt of #1.south east] %
\tikzstyle{factor caption} = [caption] %
\tikzstyle{every label} += [caption] %

\newcommand{\edge}[3][]{ %
  \foreach \x in {#2} { %
    \foreach \y in {#3} { %
      \path (\x) edge [->, >={triangle 45}, #1] (\y) ;%
    } ;
  } ;
}



%% file: sections/main/intro.tex
\section{Introduction}
While the traditional approach to learning and generalization follows the Vapnik-Chervonenkis \cite{vapnik1998statistical} and Rademacher \cite{bartlett2002rademacher} worst-case type bounds, there has been a considerable body of theoretical work on calculating the generalization ability of neural networks for data arising from a probabilistic model within the framework of statistical mechanics \cite{seung1992statistical,watkin1993statistical,monasson1995learning,monasson1995weight,engel2001statistical}. In the wake of the need to understand the effectiveness of neural networks and also the limitations of the classical approaches \cite{zhang2016understanding}, it is of interest to revisit the results that have emerged thanks to the physics perspective. This direction is currently experiencing a strong revival, see e.g. \cite{chaudhari2016entropy,martin2017rethinking,barbier2017phase,baity2018comparing}.

Of particular interest is the so-called teacher-student approach, where labels are generated by feeding i.i.d.\ random samples to a neural network architecture (the {\it teacher}) and are then presented to another neural network (the {\it student}) that is trained using these data. Early studies computed the information theoretic limitations of the supervised learning abilities of the teacher weights by a student who is given $m$ independent $n$-dimensional examples with $\alpha\!\equiv\!m/n\!=\!\Theta(1)$ and $n\to \infty$ \cite{seung1992statistical,watkin1993statistical,engel2001statistical}. These works relied on non-rigorous heuristic approaches, such as the replica and cavity methods \cite{mezard1987spin,mezard2009information}. 
Additionally, no provably efficient algorithm was provided to achieve the predicted learning abilities, and it was thus difficult to test those predictions, or to assess the computational difficulty.

Recent developments in statistical estimation and information theory ---in particular of approximate message passing algorithms (AMP) \cite{donoho2009message,rangan2011generalized,bayati2011dynamics,javanmard2013state}, and a rigorous proof of the replica formula for the optimal generalization error \cite{barbier2017phase}--- allowed to settle these two missing points for single-layer neural networks (i.e. without any hidden variables). In the present paper, we leverage on these works, and provide rigorous asymptotic predictions and corresponding message passing algorithm for a class of two-layers networks.

\section{Summary of contributions and related works}
While our results hold for a rather large class of non-linear activation functions, we illustrate our findings on a case considered most commonly in the early literature: the committee machine.  This is possibly the simplest version of a two-layers neural network where all the weights in the second layer are fixed to unity, and we illustrate it in Fig.~\ref{fig:committee}. Denoting $Y_\mu$ the label associated with a $n$-dimensional sample $X_\mu$, and $W_{il}^*$ the weight connecting the $i$-th coordinate of the input to the $l$-th node of the hidden layer, it is defined by:
\begin{equation}
Y_\mu = {\rm{\sign}}\Big[\sum_{l=1}^K {\sign} \Big( \sum_{i=1}^nX_{\mu i}  W_{i l}^* \Big) \Big]\,. \label{model:com}
\end{equation}
We concentrate here on the teacher-student scenario: The teacher generates i.i.d.\ data samples with i.i.d.\ standard Gaussian coordinates $X_{\mu i} \sim\mathcal{N}(0,1)$, then she/he generates the associated labels $Y_\mu$ using a committee machine as in \eqref{model:com}, with i.i.d.\ weights $W_{il}^*$ unknown to the student (in the proof section we will consider the more general case of a distribution for the weights of the form $\prod_{i=1}^n P_0(\{W_{il}^*\}_{l=1}^K)$, but in practice we consider the fully separable case). The student is then given the $m$ input-output pairs $(X_{\mu},Y_\mu)_{\mu=1}^m$ and knows the distribution  $P_0$ used to generate $W_{il}^*$. The goal of the student is to learn the weights $W_{il}^*$ from the available examples  $(X_{\mu},Y_\mu)_{\mu=1}^m$ in order to reach the smallest possible generalization error (i.e. to be able to predict the label the teacher would generate for a new sample not present in the training set).  

There have been several studies of this model within the non-rigorous statistical physics approach in the limit where $\alpha\equiv m/n=\Theta(1)$, $K=\Theta(1)$ and $n\to \infty$~\cite{schwarze1993learning,schwarze1992generalization,schwarze1993generalization,MatoParga92,monasson1995weight,engel2001statistical}. A particularly interesting result in the teacher-student setting is the {\it specialization of hidden neurons} (see sec. 12.6 of \cite{engel2001statistical}, or \cite{saad1995line} in the context of online learning): For $\alpha<\alpha_{\rm spec}$ (where $\alpha_{\rm spec}$ is a certain critical value of the sample complexity), the permutation symmetry between hidden neurons remains conserved even after an optimal learning, and the learned weights of each of the hidden neurons are identical. For $\alpha>\alpha_{\rm spec}$, however, this symmetry gets broken as each of the hidden units correlates strongly with one of the hidden units of the teacher. Another remarkable result is the calculation of the optimal generalization error as a function of $\alpha$.

Our first contribution consists in a proof of the replica formula conjectured in the statistical physics literature, using the adaptive interpolation method of \cite{BarbierM17a,barbier2017phase}, that allows to put several of these results on a rigorous basis.
This proof uses a technical unproven assumption. 
Our second contribution is the design of an AMP-type of algorithm that is able to achieve the optimal generalization error in the above limit of large dimensions for a wide range of parameters. 
The study of AMP ---that is widely believed to be optimal between all polynomial algorithms in the above setting \cite{donoho2013accurate,REVIEWFLOANDLENKA,deshpande2015finding,bandeira2018notes}--- unveils,
in the case of the committee machine with a large number of hidden neurons, the existence a large {\it hard phase} in which learning is information-theoretically possible, leading to a good generalization error decaying asymptotically as $1.25 K/\alpha$ (in the $\alpha = \Theta(K)$ regime), 
but where AMP fails and provides only a poor generalization that does not go to zero when increasing $\alpha$. This strongly suggests that no efficient algorithm exists in this hard region and therefore there is a computational gap in learning such neural networks. 
In other problems where a hard phase was identified its study boosted the development of algorithms that are able to match the predicted thresholds, and we anticipate this will translate to the present model.

{\color{black} We also want to comment on a related line of work that studies the loss-function landscape of neural networks. While a range of works show under various assumptions that spurious local minima are absent in neural networks, others show under different conditions that they do exist, see e.g.~\cite{safran2017spurious}. 
The regime of parameters that is hard for AMP must have spurious local minima, but the converse is not true in general. It might be that there are spurious local minima, yet 
the AMP approach succeeds. Moreover, in all previously studied models in the Bayes-optimal setting the (generalization) error obtained with the AMP is the best 
known and other approaches, e.g. (noisy) gradient based, spectral algorithms or semidefinite programming, are not better 
in generalizing even in cases where the “student” models are overparametrized. Of course in order to be in the Bayes-optimal setting one needs to know the model used by the teacher which is not the case in practice. }

\begin{figure}
\centering
\begin{tikzpicture}[scale=1.5]
    \tikzstyle{factor}=[rectangle,minimum size=4pt,draw=black, fill opacity=1.]
    \tikzstyle{latent}=[circle,minimum size=19pt,draw=black, fill opacity=1.,fill=white]
    \tikzstyle{output}=[circle,minimum size=19pt,draw=black, fill opacity=1.,fill=white]
        \tikzstyle{noise}=[circle,minimum size=18pt,draw=black, fill opacity=1.,fill=white]
    \tikzstyle{output_y}=[rectangle,minimum size=15pt,draw=black, fill opacity=1.,fill=white]
    \tikzstyle{annot} = [text width=3cm, text centered]
    \tikzstyle{annot} = [text width=3cm, text centered]
    \tikzstyle{annotLarge} = [text width=5cm, text centered]
    \def\NX{15}
    \def\NK{2}
    \def\NY{1}
    \def\middle{0}
    \foreach \i in {1,...,\NX}
     	  \pgfmathparse{0.9*rnd+0.3}
          \definecolor{MyColor}{rgb}{\pgfmathresult,\pgfmathresult,\pgfmathresult}
    	\node[factor, fill=MyColor] (X-\i) at (\middle 0, 0.24*0.5*\NX+0.12 - 0.24*\i ) {}; 
    \foreach \k in {1,...,\NK}
    	\node[latent] (K-\k) at (2,\middle  \k - 1.5) {}; 
    \foreach \y in {1,...,\NY}
    	\node[output] (Y-\y) at (3.5,0) {}; 	
    \node[output_y] (Y-2) at (4.5,0) {}; 
    
    \foreach \i in {1,...,\NX}
    	\foreach \k in {1,...,\NK}
    		\path[-] (X-\i) edge (K-\k);
   	\foreach \k in {1,...,\NK}
    	\foreach \y in {1,...,\NY}
    		\path[-] (K-\k) edge (Y-\y);
    \path[-] (Y-1) edge (Y-2);
    \node[annotLarge] at (-1.5,0) {$(X_{\mu i})_{\mu,i=1}^{m,n}$ \\ samples};
    \node[annot] at (1.2,-1.6) {$W^*_{il} \in \bbR^{n \times K}$}  ;
    \node[annot] at (4.525,0) {${Y_{\mu}}$}  ;
    \node[annot] at (2.9,-0.8) {$W^{(2)}\in\bbR^{K}$}  ;
    \node[annot] at (2,-0.45) {\small $ f^{(1)}$}  ;
    \node[annot] at (2,0.55) {\small $ f^{(1)}$}  ;
    \node[annot] at (3.5,0.05) {\small $ f^{(2)}$}  ;
    \node[annot] at (0,2.2) {\small $n$ features}  ;
    \node[annot] at (2,1.3) {\small $K$ hidden\\ units}  ;
    \node[annot] at (4.5,0.7) {\small output}  ;
	\end{tikzpicture}	
	\caption{The \emph{committee machine} is one of the simplest models belonging to the considered model class \eqref{model}, and on which we focus to illustrate our results. It is a two-layers neural network with activation sign functions $f^{(1)},f^{(2)}=\sign$ and weights $W^{(2)}$ fixed to unity. It is represented for $K=2$.}
	\label{fig:committee}
\end{figure}
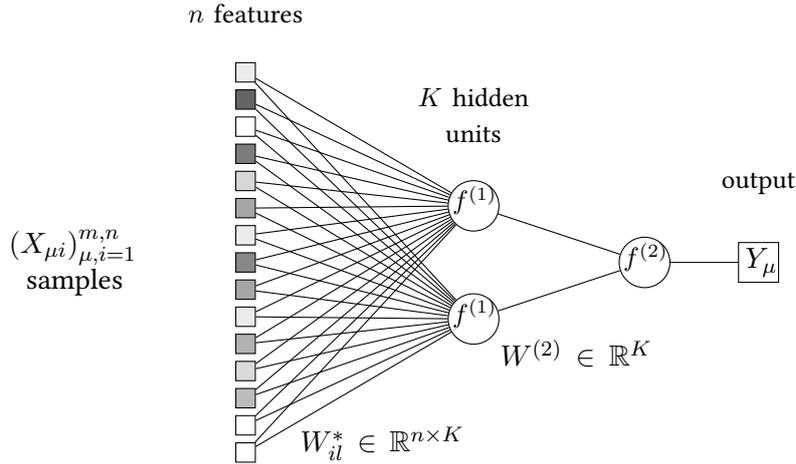

%% file: sections/main/modelandresults.tex
\section{Main technical results} \label{part2}

\subsection{A general model}
While in the illustration of our results we shall focus on the model (\ref{model:com}), all our formulas are valid for a broader class of models: Given $m$ input samples $(X_{\mu i})_{\mu,i=1}^{m,n}$, we denote $W_{il}^* $ the teacher-weight connecting the $i$-th input (i.e. visible unit) to the $l$-th node of the hidden layer. For a generic function $\varphi_{\rm out}: \bbR^K \times \bbR \to \bbR$ one can formally write the output as
\begin{equation}\label{model}
Y_\mu = \varphi_{\rm out} \Big(\Big\{\frac{1}{\sqrt n}\sum_{i=1}^n X_{\mu i} W_{i l}^* \Big\}_{l=1}^K, A_\mu \Big)~~~~~~\text{or}~~~~~~ Y_\mu \sim P_{\rm out}\Big( \cdot\Big| \Big\{\frac{1}{\sqrt n}\sum_{i=1}^n X_{\mu i} W_{i l}^* \Big\}_{l=1}^K \Big)\, ,
\end{equation}
where $(A_\mu)_{\mu=1}^m$ are i.i.d.\ real valued random variables with
known distribution $P_A$, that form the probabilistic part of the model, generally accounting for noise. 

For deterministic models the second argument is simply absent (or is a Dirac mass). We can view alternatively \eqref{model} as a channel where the transition kernel $P_{\rm out}$ is directly related to $\varphi_{\rm out}$. As discussed above, we focus on the teacher-student scenario where the teacher generates Gaussian i.i.d.\ data $X_{\mu i}\sim{\cal N}(0,1)$, and i.i.d.\ weights $W_{il}^* \sim P_0$. The student then learns $W^* $ from the data $(X_{\mu},Y_\mu)_{\mu=1}^m$ by computing marginal means of the posterior probability distribution (\ref{posterior-measure}).

Different scenarii fit into this general framework. Among those, the committee machine is obtained when choosing $\varphi_{\rm out}({h})={\rm{\sign}}(\sum_{l=1}^K {\sign} ( h_l) )$ while another model considered previously is given by the parity machine, when $\varphi_{\rm out}({h})=\prod_{l=1}^K {\sign} ( h_l)$, see e.g.  \cite{engel2001statistical} and sec.~\ref{sec:parity} for the numerical results in the case $K=2$. A number of layers beyond two has also been considered, see \cite{MatoParga92}. Other activation functions can be used, and many more problems can be described, e.g. compressed pooling \cite{alaoui2016decoding,el2017decoding} or multi-vector compressed sensing \cite{zhu2017performance}. 

\subsection{Two auxiliary inference problems} 
Denote $\mathcal{S}_K$ the finite dimensional vector space of $K\times K$ matrices,
$\mathcal{S}_K^+$ the convex set of
semi-definite positive $K\times K$ matrices, $\mathcal{S}_K^{++}$ for
positive definite $K\times K$ matrices, and
$\forall \,N \in \mathcal{S}_K^+$ we set
$S_K^{+}(N) \equiv \{M \in S_K^+ \text{ s.t. } N-M \in \mathcal{S}_K^+\}$. Note that ${\cal S}_K^+(N)$ is convex and compact.

Stating our results requires introducing two simpler auxiliary $K$-dimensional estimation problems: \\
$\bullet$ The first one consists in retrieving a $K$-dimensional input vector $W_0\sim P_0$ from the output of a Gaussian vector
channel with $K$-dimensional observations $$Y_0= r^{1/2} W_0 + Z_0\,,$$ $Z_0 \sim \mathcal{N}(0, I_{K\times K})$ and the ``channel gain''
matrix $r\in \mathcal{S}_K^+$. The posterior distribution
on $w= (w_l)_{l=1}^K$ is 
\begin{align}
 P(w \vert Y_0)=\frac{1}{{\cal Z}_{P_0}}P_0(w) e^{Y_0^\intercal r^{1/2} w - \frac{1}{2} w^\intercal r w}\, ,
\label{aux-model-1}
\end{align}
and the associated {\it free entropy} (or minus {\it free energy}) is given by the expectation over $Y_0$ of the log-partition function $$\psi_{P_0}(r) \equiv \mathbb{E}\ln {\cal Z}_{P_0}$$ and involves $K$ dimensional integrals. \\
$\bullet$ The second problem considers $K$-dimensional i.i.d.\ vectors $V, U^* \sim \mathcal{N}(0, I_{K\times K})$ where $V$ is considered to be known and one has to retrieve $U^* $ from a scalar observation obtained as $$\widetilde Y_0 \sim P_{\rm out}(\,\cdot\,| q^{1/2} V + (\rho-q)^{1/2}U^* )$$ where the second moment 
matrix $\rho \equiv \mathbb{E}[W_0 W_0^\intercal]$ is in $\mathcal{S}_K^+$ (where $W_0\sim P_0$) and the so-called ``overlap matrix''
$q$ is in $S_K^{+}(\rho)$. The associated posterior is 
\begin{align}
 P(u \vert \widetilde Y_0, V) = \frac{1}{{\cal Z}_{P_{\rm out}}}\frac{e^{-\frac{1}{2}u^\intercal u}}{{(2\pi)^{K/2}}}P_{\rm out}\big(\widetilde{Y}_0 | q^{1/2} V + (\rho - q)^{1/2}u\big)\, ,
\label{aux-model-2}
\end{align}
and the free entropy reads this time $$\Psi_{P_{\rm out}}(q;\rho) \equiv \mathbb{E} \ln {\cal Z}_{P_{\rm out}}$$ (with the expectation over $\widetilde Y_0$ and $V$) and also involves $K$ dimensional integrals.

\subsection{The free entropy}
 The central object of study leading to the optimal learning and generalization errors in the present setting is the posterior distribution of the weights:
\begin{align}
 P(\{w_{il}\}_{i,l=1}^{n, K}  \mid  \{X_{\mu i},Y_\mu\}_{\mu,i=1}^{m,n})  =   \frac{1}{{\cal Z}_n}\prod_{i=1}^n  P_0(\{w_{il}\}_{l=1}^K) \prod_{\mu=1}^m  P_{\rm out}\Big(Y_\mu\Big|\Big\{\frac{1}{\sqrt n}\sum_{i=1}^n X_{\mu i} w_{i l}\Big\}_{l=1}^K\Big)\,,
\label{posterior-measure}
\end{align}
where the normalization factor is nothing else than a {\it partition function}, i.e. the integral of the numerator over $\{w_{il}\}_{i,l=1}^{n, K}$. 
The expected\footnote{The symbol $\mathbb{E}$ will generally denote an expectation over all random variables in the ensuing expression (here $\{X_{\mu i}, Y_\mu\}$). 
Subscripts will be used 
only when we take partial expectations or if there is an ambiguity. 
} free entropy is by definition
\begin{align}
	 f_n \equiv  \frac1n\mathbb{E}\ln {\cal Z}_n \,. \label{freeent}
\end{align}
The replica formula gives an explicit (conjectural) expression of
$f_n$ in the high-dimensional limit $n,m\to \infty$ with $\alpha = m/n$
fixed. We show in sec.~\ref{sec:replicacomputation} how the heuristic replica
method \cite{mezard1987spin,mezard2009information} yields the
formula. This computation was first performed, to the best of our knowledge, by \cite{schwarze1993learning} in the case of the committee machine. 
Our first contribution is a rigorous proof of the corresponding free entropy formula using an interpolation method \cite{guerra2003broken,talagrand2003spin,BarbierM17a}, under a technical Assumption~\ref{assumption}.

In order to formulate our results, we add an (arbitrarily small) Gaussian regularization noise $Z_\mu\sqrt{\Delta}$ to the first expression of the model \eqref{model}, where $\Delta>0$, $Z_\mu\sim\mathcal{N}(0,1)$, which thus becomes
\begin{align}\label{modelnoise}
Y_\mu = \varphi_{\rm out} \Big(\Big\{\frac{1}{\sqrt n}\sum_{i=1}^n X_{\mu i} W_{i l}^* \Big\}_{l=1}^K, A_\mu \Big)+Z_\mu\sqrt{\Delta}\,,	
\end{align}
so that the channel kernel is ($u\in \mathbb{R}^K$)
\begin{align}\label{new-kernel}
 P_{\rm out}(y |u) = \frac{1}{\sqrt{2\pi\Delta}} \int_{\mathbb{R}} dP_A(a) e^{-\frac{1}{2\Delta}(y -\varphi_{\rm out}(u, a))^2}\,.
\end{align}
Let us define the {\it replica symmetric (RS) potential} as 
\begin{align}\label{RSpot}
f_{\rm RS}(q, r)=f_{\rm RS}(q, r;\rho)\equiv \psi_{P_0}(r) +\alpha\Psi_{P_{\rm out}}(q;\rho) -\frac{1}{2}{\rm Tr} (r q),
\end{align}
where $\alpha\equiv m/n$, and $\Psi_{P_{\rm out}}(q ;\rho)$ and $\psi_{P_0}(r)$
are the free entropies of the two simpler $K$-dimensional estimation problems \eqref{aux-model-1} and \eqref{aux-model-2}.

All along this paper, we assume the following hypotheses for our rigorous statements:

\begin{enumerate}[label=(H\arabic*),noitemsep]
	\item \label{hyp:1} The prior $P_0$ has bounded support in $\mathbb{R}^K$.
	\item \label{hyp:2} The activation $\varphi_{\rm out}: \mathbb{R}^K \times\mathbb{R} \to \mathbb{R}$ is a bounded $\cC^2$ function with bounded first and second derivatives w.r.t.\ its first argument (in $\mathbb{R}^K$-space).
	\item \label{hyp:3} For all $\mu=1, \ldots, m$ and $i=1,\ldots, n$ we have i.i.d.\ $X_{\mu i} \sim {\cal N}(0,1)$. 
\end{enumerate}

We finally rely on a technical hypothesis, stated as Assumption~\ref{assumption} in section \ref{subsec:technical-assumption}.

\begin{thm}[Replica formula]\label{main-thm}
 Suppose \ref{hyp:1}, \ref{hyp:2} and \ref{hyp:3}, and Assumption~\ref{assumption}\footnote{Since the publication of this work the adaptive interpolation method used in this paper has been improved for finite-rank models and can now circumvent this artificial hypothesis, see \cite{barbier2020information} and \cite{reeves2020information}.}. Then for the model
 \eqref{modelnoise} with kernel \eqref{new-kernel} the limit of the free entropy is:
\begin{align}\label{repl-1}
\lim_{n\to \infty}f_n \equiv \lim_{n\to \infty}\frac1n \mathbb{E}\ln {\cal Z}_n= {\adjustlimits \sup_{r\in \mathcal{S}^{+}_K} \inf_{q\in\mathcal{S}_K^+(\rho)}} f_{\rm RS}(q, r)\,.
\end{align}
\end{thm}
This theorem extends the recent progress for generalized linear models of \cite{barbier2017phase}, which includes the case $K=1$ of the present contribution, to the phenomenologically richer case of two-layers problems such as the committee machine. The proof sketch based on an {\it adaptive interpolation method} recently developed in \cite{BarbierM17a} is outlined in sec.~\ref{sec:proofsketch} and the details can be found in sec.~\ref{smproof}. 
\begin{remark}[Relaxing the hypotheses]
Note that, following similar approximation arguments as in \cite{barbier2017phase}, the hypothesis \ref{hyp:1} can be relaxed to the existence of the second moment of the prior; thus covering the Gaussian case, \ref{hyp:2} can be dropped (and thus include model \eqref{model:com} and its $\sign(\cdot)$ activation) and \ref{hyp:3} extended to data matrices $X$ with i.i.d.\ entries of zero mean, unit variance and finite third moment. Moreover, the case $\Delta=0$ can be considered when the outputs are discrete, as in the committee machine \eqref{model:com}, see \cite{barbier2017phase}. The channel kernel becomes in this case $P_{\rm out}(y |u) =  \int dP_A(a) \mathbf{1}(y -\varphi_{\rm out}(u, a))$ and the replica formula is the limit $\Delta\to0$ of the one provided in Theorem~\ref{main-thm}. In general this regularizing noise is needed for the free entropy limit to exist.	
\end{remark}

\subsection{Learning the teacher weights and optimal generalization error} 
A classical result in Bayesian estimation is that the estimator $\hat W$ that minimizes the mean-square error with the ground-truth $W^* $ is given by the expected mean of the posterior distribution. Denoting $q^* $ the extremizer in the replica formula (\ref{repl-1}), we expect from the replica method that in the limit $n\to\infty, m/n=\alpha$, and with high probability, $\hat W^\intercal  W^* /n  \to  q^* $. We refer to proposition~\ref{concentration} and to the proof in sec.~\ref{smproof} for the precise statement, that remains rigorously valid {\it only} in the presence of an additional (possibly infinitesimal) side-information. 
From the overlap matrix $q^*$, one can compute the Bayes-optimal generalization error when the student tries to classify a new, yet unseen, sample $X_{\rm new}$. The estimator of the new label $\hat Y_{\rm new}$ that minimizes the mean-square error with the true label is given by computing the posterior mean of $\varphi_{\rm out}(X_{\rm new}  w)$ ($X_{\rm new}$ is a row vector). Given the new sample, the optimal generalization error is then 
\begin{align}
 \frac{1}{2} \EE_{X,W^*} \left[\left(\EE_{w|X,Y} \big[ \varphi_{\rm out}(X_{\rm new}  w)\big] - \varphi_{\rm out}(X_{\rm new} W^* )\right)^2\right]\xrightarrow[n \to \infty]{}\epsilon_g(q^*),
\end{align}
where $w$ is distributed according to the posterior measure \eqref{posterior-measure} (note that this Bayes-optimal computation differs from the so-called Gibbs estimator by a factor $2$, see sec.~\ref{sec:generalization}). In particular, when the data $X$ is drawn from the standard Gaussian distribution on $\bbR^{m\times n}$, and is thus rotationally invariant, it follows that this error only depends on $ w^\intercal  W^*/n $, which converges to $q^* $. Then a direct algebraic computation gives a lengthy but explicit formula for $\epsilon_g(q^*)$, as shown in sec.~\ref{sec:generalization}.

\subsection{Approximate message passing, and its state evolution} 
Our next result is based on an adaptation of a popular algorithm to solve random instances of generalized linear models, the \emph{Approximate Message Passing} (AMP) algorithm \cite{donoho2009message,rangan2011generalized}, for the case of the committee machine and models described by \eqref{model}. 

The AMP algorithm can be obtained as a Taylor expansion of loopy belief-propagation (see sec.~\ref{sec:AMP}) and also originates in earlier statistical physics works \cite{thouless1977solution,mezard1989space,opper1996mean,Kaba,Baldassi26062007,REVIEWFLOANDLENKA}. It is conjectured to perform the best among all polynomial algorithms in the framework of these models. It thus gives us a tool to evaluate both the intrinsic algorithmic hardness of the learning and the performance of existing algorithms with respect to the optimal one in this model. 

\input{sections/main/AMP.tex}

%% file: sections/main/AMP.tex
\begin{algorithm}
\caption{Approximate Message Passing for the committee machine\label{alg:AMP}}  
\begin{algorithmic}
    \STATE {\bfseries Input:} vector $Y \in \bbR^m$ and matrix $X\in \bbR^{m \times n}$:
    \STATE \emph{Initialize}: $g_{\rm out,\mu} = 0, \Sigma_i = I_{K\times K} $ for $ 1 \leq i \leq n $ and $ 1 \leq \mu \leq m $ at $t=0$.
    \STATE \emph{Initialize}: $\hat{W}_i \in \bbR^K$ and $\hat{C}_i$, $\partial_{\omega} g_{\rm out , \mu}$ $\in \mathcal{S}_K^+$ for $ 1 \leq i \leq n $ and $ 1 \leq \mu \leq m $ at $t=1$.
    \REPEAT   
    \STATE Update of the mean $\omega_{\mu} \in \bbR^K$ and covariance $V_{\mu}\in \mathcal{S}_K^+$: \\
    \hspace{0.5cm} $\omega_{\mu}^t = \sum\limits_{i = 1}^n \big(\frac{X_{\mu
      i}}{\sqrt{n}}\hat{W}_{i}^t -    \frac{X_{\mu
      i}^2}{n}
    \left(\Sigma_{i}^{t-1}\right)^{-1}\hat{C}_{i}^t \Sigma_{i
    }^{t-1}g_{\rm out,\mu}^{t-1} \big)    \hspace{0.5cm}|\hspace{0.5cm}  V_{\mu}^t = \sum\limits_{i=1}^n\frac{X_{\mu
      i}^2}{n} \hat{C}_{i}^t $\vspace{0.1cm}
    \STATE Update of $g_{\rm out, \mu} \in \bbR^K$ and $\partial_{\omega} g_{\rm out , \mu} \in \mathcal{S}_K^+$: \\
    \hspace{0.5cm}$g_{\rm out, \mu}^t = g_{\rm out} (\omega_{\mu}^t , Y_{\mu}, V_{\mu}^t) \hspace{0.5cm}|\hspace{0.5cm}  \partial_{\omega} g_{\rm out, \mu}^t = \partial_{\omega}  g_{\rm out} (\omega_{\mu}^t , Y_{\mu}, V_{\mu}^t)  $ \vspace{0.1cm}
    \STATE Update of the mean $T_i \in \bbR^K$ and covariance $\Sigma_i \in \mathcal{S}_K^+$:\\
    \hspace{0.5cm}$T_i^t = \Sigma_{i}^t \Big(  \sum\limits_{\mu =1}^m
      \frac{X_{\mu
      i}}{\sqrt{n}}g_{\rm out,\mu}^t  -\frac{X_{\mu
      i}^2}{n}  \partial_{\omega} g_{\rm out , \mu}^t \hat{W}_{i}^t \Big) \hspace{0.5cm}|\hspace{0.5cm}  \Sigma_{i}^t = -\Big(\sum\limits_{\mu =1}^m \frac{X_{\mu
      i}^2}{n}  \partial_\omega g_{\rm out,\mu}^t \Big)^{-1} $\vspace{0.1cm}
    \STATE Update of the estimated marginals $\hat{W}_i \in \bbR^K$ and $\hat{C}_i \in \mathcal{S}_K^+$: \\
    \hspace{0.5cm}$\hat{W}_i^{t+1} = f_w( \Sigma_i^t , T_i^t )   \hspace{0.5cm}|\hspace{0.5cm}  \hat{C}_i^{t+1} = f_c( \Sigma_i^t , T_i^t )$\vspace{0.1cm}
    \STATE ${t} = {t} + 1$ 
    \UNTIL{Convergence on
    $\hat{W}$, $\hat{C}$.} 
    \STATE {\bfseries Output:}
    $\hat{W}$ and $\hat{C}$.
\end{algorithmic}
\end{algorithm}

The AMP algorithm is summarized by its pseudo-code in Algorithm~\ref{alg:AMP}, where the update
functions $g_{\rm out}$, $\partial_{\omega}g_{\rm out}$, $f_w$ and
$f_c$ are related, again, to the two auxiliary problems
(\ref{aux-model-1}) and (\ref{aux-model-2}). The functions $f_w(\Sigma,T)$
and $f_c(\Sigma,T)$ are respectively the mean and variance under the posterior distribution \eqref{aux-model-1} when $r \to \Sigma^{-1}$ and $Y_0 \to \Sigma^{1/2}T
$, while  $g_{\rm
  out}(\omega_{\mu},Y_{\mu},V_{\mu})$  is given by the product of $V_{\mu}^{-1/2}$ and the mean of $u$ under the posterior \eqref{aux-model-2} using
$\widetilde Y_0 \to Y_{\mu} $, $\rho-q \to V_{\mu}$ and $q^{1/2}V \to \omega_{\mu}$
 (see sec.~\ref{sec:AMP} for more details). After
convergence, $\hat W$ estimates the weights of the teacher-neural
network. The label of a sample $X_{\rm new}$ not seen in the training
set is estimated by the AMP algorithm as 
\begin{align}
      Y^t_{\rm new} = \int dy \big(\prod_{l=1}^K dz_l\big) \, y\,
  P_{\rm out}(y|\{z_l\}_{l=1}^K) {\cal N}(z; \omega_{\rm new}^t ,
  V_{\rm new}^t)\, ,    \label{AMP_gen}
\end{align}
where $\omega_{\rm new}^t = \sum_{i=1}^n X_{{\rm new},  i} \hat W_i^t$ is
the mean of the normally distributed variable $z\in {\mathbb R}^K$,
and $V_{\rm new}^t=\rho-q_{\rm AMP}^t$ is the $K\times K$ covariance
matrix (see below for the definition of $q_{\rm AMP}^t$). We provide a demonstration code of the algorithm on \href{https://github.com/benjaminaubin/TheCommitteeMachine}{GitHub}  \cite{GitHub_AMP_Aubin}.

AMP is particularly interesting because its performance can be tracked rigorously, again in the asymptotic limit when $n \to
\infty$, via a procedure known as state evolution (a rigorous version
of the cavity method in physics \cite{mezard2009information}), see \cite{javanmard2013state}. State evolution tracks the value of the overlap between
the hidden ground truth $W^*$ and the AMP estimate $\hat W^t$, defined as $q_{\rm AMP}^t\equiv\lim_{n\to\infty}({\hat W^t})^\intercal   W^*/n$, via the iteration of the following equations:
\begin{equation}
\label{main:StateEvolution}
	q_{\rm AMP}^{t+1} = 2 \nabla\psi_{P_0}(r_{\rm AMP}^t) \, ,	 \hspace{1cm}
	r_{\rm AMP}^{t+1} = 2 \alpha \nabla\Psi_{P_{\rm out}}(q_{\rm AMP}^t;\rho)\, .
\end{equation}
See sec.~\ref{sec:se} for more details and note that the fixed points of these equations correspond to the critical points
of the replica free entropy (\ref{repl-1}).

{\color{black} Let us comment further on the convergence of the
  algorithm. In the large $n$ limit, and if the integrals are
  performed without errors, then the algorithm is guaranteed to converge. This is a consequence of the state evolution combined with
the Bayes-optimal setting. In practice, of course, $n$ is finite and integrals are
approximated. In that case convergence is not guaranteed, but is
robustly achieved in all the cases presented in this paper.  
We also expect (by experience with the single layer case) that if the
input-data matrix is not random (which is beyond our assumptions) then we
will encounter convergence issues, which could be fixed by moving to
some variant of the algorithm such as VAMP \cite{[arXiv:1612.01186]}.}

%% file: sections/main/2neurons.tex
\section{From two to more hidden neurons, and the specialization phase transition}
\subsection{Two neurons} 

\begin{figure}[t]
\centering
\includegraphics[width=1.0\linewidth]{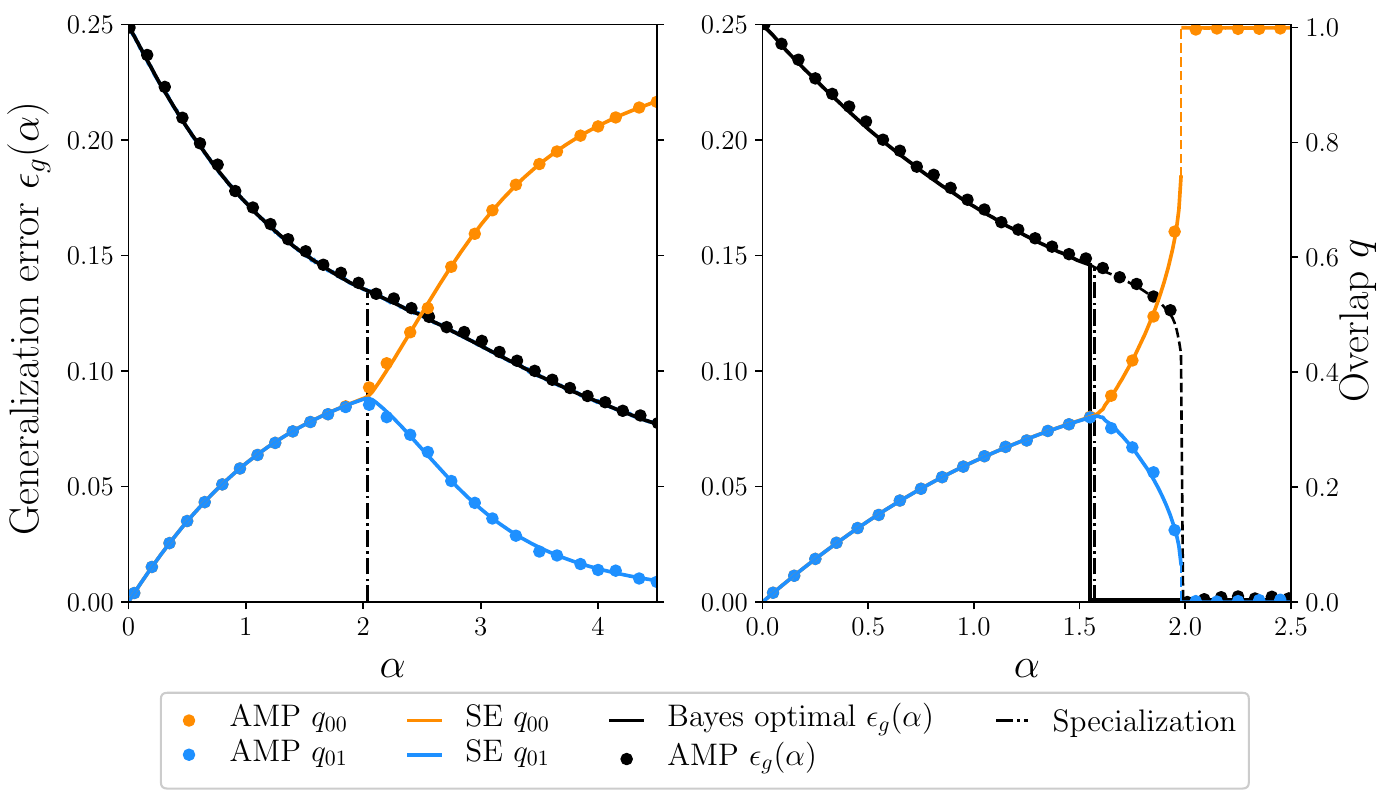}
 \caption{Generalization error and order parameter for a committee
   machine with two hidden neurons ($K=2$) with Gaussian weights (left),
   binary/Rademacher weights (right). These are shown as a function
   of the ratio $\alpha=m/n$ between the number of samples $m$ and the
   dimensionality $n$. Lines are obtained from the state evolution (SE)
   equations (dominating solution is shown in full line), data-points from the AMP algorithm averaged over 10
   instances of the problem of size $n=10^4$.
   $q_{00}$ and $q_{01}$ denote diagonal and off-diagonal overlaps,
   and their values are given by the labels on the far-right of the figure.
   }
\label{fig:phaseDiagramK2}
\vspace{-0.5cm}
\end{figure}

\begin{figure}[t]
\centering
\includegraphics[width=1.0\linewidth]{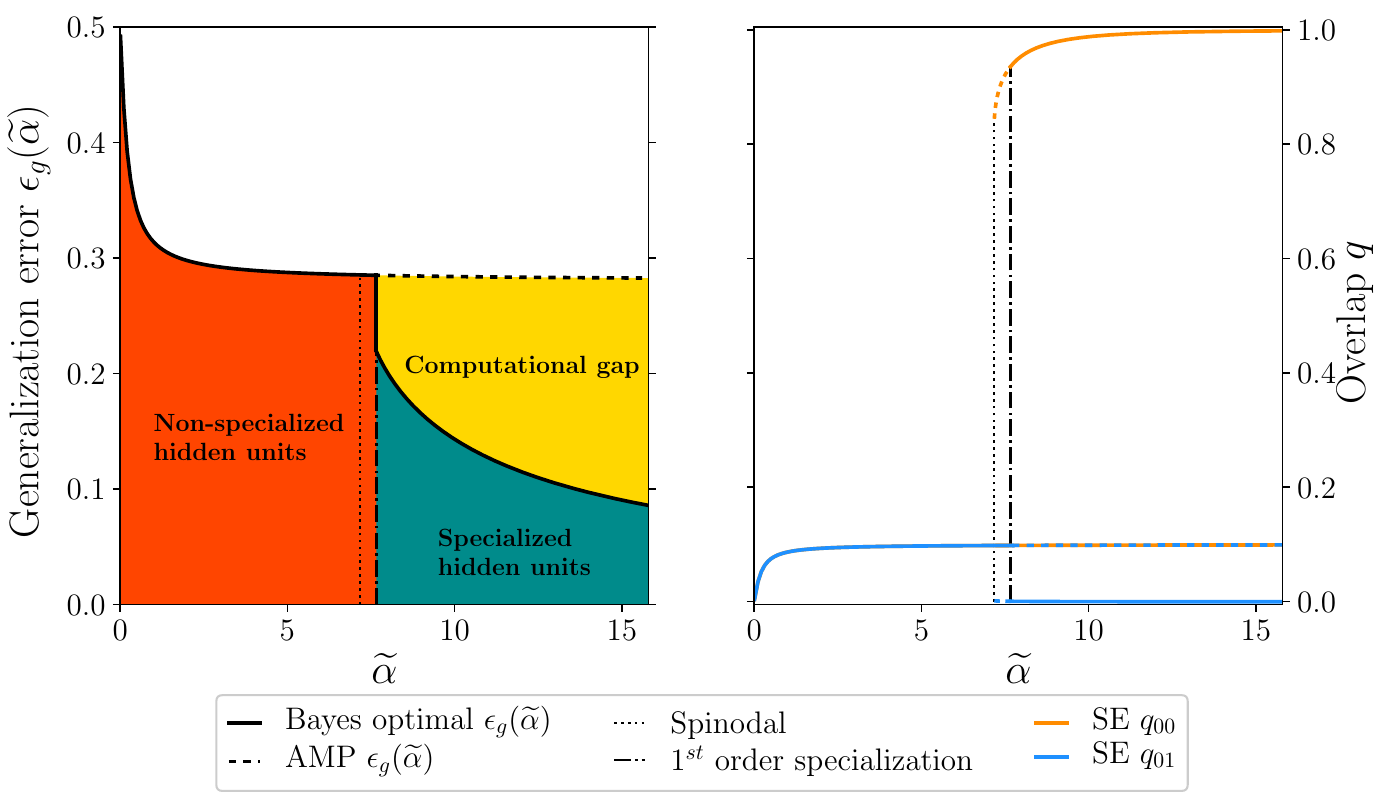}
 \caption{(Left) Bayes optimal and AMP generalization errors and (right) diagonal and off-diagonal overlaps $q_{00}$ and $q_{01}$ for a committee
   machine with a large number of hidden neurons $K$ and Gaussian weights, as a function
   of the rescaled parameter $\tilde{\alpha}=\alpha/K$. Curves shown correspond to the value $K = 10$.
   Solutions corresponding to global and local minima of the replica free entropy are respectively represented with full and dashed lines. The dotted line marks the spinodal at $\widetilde \alpha^G_{\rm spinodal}\simeq 7.17$, i.e.\ the apparition of a local minimum in the replica free entropy, associated to a solution with specialized hidden units.
  The dotted-dashed line shows the first order specialization transition at $\widetilde \alpha^G_{\rm spec} \simeq 7.65$, at which the specialized fixed point becomes the global minimum. 
  For $\widetilde \alpha < \widetilde \alpha^G_{\rm spec}$, AMP reaches the Bayes-optimal generalization error and overlaps, corresponding to a non-specialized solution 
  with $q_{00} = q_{01}$. 
   However, for $\widetilde \alpha > \widetilde \alpha^G_{\rm spec}$, the AMP algorithm does not follow the optimal specialized solution and is stuck in the non-specialized solution plateau, represented with dashed lines 
  (in particular $q_{00}^{\mathrm{AMP}} = q_{01}^{\mathrm{AMP}} \simeq 1/K$ at large $\widetilde \alpha$). 
   Hence, it unveils a large computational gap (yellow area). 
   We finally emphasize that the initial descent of the generalization error of the non-specialized solution to a plateau occurs for finite $\alpha$ as $K \to \infty$ (i.e.\ for $\widetilde \alpha$ going to $0$).
   On the other hand, the $K \to \infty$ limit of the transition points $(\widetilde \alpha^G_{\rm spec},\widetilde \alpha^G_{\rm spinodal})$, as well the generalization error values for all finite $\widetilde \alpha$, are found to be very well approximated by their values for $K = 10$.
   }
\label{fig:phaseDiagramKlarge}
\vspace{-0.5cm}
\end{figure}

Let us now discuss how the above results can be used to
study the optimal learning in the simplest non-trivial case of a
two-layers neural network with two hidden neurons, that is when model (\ref{model:com}) is simply
$$Y_{\mu}={\rm{\sign}}\Big[{\sign} \Big( \sum_{i=1}^n X_{\mu i}
W^*_{i1}\Big)+{\sign} \Big( \sum_{i=1}^n X_{\mu i} W^*_{i2}\Big)\Big]\,,$$ and is represented in Fig.~\ref{fig:committee}, with the
convention that ${\sign}(0)=0$. We remind that the
input-data matrix $X$ has i.i.d.\ ${\cal N}(0,1)$ entries, and the
teacher-weights $W^*$ used to generate the labels $Y$ are taken i.i.d.\ from $P_0$.

In Fig.~\ref{fig:phaseDiagramK2} we plot the optimal generalization error as a
function of the sample complexity $\alpha=m/n$. In the left panel the weights are Gaussian
(for both the teacher and the student), while in the right panel they are
binary/Rademacher. The full line is obtained from the fixed point of
the state evolution (SE) of the
AMP algorithm \eqref{main:StateEvolution}, corresponding to the extremizer of the replica free entropy \eqref{repl-1}. The points are results of the AMP algorithm run till convergence averaged
over 10 instances of size $n=10^4$. 
{\color{black} In this case and with random initial conditions the AMP algorithm did converge in all our trials.}
As expected we observe excellent agreement
between the SE and AMP.

In both left and right panels of Fig.~\ref{fig:phaseDiagramK2} we
observe the so-called {\it specialization} phase transition. 
Indeed, \eqref{main:StateEvolution} has two types of fixed points: a {\it
  non-specialized} fixed point where every matrix element of the $K\times K$
order parameter $q$ is the same (so that both hidden neurons learn the
same function) and a {\it specialized} fixed point where the diagonal
elements of the order parameter are different from the non-diagonal
ones. We checked for other types of fixed points for $K=2$ (one where
the two diagonal elements are not the same), but have not found
any. In terms of weight-learning, this means for the non-specialized fixed point
that the estimators for both $W_{1}$ and $W_{2}$ are the same,
whereas in the specialized fixed point the estimators of the weights
corresponding to the two hidden neurons are different, and that the
network ``figured out'' that the data are better described by a model that 
is not linearly
separable. The specialized fixed point is associated
with lower error than the non-specialized one (as one can see in
Fig.~\ref{fig:phaseDiagramK2}). The existence of this phase transition
was discussed in statistical physics literature on the
committee machine, see
e.g. \cite{schwarze1992generalization,saad1995line}.

For Gaussian weights (Fig.~\ref{fig:phaseDiagramK2} left), the
specialization phase transition arises continuously at
$\alpha^G_{\rm spec}(K=2)\simeq 2.04$. This means that for
$\alpha<\alpha^G_{\rm spec}(K=2)$ the number of samples is too small,
and the student-neural network is not able to learn that two different
teacher-vectors $W_1$ and $W_2$ were used to generate the observed
labels. For $\alpha>\alpha^G_{\rm spec}(K=2)$, however, it  is able to
distinguish the two different weight-vectors and the generalization
error decreases fast to low values (see
Fig.~\ref{fig:phaseDiagramK2}). For completeness, we remind that in the
case of $K=1$ corresponding to single-layer neural network no such
specialization transition exists. We show in sec.~\ref{sec:linear_net} that it is absent also in multi-layer neural networks as long as the activations remain linear. The
non-linearity of the activation function is therefore an essential
ingredient in order to observe a specialization phase transition.

The right part of Fig.~\ref{fig:phaseDiagramK2} depicts the fixed
point reached by the state evolution of AMP for the case of binary
weights. We observe two phase transitions in the performance of AMP in
this case: (a) the specialization phase transition at
$\alpha_{\rm spec}^{B}(K=2) \simeq 1.58$, and for slightly larger
sample complexity a transition towards {\it perfect generalization}
(beyond which the generalization error is asymptotically zero) at
$\alpha^B_{\rm perf}(K=2)\simeq 1.99$. The binary case with $K=2$
differs from the Gaussian one in the fact that perfect generalization
is achievable at finite $\alpha$. While the specialization transition
is continuous here, the error has a discontinuity at the 
transition of perfect generalization. This discontinuity is associated
with the 1st order phase transition (in the physics nomenclature), leading
to a gap between algorithmic (AMP in our case) performance and
information-theoretically optimal performance reachable by exponential
algorithms. To quantify the optimal performance we need to evaluate
the global extremum of the replica free entropy (not the local one
reached by the state evolution). In doing so that we get that information
theoretically there is a single discontinuous phase transition towards
perfect generalization at $\alpha^B_{\rm IT}(K=2)\simeq 1.54$.

While the information-theoretic and specialization
phase transitions were identified in the physics literature
on the committee machine
\cite{schwarze1992generalization,schwarze1993generalization,seung1992statistical,watkin1993statistical},
the gap between the information-theoretic performance and the
performance of AMP ---that is conjectured to be optimal
among polynomial algorithms--- was not yet discussed in the context of
this model. Indeed, even its understanding in simpler models than those
discussed here, such as the single layer case, is more recent
\cite{donoho2009message,REVIEWFLOANDLENKA,donoho2013accurate}.

%% file: sections/main/moreneurons.tex
\subsection{More is different} It becomes more difficult to study the
replica formula for larger values of $K$ as it involves (at least)
$K$-dimensional integrals. Quite interestingly, it is possible to
work out the solution of the replica formula in the large $K$ limit (thus taken {\it after} the large $n$ limit, so that $K/n$ vanishes).
It is indeed natural to look for solutions of the replica formula, as
suggested in \cite{schwarze1993learning}, of the form $q = q_d
I_{K\times K} + ({q_a}/{K}) \textbf{1}_K \textbf{1}_K^\intercal$, with the
unit vector $\textbf{1}_K = (1)_{l=1}^K$. Since both $q$ and $\rho$
are assumed to be positive, this scaling implies that $0\leq q_d \leq 1$ and $0 \leq q_a +
q_d \leq 1$, as it should, see sec. \ref{sec:largeK}. We also detail in this same section the corresponding large $K$ expansion of the free entropy
for the teacher-student scenario with Gaussian weights. Only the
information-theoretically reachable generalization error was computed
\cite{schwarze1993learning}, thus we concentrated on the analysis of
performance of AMP by tracking the state evolution equations. In doing
so, we unveil a large computational gap.

In the right panel of Fig.~\ref{fig:phaseDiagramKlarge} we show the fixed
point values of the two overlaps $q_{00} = q_d + q_a/K$ and $q_{01} = q_a/K$ and the resulting generalization error, plotted in the left panel. As discussed in \cite{schwarze1993learning} it
can be written in a closed form as $\epsilon_g= \arccos
\left[2\left(q_a + \arcsin q_d\right)/\pi\right]/\pi$, represented in the left panel of Fig.~\ref{fig:phaseDiagramKlarge}. The specialization
transition arises for $\alpha=\Theta(K)$, so we define $\widetilde
\alpha\equiv \alpha/K$. The specialization is now a 1st order phase
transition, meaning that the specialization fixed point first appears
at $\widetilde \alpha^G_{\rm spinodal}\simeq 7.17$ but the
free entropy global extremizer remains the one of the non-specialized fixed point until
$\widetilde  \alpha^G_{\rm spec}\simeq 7.65$. This has
interesting implications for the optimal generalization
error that gets towards a plateau of value $\varepsilon_{\rm plateau}
\simeq 0.28$ for $\widetilde \alpha < \widetilde \alpha^G_{\rm spec} $ and then
jumps discontinuously down to reach a decay asymptotically as $1.25/ \widetilde \alpha$. See left panel of Fig.~\ref{fig:phaseDiagramKlarge}.

AMP is conjectured to be optimal among all polynomial algorithms (in the considered limit) and thus analyzing its state
evolution sheds light on possible computational-to-statistical gaps
that come hand in hand with 1st order phase transitions. In the
regime of $\alpha = \Theta(K)$ for large $K$ the non-specialized fixed
point is always stable implying that AMP will not be
able to give a lower generalization error than $\varepsilon_{\rm
plateau}$. Analyzing the replica formula for large $K$ in more details, see sec. \ref{sec:largeK}, we concluded that AMP will not reach the optimal generalization for any
$\alpha < \Theta(K^2)$. This implies a rather sizable gap between the
performance that can be reached information-theoretically and the one
reachable tractably (see yellow area in Fig.~\ref{fig:phaseDiagramKlarge}). Such large computational gaps have been
previously identified in a range of inference problems ---most
famously in the planted clique problem \cite{deshpande2015finding}---
but the committee machine is the first model of a multi-layer neural
network with realistic non-linearities (the parity machine is another
example but use a very peculiar non-linearity) that presents such large gap.

%% file: sections/main/proof_new_last.tex
\section{Structure of the proof of Theorem \ref{main-thm}}\label{sec:proofsketch} 

All along this section we assume \ref{hyp:1}, \ref{hyp:2} and \ref{hyp:3}, and all the rigorous statements are implicitly assuming these hypotheses. We denote $K$-dimensional column vectors by underlined letters. In particular $\underline{W}_i^* = (W_{il}^*)_{l=1}^K$, 
$\underline{w}_i = (w_{il})_{l=1}^K$. 
For $\mu=1,\ldots m$, let $\underline{V}_\mu$, $\underline{U}_\mu^*$ be $K$-dimensional vectors with i.i.d.\ ${\cal N}(0,1)$ components. 
Let $s_n \in (0,1/2]$ a sequence that goes to $0$ as $n$ increases,  
and let $\mathcal{M}$ be the compact subset of 
matrices in $S_K^{++}$ with eigenvalues in the interval $[1,2]$. 
For all $M \in s_n\mathcal{M}$, $2s_nI_{K\times K} - M \in \mathcal{S}_K^{+}$.
\subsection{Interpolating estimation problem}
Let $\epsilon = (\epsilon_1,\epsilon_2)\in (s_n \mathcal{M})^2$. Let $q : [0,1] \to \mathcal{S}_K^+(\rho)$ and $r : [0,1] \to \mathcal{S}_K^+$ be two ``interpolation functions'' (that will later on depend on $\epsilon$), and 
\begin{align}\label{Rmap}
  R_1(t) \equiv \epsilon_1 + \int_0^t r(v) dv \,, \qquad R_2(t) \equiv \epsilon_2 + \int_0^t q(v) dv \,.
\end{align} 
For $t\in [0,1]$, define the $K$-dimensional vector: 
\begin{align}
\underline{S}_{t,\mu} \equiv \sqrt{\frac{1-t}{n}}\, \sum_{i=1}^n X_{\mu i} \underline{W}_i^*  + \sqrt{R_2(t)} 
\,\underline{V}_{\mu}  + \sqrt{t \rho - R_2(t) + 2s_n I_{K\times K}} \,\underline{U}_{\mu}^* \label{bigs}  
\end{align}
where matrix square-roots (that we denote equivalently $A^{1/2}$ or $\sqrt{A}$) are well-defined.
We interpolate with auxiliary problems related to those discussed in sec.~\ref{part2}; the interpolating estimation problem is given by the following observation model, with two types of $t$-dependent observations:
\begin{align}
  \label{2channels}
      \left\{
    \begin{array}{lll}
      Y_{t,\mu}\sim  P_{\rm out}(\ \cdot \ | \, \underline{S}_{t,\mu}),\qquad &1 \leq \mu \leq m\,,\\
	\underline{Y}'_{t,i} = \sqrt{R_1(t)} \, \underline{W}^*_i + \underline{Z}'_i,\qquad &1 \leq i \leq n\,,
  \end{array}
  \right.
\end{align}
where $\underline{Z}'_i$ is (for each $i$) a $K$-vector with i.i.d.\ ${\cal N}(0,1)$ components, and $\underline{Y}'_{t,i}$ is a $K$-vector as well. Recall that in our notation the $*$-variables have to be retrieved, while the other random variables are assumed to be known (except for the noise variables obviously). 
Define now $\underline{s}_{t,\mu}$ by the expression of $\underline{S}_{t, \mu}$ but with $\underline{w}_i$ replacing $\underline{W}_i^*$ and $\underline{u}_\mu$ replacing $\underline{U}_\mu^*$. 
We introduce the {\it interpolating posterior}:
\begin{equation}
 P_{t,\epsilon}(w, u |Y_t, Y'_t, X, V)   = \frac{1}{{\cal Z}_{n,\epsilon}(t)} \prod_{i=1}^n P_0(\underline w_i) e^{-\frac{1}{2}\Vert \underline{Y}_{t,i}' - \sqrt{R_1(t)} \underline w_i\Vert_2^2}
 \prod_{\mu=1}^m \frac{e^{-\frac{1}{2}\|\underline{u}_\mu\|_2^2}}{(2\pi)^{K/2}}P_{\rm out}(Y_{t, \mu}| \underline s_{t, \mu}) \label{tpost}
\end{equation}
where the normalization factor ${\cal Z}_{n,\epsilon}(t)$ equals the numerator integrated over all components of $w$ and $u$. The average 
free entropy at time $t$ is by definition 
\begin{align}
f_{n, \epsilon}(t) \equiv \frac1n\mathbb{E}\ln {\cal Z}_{n,\epsilon}(t)=\frac1n\mathbb{E}\ln\int{\cal D}u\prod_{i=1}^n dP_0(\underline w_i) 
 \prod_{\mu=1}^m  P_{\rm out}(Y_{t, \mu}| \underline s_{t, \mu})
 \prod_{i=1}^n e^{-\frac{1}{2}\Vert \underline{Y}_{t,i}' - \sqrt{R_1(t)} \underline w_i\Vert_2^2}\,,\label{f(t)}
\end{align}
where $\mathcal{D}u = \prod_{\mu=1}^m\prod_{l=1}^K (2\pi)^{-1/2}e^{-u_{\mu l}^2/2}$.

The presence of the small ``perturbation'' $\epsilon$ induces a proportional change in the free entropy of the interpolating model:
\begin{lemma}[Perturbation of the free entropy]
For all $\epsilon\in(s_n {\cal M})^2$ we have for $t=0$ that $|f_{n,\epsilon}(0) - f_{n,\epsilon=(0,0)}(0) | \leq C' s_n$ for some positive constant $C'$. Moreover, $|f_{n} - f_{n,\epsilon=(0,0)}(0)|\le C s_n$ for some positive constant $C$, so that 
\begin{align*}
  |f_{n} - f_{n,\epsilon=(0,0)}(0) | =\smallO_{n}(1)\,.
\end{align*} 
\end{lemma}
\begin{proof}
Let us compute (or directly obtain by the I-MMSE formula for vector channels \cite{reeves2018mutual,payaro2011yet,lamarca2009linear})
\begin{align}
\nabla_{\epsilon_1} f_{n,\epsilon}(0) =  - \frac{1}{2} \left[\rho - \EE\langle Q\rangle_{n,0,\epsilon} \right],
\end{align}
where the $K\times K$ {\it overlap matrix} $(Q_{ll'})$ is defined below by \eqref{overlap_def}.
 Note that the r.h.s.\ of the above equation is (up to a factor $-1/2$) the $K\times K$ MMSE matrix. 
Set $u_y(x)\equiv\ln P_{\rm out}(y|x)$.
 Now we compute (by calculations very similar to the ones used in the proof of the following Proposition~\ref{prop:der_f_t}):
\begin{align}
\nabla_{\epsilon_2} f_{n,\epsilon}(0) = \frac{1}{2n} \sum_{\mu = 1}^m \EE \Big[\nabla u_{Y_{t,\mu}}(\underline{S}_{t,\mu}) \Big \langle \nabla u_{Y_{t,\mu}}(\underline{s}_{t,\mu}) \Big \rangle_{n,0,\epsilon}\Big].
\end{align}
Note that the r.h.s.\ of the above equation is symmetric by the Nishimori identity Proposition~\ref{prop:nishimori}.
By the mean value theorem we obtain then directly that $|f_{n,\epsilon}(0) - f_{n,\epsilon=(0,0)}(0) |\le \|\nabla_{\epsilon_1} f_{n,\epsilon}(0)\|_{\rm F} \|\epsilon_1\|_{\rm F} + \|\nabla_{\epsilon_2} f_{n,\epsilon}(0)\|_{\rm F} \|\epsilon_2\|_{\rm F} \le C \max_i \|\epsilon_i\|\le C' s_n$.
\end{proof}

Using this lemma one verifies, using in particular continuity and boundedness properties of $\psi_{P_0}$ and $\Psi_{\rm P_{out}}$ (see Lemma \ref{convexity} in sec.~\ref{smproof} for details; sec.~\ref{smproof} gathers the detailed proofs of all the propositions below):
\begin{align}\label{bound}
  \left\{
    \begin{array}{lll}
      f_{n, \epsilon}(0) &=& f_n - \frac{K}{2} + \smallO_{n}(1) \,,\\
      f_{n, \epsilon}(1) &=& \psi_{P_0}(\int_0^1 r(t) dt) + \alpha \Psi_{P_{\rm out}}(\int_0^1 q(t) dt;\rho) - \frac{1}{2}\int_0^1 {\rm Tr}[\rho \,r(t)]dt - \frac{K}{2} + \smallO_{n}(1) \,.
    \end{array}
  \right.
\end{align}
Here $\smallO_n(1)\to 0$ in the
  $n,m\!\to\!\infty$ limit uniformly in $t$, $q$, $r$, $\epsilon$. 

\subsection{Overlap concentration and fundamental sum rule}
Notice from \eqref{bound} that at $t=1$ the interpolating estimation problem constructs part of the RS potential \eqref{RSpot}, while at $t=0$ it is the free entropy \eqref{freeent} of the original model \eqref{modelnoise} (up to a constant). We thus now want to compare these boundary values thanks to the identity
\begin{align}
f_n = f_{n,\epsilon}(0)+\frac{K}{2}+\smallO_{n}(1)=f_{n,\epsilon}(1)-\int_0^1\frac{df_{n,\epsilon}(t)}{dt} dt +\frac{K}{2}+\smallO_{n}(1)  \,.\label{f0_f1_int}
\end{align}

The next obvious step is therefore to compute the free entropy variation along the interpolation path, see sec.~\ref{sec:proofderft} for the proof:
\begin{proposition}[Free entropy variation]\label{prop:der_f_t} 
  Denote by $\langle -\rangle_{n, t, \epsilon}$ the (Gibbs) expectation w.r.t.
  the posterior $P_{t,\epsilon}$ given by \eqref{tpost}. Set $u_y(x)\equiv\ln P_{\rm out}(y|x)$. For all $t \in [0,1]$ we have
\begin{small}
  \begin{equation*}
   \frac{df_{n,\epsilon}(t)}{dt}=- \frac{1}{2} 
    \mathbb{E} \Big \langle  {\rm Tr} \Big[
      \Big( 
        \frac{1}{n} \sum_{\mu=1}^{m}\nabla u_{Y_{t,\mu}}(\underline{s}_{t,\mu}) \nabla u_{Y_{t,\mu}}(\underline{S}_{t,\mu})^\intercal
        - r(t)
        \Big) 
      \big( 
        Q - q(t)
       \big)  \Big]
    \Big\rangle_{n, t, \epsilon}
+ \frac{1}{2}{\rm Tr}\left[r(t) (q(t)-\rho)\right]  +  \smallO_n(1)\,,
  \end{equation*}
\end{small}
where $\nabla$ is the $K$-dimensional gradient w.r.t. the
  argument of $u_{Y_{t,\mu}}(\cdot)$, and $\smallO_n(1)\to 0$ in the
  $n,m\!\to\!\infty$ limit uniformly in $t$, $q$, $r$, $\epsilon$. Here, the $K\!\times\!
  K$ {\it overlap} matrix $Q$ is defined as 
  \begin{align}
    Q_{ll'}
  \equiv \frac1n\sum_{i=1}^n W_{il}^* w_{il'}\,.\label{overlap_def}
  \end{align}
\end{proposition}

We will plug this expression in identity \eqref{f0_f1_int}, but in order to simplify it we need the following crucial proposition, which says that the overlap concentrates. This property is what is generally referred to as a replica symmetric behavior in statistical physics.
\begin{proposition}[Overlap concentration] \label{concentration}
Assume that for any $t \in (0,1)$ the transformation $\epsilon \in (s_n \mathcal{M})^2 \mapsto (R_1(t,\epsilon),R_2(t,\epsilon))$ is a $\mathcal{C}^1$ diffeomorphism with a Jacobian determinant greater or equal to $1$.
Then one can find a sequence $s_n$ going to $0$ slowly enough such that there exists a constant $C(\varphi_{\rm out},S,K,\alpha)>0$ depending only on the activation $\varphi_{\rm out}$, the support $S$ of the prior $P_0$, the number of hidden neurons $K$ and the sampling rate $\alpha$, and a constant $\gamma > 0$ such that ($\|-\|_{\rm F}$ is the Frobenius norm):
$$
\frac{1}{\mathrm{Vol}(s_n \mathcal{M})^{2}} \int_{(s_n \mathcal{M})^2} 
d \epsilon \int_0^1 dt\, \mathbb{E}\big\langle  \big\Vert Q -  \mathbb{E}\langle Q\rangle_{n, t, \epsilon}\big\Vert_{\rm F}^2  \big\rangle_{n, t, \epsilon}  \leq \frac{C(\varphi_{\rm out},S,K,\alpha)}{n^{\gamma}}\,.$$
\end{proposition}
The proof of this concentration result can be directly adapted from \cite{barbier2019overlap}. Using the results of \cite{barbier2019overlap} is straightforward, under the assumption that $\epsilon\mapsto R(t,\epsilon)$ is a $\mathcal{C}^1$ diffeomorphism with a Jacobian determinant greater or equal to $1$. This Jacobian determinant can be computed from formula \eqref{detgreater1_2}. To check that it is greater than one we use Lemma \ref{lemma:positivity_trace_jac} and need Assumption \ref{assumption} stated in paragraph \ref{subsec:technical-assumption} below. With a Jacobian determinant greater than one, we can ``replace'' (i.e., lower bound) the integrations over $R_1(t,\epsilon)$, that naturally appear in the proof of Proposition \ref{concentration}, by integrations over the perturbation matrix $\epsilon$. This is {\it exactly} what has been done in the $K=1$ version of the present model in \cite{barbier2017phase} or in \cite{barbier_adaptInterp_review} i.e., in the scalar overlap case (see also \cite{barbier2019mutual} for a setting with a matrix overlap as in the present case). 

From there we can deduce the following fundamental sum rule which is at the core of the proof:
\begin{proposition}[Fundamental sum rule]\label{sum rule}
  Assume that the interpolation functions $r$ and $q$ are such that the map $\epsilon=(\epsilon_1,\epsilon_2) \mapsto R(t,\epsilon)=(R_1(t,\epsilon),R_2(t,\epsilon))$ given by \eqref{Rmap} is a ${\cal C}^1$ diffeomorphism whose Jacobian determinant $J_{n,\epsilon}(t)$ is greater or equal to $1$. Assume that for all $t \in [0,1]$ and $\epsilon \in (s_n {\cal M})^2$ we have $q(t)= q(t,\epsilon) = \EE \langle Q \rangle_{n,t,\epsilon}\in\mathcal{S}_K^+(\rho)$. Then
  \begin{align}
f_n &= \frac{1}{\mathrm{Vol}(s_n \mathcal{M})^2} \int_{(s_n \mathcal{M})^2} d\epsilon\Big\{ \psi_{P_0}\Big(\int_0^1 r(t) dt\Big) 
 + \alpha \Psi_{P_{\rm out}}\Big(\int_0^1 q(t,\epsilon) dt ; \rho\Big)\nn
 &\qquad\qquad\qquad\qquad\qquad\qquad\qquad\qquad\qquad\qquad\qquad\qquad- \frac{1}{2} \int_0^1 {\rm Tr}[q(t,\epsilon) r(t)] dt\Big\} + \smallO_n(1)\,.
  \end{align}
\end{proposition}
\begin{proof}
  Let us denote $V_n\equiv\mathrm{Vol}(s_n \mathcal{M})^2$. The integral over $\epsilon$ is always over $(s_n \mathcal{M})^2$. Consider the first term, i.e. the Gibbs bracket, in the free entropy derivative given by Proposition \ref{prop:der_f_t}. By the Cauchy-Schwarz inequality
  \begin{align*}
  &\Big(\mathbb{E} \Big \langle  {\rm Tr} \Big[
      \Big( 
        \frac{1}{n} \sum_{\mu=1}^{m}\nabla u_{Y_{t,\mu}}(\underline{s}_{t,\mu}) \nabla u_{Y_{t,\mu}}(\underline{S}_{t,\mu})^\intercal
        - r(t)
        \Big) 
      \big( 
        Q - q(t)
       \big)  \Big]
    \Big\rangle_{n, t, \epsilon}\Big)^2
    \nn
    \leq 
    \frac{1}{V_n}\int &\,d\epsilon\int_0^1 dt\, 
      \EE \Big\langle 
        \Big\|
          \frac{1}{n} \sum_{\mu=1}^{m}\nabla u_{Y_{t,\mu}}(\underline{s}_{t,\mu}) \nabla u_{Y_{t,\mu}}(\underline{S}_{t,\mu})^\intercal
          - r(t)
        \Big\|_{\rm F}^2
    \Big\rangle_{n, t, \epsilon}\times\frac{1}{V_n}\int d\epsilon\int_0^1dt\, \EE \big\langle 
        \big\|
          Q - q(t)
        \big\|_{\rm F}^2
      \big\rangle_{n, t, \epsilon}\,.\nonumber
  \end{align*}
  The first term of this product is bounded by some constant $C(\varphi_{\rm out},\alpha)$ that only depend on $\varphi_{\rm out}$ and $\alpha$, see Lemma \ref{Bounded-fluctuation} in sec.~\ref{sec:techLemmas}. The second term is bounded by $C(\varphi_{\rm out},S,K,\alpha)n^{-\gamma}$ by Proposition \ref{concentration}, since we assumed that for all $\epsilon \in \mathcal{B}_n$ and all $t \in [0,1]$ we have $q(t) =q(t,\epsilon) = \EE \langle Q \rangle_{n,t,\epsilon}$.
  Therefore, from Proposition \ref{prop:der_f_t} we obtain
  \begin{align}
    \frac{1}{V_n}\int d\epsilon \int_0^1 \frac{df_{n,\epsilon}(t)}{dt} dt = 
    \frac{1}{2V_n}\int d\epsilon \int_0^1 
    {\rm Tr}\big[q(t,\epsilon)r(t) 
    - r(t)\rho \big]dt
    + \smallO_n(1)+{\cal O}(n^{-\gamma/2})
    \,.
    \label{eq:id_fluctuation}
  \end{align}
  Here the small terms are both going to $0$ uniformly w.r.t. to the choice of $q$ and $r$.
  When replacing \eqref{eq:id_fluctuation} in \eqref{f0_f1_int} and combining it with \eqref{bound} we reach
  the claimed identity.
\end{proof}

\subsection{A technical lemma and an assumption}\label{subsec:technical-assumption}

We give here a technical lemma used in the rest of the proof, 
and which allows us to detail the unproven assumption on which we rely to prove Thm~\ref{main-thm}.

\begin{lemma}\label{lemma:positivity_trace_jac}
  The quantity $\EE \langle Q\rangle_{n,t,\epsilon}$ is a function of $(n,t,R(t,\epsilon))$. 
  We define $F_n^{(1)}(t,R(t,\epsilon)) \equiv \EE \langle Q \rangle_{n,t,\epsilon}$ and $F_n^{(2)}(t,R(t,\epsilon)) \equiv 2 \alpha \nabla \Psi_{P_{\rm out}} (\EE \langle Q \rangle_{n,t,\epsilon})$.
  $F_n \equiv (F_n^{(1)},F_n^{(2)})$ is defined on the set:
  \begin{align}
    D_n &= \Big\{ (t,r_1,r_2) \in [0,1] \times {\cal S}_K^+ \times {\cal S}_K^+ \Big| (\rho t - r_2 + 2 s_n I_K) \in {\cal S}_K^+ \Big \}.
  \end{align}
  $F_n$ is a continuous function from $D_n$ to ${\cal S}_K^+ \times {\cal S}_K^+(\rho)$. Moreover, $F_n$ 
  admits partial derivatives with respect to $R_1$ and $R_2$ on the interior of $D_n$. For every $(t,R(t,\epsilon))$ for 
  which they are defined, they satisfy:
  \begin{align}\label{eq:Fn1}
    \sum_{l \leq l'}^K \frac{\partial (F_n^{(1)})_{ll'}}{\partial (R_1)_{ll'}} \geq 0. 
  \end{align}
\end{lemma}
We can now state the technical assumption on which we rely\footnote{Since the publication of this work the adaptive interpolation method used in this paper has been improved for finite-rank models and can now circumvent this artificial hypothesis, see \cite{barbier2020information} and \cite{reeves2020information}.}, and which essentially allows us to derive that the map $\epsilon\mapsto R(t,\epsilon)$ is a $\mathcal{C}^1$ diffeomorphism with a Jacobian determinant greater or equal to $1$ as it will become clear in the next section:  
\begin{assumption}\label{assumption}
  With the notations of Lemma~\ref{lemma:positivity_trace_jac}, 
  \begin{align*}
    \sum_{l \leq l'}^K \frac{\partial (F_n^{(2)})_{ll'}}{\partial (R_2)_{ll'}} \geq 0. 
  \end{align*}
\end{assumption}

\begin{proof}[Proof of Lemma~\ref{lemma:positivity_trace_jac}]
  The fact that the image domain of $F_n$ is ${\cal S}_K^+ \times {\cal S}_K^+(\rho)$ is known from Lemma~\ref{lemma:posMat}.
  The continuity and differentiability of $F_n$ follows from standard theorems of continuity and derivation under the integral sign 
  (recall that we are working at finite $n$). Indeed, the domination hypotheses are easily satisfied since we work under \ref{hyp:1}
  and \ref{hyp:2}.
  
  Let us now prove \eqref{eq:Fn1}.
 We write the formal differential of $F_n^{(1)}$ with respect to $R_1$ 
 as $\mathcal{D}_{R_1} F_n^{(1)}$, which is a $4$-tensor, and our goal is to prove that $\text{Tr}[\mathcal{D}_{R_1} F_n^{(1)}] \geq 0$,
  the trace of a 4-tensor over $S_K$ $A_{(ij)(kl)}$ being $\text{Tr}[A] = \sum_{i\leq j} A_{(ij)(ij)}$.
 Then one can write $\text{Tr}[\mathcal{D}_{R_1} F_n^{(1)}] = \text{Tr}[\nabla \nabla^\intercal \Psi_{P_{\rm out}}(\EE \langle Q \rangle_{n,t,\epsilon}) \times \nabla_{R_1} \EE \langle Q \rangle_{n,t,\epsilon}]$. 
 We know from Lemma~\ref{lemma:posMat} and Lemma~\ref{convexity} that $\nabla \nabla^\intercal \Psi_{P_{\rm out}}(\EE \langle Q \rangle_{n,t,\epsilon})$ 
 is a positive symmetric matrix (when seen as a linear operator over $\mathcal{S}_K$). 
 Moreover, it is a known result that the derivative $\nabla_{R_1} \EE \langle Q \rangle_{n,t,\epsilon}$ is also positive symmetric,
  since $R_1$ is the matrix snr of a \emph{linear} channel (see \cite{reeves2018mutual,payaro2011yet,lamarca2009linear}). 
  Since the product of two symmetric positive matrices has always positive trace, this shows that $\text{Tr}[\mathcal{D}_{R_1} F_n^{(1)}] \geq 0$.
\end{proof}

\subsection{Matching bounds}

\begin{proposition}[Lower bound]\label{lower} Under Assumption~\ref{assumption}, the free entropy of model \eqref{modelnoise} verifies $$\liminf_{n \to \infty} f_n \geq {\adjustlimits\sup_{r \in \mathcal{S}_K^+} \inf_{q  \in \mathcal{S}_K^+(\rho)}} f_{\rm RS}(q,r)\,.$$
\end{proposition}

\begin{proof} 
Choose first $r(t) = r \in \mathcal{S}_K^+$ a fixed matrix. Then $R(t)=(R_1(t), R_2(t))$ can be fixed as the solution to the first 
order differential equation:
\begin{align}
\frac{d}{dt}R_{1}(t) = r\,,\qquad \frac{d}{dt} R_{2}(t) = \mathbb{E}\langle Q\rangle_{n,t,\epsilon}\,, \qquad \text{and} \qquad R(0) = \epsilon\,.\label{eqdifflower}
\end{align} 
We denote this (unique) solution $R(t,\epsilon) = (rt + \epsilon_1, \int_0^t q(v,\epsilon;r)dv + \epsilon_2)$. It is 
possible to check that this ODE satisfies the hypotheses of the 
parametric Cauchy-Lipschitz theorem, and that by the Liouville formula the determinant $J_{n,\epsilon}(t)$ of the Jacobian of 
$\epsilon \mapsto R(t,\epsilon)$ satisfies (see Lemma~\ref{lemma:CL-liouville} in sec.~\ref{smproof})
\begin{align}
J_{n,\epsilon}(t) = \exp\Big(\int_{0}^t\sum_{l\geq l'}^K \frac{\partial \mathbb{E}\langle Q_{ll'}\rangle_{n,s,\epsilon}}{\partial({R_{2}})_{ll'}}(s,R(s,\epsilon))  \, ds\Big) \geq 1\,.  
\end{align} 
Indeed, this sum of partial derivatives is always positive by Assumption~\ref{assumption}.
Moreover, from \eqref{eqdifflower}, $q(t,\epsilon;r)=\mathbb{E}\langle Q\rangle_{n,t,\epsilon}$, which is in $\mathcal{S}_K^+$ by Lemma \ref{lemma:posMat} in sec.~\ref{smproof}. The fact that the map $\epsilon \mapsto R(t,\epsilon)$ is a ${\cal C}^1$ diffeomorphism is easily verified by its bijectivity (from the positivity of $J_{n,\epsilon}(t)$) combined with the local inversion Theorem. All the assumptions of Proposition~\ref{sum rule} are verified which then implies, recalling the potential expression \eqref{RSpot},
$$f_n = \frac{1}{\mathrm{Vol}(s_n \mathcal{M})^{2}} \int_{(s_n \mathcal{M})^2} d\epsilon \,f_{\rm RS}\Big(\int_0^1 q(v,\epsilon;r) dv,r\Big) + \smallO_n(1)\,.$$
This implies the lower bound as this equality is true for any $r\in \mathcal{S}_K^+$.
\end{proof}

\begin{proposition}[Upper bound]\label{upper} Under Assumption~\ref{assumption}, the free entropy of model \eqref{modelnoise} verifies $$\limsup_{n \to \infty} f_n \leq {\adjustlimits \sup_{r \in \mathcal{S}_K^+} \inf_{q  \in \mathcal{S}_K^+(\rho)}} f_{\rm RS}(q,r)\,.$$
\end{proposition}

\begin{proof} We now fix $R(t)=(R_1(t), R_2(t))$ as the solution $R(t,\epsilon) = (\int_0^t r(v,\epsilon)dv + \epsilon_1, \int_0^t q(v,\epsilon)dv + \epsilon_2)$ to the 
following Cauchy problem: $$\frac{d}{dt} R_{1}(t) = 2  \alpha \nabla \Psi_{{P_{\rm out}}}(\mathbb{E}\langle Q\rangle_{n,t,\epsilon})\,, \qquad \frac{d}{dt} R_{2}(t) = \mathbb{E}\langle Q\rangle_{n,t,\epsilon}\,, \qquad \text{and} \qquad R(0) =  \epsilon\,.$$ We denote this 
equation as $\partial_t R(t)  = F_n(t,R(t)), R(0)=\epsilon$. It is then possible to verify that $F_n(R(t),t)$
is a bounded $\cC^1$ function 
  of $R(t)$, and thus a direct application of the Cauchy-Lipschitz theorem implies that $R(t,\epsilon)$ is a $\cC^1$ 
  function of $t$ and $\epsilon$. The Liouville formula for the Jacobian determinant of the map $\epsilon \in (s_n\mathcal{M})^2 \mapsto R(t, \epsilon)\in R(t, (s_n\mathcal{M})^2)$ gives this time (see Lemma~\ref{lemma:CL-liouville} in sec.~\ref{smproof})
   \begin{align}\label{detgreater1_2}
    J_{n,\epsilon}(t) = \exp\Big(\int_{0}^t\sum_{l\geq l'}^K\Big\{ \frac{\partial (F_{n,1})_{ll'}}{\partial({R_{1}})_{ll'}}(s,R(s,\epsilon))+\frac{\partial (F_{n,2})_{ll'}}{\partial({R_{2}})_{ll'}}(s,R(s,\epsilon))\Big\} \,ds\Big)\ge 1\,.
  \end{align}
  The fact that this determinant is greater or equal to $1$ for all $t \in [0,1]$ follows again from the positivity of this sum of partials, 
  see Lemma~\ref{lemma:positivity_trace_jac} and Assumption~\ref{assumption}. Identity \eqref{detgreater1_2} implies the bijectivity of $\epsilon \mapsto R(t, \epsilon)$ which, combined with the local inversion theorem, makes it a diffeomorphism. Since $\mathbb{E}\langle Q\rangle_{n,t,\epsilon}$ and $\rho -
  \mathbb{E}\langle Q\rangle_{n,t,\epsilon}$ are positive matrices
  (see Lemma \ref{lemma:posMat} in sec.~\ref{smproof}) we also have that $q(t,\epsilon) \in \mathcal{S}_K^+(\rho)$ and since by the differential equation we have
  $r(t,\epsilon) = 2\alpha\nabla\Psi_{P_{\rm out}}(q(t,\epsilon))$ and as $\nabla\Psi_{P_{\rm out}}(q)\in \mathcal{S}_K^+$ (see Lemma \ref{convexity} in sec.~\ref{smproof}), then 
  $r(t,\epsilon)\in \mathcal{S}_K^+$ too. We have everything needed for applying Proposition~\ref{sum rule} again which gives in this case
  \begin{equation*}
   f_n = \frac{1}{\mathrm{Vol}(s_n \mathcal{M})^2} \int d\epsilon\Big\{\! \psi_{P_0}\Big(\int_0^1 r(v,\epsilon) dv\Big) 
 + \alpha \Psi_{P_{\rm out}}\Big(\int_0^1 q(v,\epsilon) dv ; \rho\Big)
 - \frac{1}{2} {\rm Tr} \int_0^1 q(v,\epsilon) r(v,\epsilon) dv\!\Big\} + \smallO_n(1).
  \end{equation*}
 Then by convexity of $\psi_{P_0}$ and $\Psi_{P_{\rm out}}$ (see Lemma \ref{convexity}), 
\begin{align*}
f_n &\leq \frac{1}{\mathrm{Vol}(s_n \mathcal{M})^2} \int d\epsilon\int_0^1 dv\Big\{\! \psi_{P_0}( r(v,\epsilon) dv) 
 + \alpha \Psi_{P_{\rm out}}(q(v,\epsilon) ; \rho)
 - \frac{1}{2} {\rm Tr} [q(v,\epsilon) r(v,\epsilon)] \!\Big\} + \smallO_n(1) \nn
&= \frac{1}{\mathrm{Vol}(s_n \mathcal{M})^2} \int d\epsilon \int_0^1 dv \,f_{\rm RS}(q(v,\epsilon),r(v,\epsilon))+ \smallO_n(1)\,.  
\end{align*} 
   We now remark that $$f_{\rm RS}(q(v,\epsilon) ,r(v,\epsilon)) = \inf_{q \in \mathcal{S}_K^+(\rho)} f_{\rm RS}(q,r(v,\epsilon))\,.$$ 
   Indeed, for every $r \in \mathcal{S}_K^+$, the function $g_{r}: q \in \mathcal{S}_K^+(\rho) \mapsto
  f_{\rm RS}(q,r)\in \mathbb{R}$ (recall \eqref{RSpot}) is convex (by Lemma \ref{convexity}), and its $q$-derivative is
  $\nabla g_{r}(q)= \alpha  \nabla \Psi_{P_{\rm out}}(q) - {r}/{2}$. 
  Since $\nabla g_{r(v,\epsilon)}(q(v,\epsilon)) = 0$ by definition of $r(v,\epsilon)$, and $\mathcal{S}_K^+(\rho)$ is convex, 
  the minimum of $g_{r(v,\epsilon)}(q)$ is necessarily achieved at $q = q(v,\epsilon)$. 
  Therefore:
  \begin{equation*}
  f_n \leq \frac{1}{\mathrm{Vol}(s_n \mathcal{M})^2} \int_{(s_n \mathcal{M})^2}
  d\epsilon \int_0^1 dv  \underset{q \in \mathcal{S}_K^+(\rho)}{\inf} f_{\rm RS}\left(q, r(v,\epsilon)\right) + \smallO_n(1) 
   \leq {\adjustlimits \sup_{r \in \mathcal{S}_K^+}\inf_{q \in \mathcal{S}_K^+(\rho)}} f_{\rm RS}(q,r) +  \smallO_n(1),
  \end{equation*}
 which concludes the proof of Proposition~\ref{upper}. 
 \end{proof}

 Combining these two matching bounds ends the proof of Theorem~\ref{main-thm}.

%% file: sections/main/conclusion.tex
\section{Discussion}

One of the contributions of this paper is the design of an AMP-type 
algorithm that is able to achieve the Bayes-optimal learning error in
the limit of large dimensions for a range of parameters out of the
so-called hard phase. The hard phase is associated with first order phase
transitions appearing in the solution of the model. In the case of the
committee machine with a large number of hidden neurons we identify
a large hard phase in which learning is possible
information-theoretically but not efficiently. In other problems where
such a hard phase was identified, its study boosted the development
of algorithms that are able to match the predicted threshold. We anticipate
this will also be the same for the present model. We should, however,
note that for larger $K>2$ the present AMP algorithm includes higher-dimensional integrals
that hamper the speed of the algorithm. Our current strategy to
tackle this is to combine the large-$K$ expansion and use it in the
algorithm. Detailed account of the corresponding results are left for
future work. 

We studied the Bayes-optimal setting where the student-network is the same as the teacher-network, for which the replica method can be
readily applied. The method still applies when the number of hidden units in the
student and teacher are different, while our proof does not generalize easily to this case. It is an interesting subject for future work to see how the hard phase evolves
under over-parametrization and what is the interplay between the
simplicity of the loss-landscape and the achievable generalization
error. We conjecture that in the present model over-parametrization
will not improve the generalization error achieved by AMP in the
Bayes-optimal case.  

Even though we focused in this paper on a two-layers neural network,
the analysis and algorithm can be readily extended to a multi-layer
setting, see \cite{MatoParga92}, as long as the number of layers as well as the number of
hidden neurons in each layer is held constant, and as long as one
learns only weights of the first layer, for which  the proof already applies. The numerical evaluation of the phase diagram would be more
challenging than the cases presented in this paper as multiple
integrals would appear in the corresponding formulas. 
In future works, we also plan to analyze the case where the weights of
the second and subsequent layers (including the biases of the
activation functions) are also
learned. This could be done for instance with a combination of EM and
AMP along the lines of \cite{arXiv:1206.3953,arXiv:1207.3859} where this is done for the simpler single layer case.

Concerning extensions of the present work, an important open case is
the one where the number of samples per dimension $\alpha = \Theta(1)$ and
also the size of the hidden layer per dimension $K/n = \Theta(1)$ as $n\to
\infty$, while in this paper we treated the case $K = \Theta(1)$ and $n\to
\infty$. This other scaling where $K/n = \Theta(1)$ is challenging
even for the non-rigorous replica method.

%% file: sections/supplementary/details-main-thm.tex
\section{Proof details for Theorem \ref{main-thm}}
\label{smproof}
\subsection{The Nishimori property in Bayes-optimal learning}
We first state an important property of the Bayesian optimal setting (that is when all hyper-parameters of the problem are assumed to be known), that is used several times, and is often referred to as the Nishimori identity. 
\begin{proposition}[Nishimori identity] \label{prop:nishimori}
	Let $(X,Y) \in \mathbb{R}^{n_1} \times \mathbb{R}^{n_2}$ be a couple of random variables. Let $k \geq 1$ and let $X^{(1)}, \dots, X^{(k)}$ be $k$ i.i.d.
	samples (given $Y$) from the conditional distribution $P(X=\cdot\, | Y)$, independently of every other random variables. Let us denote 
	$\langle - \rangle$ the expectation operator w.r.t.\ $P(X= \cdot\, | Y)$ and $\mathbb{E}$ the expectation w.r.t. $(X,Y)$. Then, for all continuous bounded function $g$ we have
	\begin{align}
	\mathbb{E} \langle g(Y,X^{(1)}, \dots, X^{(k)}) \rangle
	=
	\mathbb{E} \langle g(Y,X^{(1)}, \dots, X^{(k-1)}, X) \rangle\,.	
	\end{align}
\end{proposition}

\begin{proof}
	This is a simple consequence of Bayes formula.
	It is equivalent to sample the couple $(X,Y)$ according to its joint distribution or to sample first $Y$ according to its marginal distribution 
	and then to sample $X$ conditionally to $Y$ from its conditional distribution $P(X=\cdot\,|Y)$. Thus, the $(k+1)$-tuple $(Y,X^{(1)}, \dots,X^{(k)})$ is equal 
	in law to $(Y,X^{(1)},\dots,X^{(k-1)},X)$. This proves the proposition.
	\end{proof}

As a first application of Proposition \ref{prop:nishimori} we prove the following Lemma which is used in the proof of the upper bound Proposition \ref{upper}.
\begin{lemma}[Positivity of some matrices]\label{lemma:posMat}
The matrices $\rho$, $\mathbb{E}\langle Q \rangle$ and $\rho - \mathbb{E}\langle Q \rangle$ are positive definite, i.e. in $\mathcal{S}_K^+$. In the application the Gibbs bracket is $\langle - \rangle_{n, t, \epsilon}$. 
\end{lemma}
\begin{proof}
The statement for $\rho$ follows from its definition (in Theorem \ref{main-thm}). Note for further use that we also have 
$\rho= \frac{1}{n}\mathbb{E}[\underline W_i^* (\underline W_i^*)^\intercal]$.
Since by definition
$Q_{ll'} \equiv \frac{1}{n}\sum_{i=1}^n W_{il}^* w_{il'}$ in matrix notation we have $Q= \frac{1}{n} \sum_{i=1}^n\underline W_i^* \underline w_i^\intercal$.
An application of the Nishimori identity shows that 
\begin{align}
\mathbb{E} \langle Q\rangle = \frac{1}{n} \sum_{i=1}^n\mathbb{E}\langle \underline W_i^* \underline w_i^\intercal\rangle = 
\frac{1}{n} \sum_{i=1}^n\mathbb{E}[\langle \underline w_i\rangle \langle \underline w_i^\intercal\rangle]
\end{align}
which is obviously in $\mathcal{S}_K^+$. Finally, we note that 
\begin{align}
\mathbb{E} [\rho - \langle Q\rangle] &= 
\frac{1}{n}\sum_{i=1}^n\Big(\mathbb{E}[\underline W_i^* (\underline W_i^*)^\intercal] - \mathbb{E}[\langle \underline w_i\rangle \langle \underline w_i^\intercal\rangle]\Big)
\nonumber =
\frac{1}{n}\sum_{i=1}^n\mathbb{E}[(\underline W_i^* - \langle \underline w_i\rangle ) ((\underline W_i^*)^\intercal - \langle \underline w_i^\intercal\rangle)] 
\end{align}
where the last equality is proved by an application of the Nishimori identity again. This last expression is obviously in 
$\mathcal{S}_K^+$, i.e. $\mathbb{E} \langle Q\rangle\in \mathcal{S}_K^+(\rho)$. 
\end{proof}

\subsection{Setting in the Hamiltonian language}
We set up some notations which will shortly be useful. 
Let $u_y(\underline x) \equiv \ln P_{\rm out}(y\vert \underline{x})$. Here $\underline x\in \mathbb{R}^K$ and $y\in \mathbb{R}$. 
We will denote by $\nabla u_y(\underline x)$ the $K$-dimensional gradient w.r.t. $\underline x$, and  
$\nabla\nabla^\intercal u_y(\underline{x})$ the $K\times K$ matrix of second derivatives (the Hessian) w.r.t. $\underline x$.
Moreover, $\nabla P_{\rm out}(y\vert \underline x)$ and $\nabla\nabla^\intercal P_{\rm out}(y\vert \underline x)$ also denote the $K$-dimensional gradient and Hessian w.r.t. $\underline{x}$. 
We will also use the matrix identity
\begin{align}\label{iden-matrix}
	\nabla \nabla^\intercal u_{Y_\mu}( \underline x ) 
	+
	\nabla u_{Y_\mu} ( \underline x ) \nabla^\intercal u_{Y_\mu} ( \underline x )
	= \frac{\nabla \nabla^\intercal P_{\rm out}(Y_{\mu} | \underline x)}{P_{\rm out}(Y_{\mu} | \underline x)}\,. 	
\end{align}
Finally, we will use the matrices $w\in \mathbb{R}^{n\times K}$, $u\in \mathbb{R}^{m\times K}$, $Y_t\in \mathbb{R}^m$, $Y_t'\in \mathbb{R}^{n\times K}$,
$X\in \mathbb{R}^{m\times n}$, $V\in \mathbb{R}^{m\times K}$, $W^*\in \mathbb{R}^{n\times K}$ and 
$U^*\in \mathbb{R}^{m\times K}$. Like in sec. \ref{sec:proofsketch} we adopt the convention that all underlined vectors are $K$-dimensional, like e.g. $\underline u_\mu$, $\underline U_\mu$, $\underline V_\mu$ and 
$\underline Y'_{t,i}$.

It is convenient to reformulate the expression of the interpolating free entropy $f_{n,\epsilon}(t)$ in the Hamiltonian language. We introduce an {\it interpolating Hamiltonian}:
\begin{align}
  \mathcal{H}_t(w,u;Y_t,Y_t',X,V)
  \equiv
  - \sum_{\mu=1}^{m}
  u_{Y_{t,\mu}}( \underline{s}_{t, \mu} )  + \frac{1}{2} \sum_{i=1}^{n}\Vert\underline{Y}'_{t,i}  - {R_1(t)^{1/2}}\, \underline{w}_i\Vert_2^2 
  \label{SI-interpolating-ham}
\end{align}
where recall that
\begin{align}
\underline{s}_{t,\mu} \equiv \sqrt{\frac{1-t}{n}}\, \sum_{i=1}^n X_{\mu i} \underline{w}_i  + \sqrt{R_2(t)} 
\,\underline{V}_{\mu}  + \sqrt{t \rho - R_2(t) + 2s_n I_{K\times K}} \,\underline{u}_{\mu}\,. \label{littles}
\end{align}
The expression of $\mathcal{H}_t(W^*,U^*;Y_t,Y_t',X,V)$ is similar to \eqref{SI-interpolating-ham}, but with $w$ replaced by $W^*$ and $\underline{s}_{t,\mu}$ given by \eqref{littles} replaced by $\underline{S}_{t,\mu}$ given by \eqref{bigs}.
The average free entropy \eqref{f(t)} at time $t$ then reads 
\begin{align}\label{ft}
  f_{n,\epsilon}(t) & \equiv \frac{1}{n} \mathbb{E} 
  \ln \int_{\mathbb{R}^{n \times K}} dP_0(w) \int_{\mathbb{R}^{m \times K}}{\cal D}u \, 
  e^{-\mathcal{H}_t(w,u;Y_t,Y_t',X,V)}
\end{align}
where $\mathcal{D}u = \prod_{\mu=1}^m\prod_{l=1}^K (2\pi)^{-1/2}e^{-u_{\mu l}^2/2}$ and $dP_0(w) = \prod_{i=1}^n P_0(\underline{w}_{i}) \prod_{l=1}^K d{w}_{il}$. To develop the calculations in the simplest manner it is fruitful to represent the expectations over $W^*, U, Y, Y'$ explicitly as integrals:
\begin{align}
  f_{n,\epsilon}(t) = \frac{1}{n}\mathbb{E}_{X, V} & \int dY_t dY_t' dP_0(W^*) \mathcal{D}U^* e^{-\mathcal{H}_t(W^*,U;Y_t,Y_t',X,V)}
  \ln \int dP_0(w)  {\cal D}u \, 
  e^{-\mathcal{H}_t(w,u;Y_t,Y_t',X,V)}.
\end{align}

\subsection{Free entropy variation: Proof of Proposition \ref{prop:der_f_t}} \label{sec:proofderft}
The proof provided here follows very closely the one in \cite{barbier2017phase} for the case $K=1$, so we are more brief and refer to this paper for more details. We first prove that for all $t \in (0,1)$
\begin{align}
\frac{df_{n,\epsilon}(t)}{dt} = &- \frac{1}{2} 
		\mathbb{E} \Big\langle 
			{\rm Tr}\Big[\Big(
				\frac{1}{n} \sum_{\mu=1}^{m} \nabla u_{Y_{t,\mu}}(\underline s_{t,\mu})\nabla u_{Y_{t,\mu}}(\underline S_{t,\mu})^\intercal
				- r(t)
			\Big)
			\Big(
				\frac{1}{n} \sum_{i=1}^{n} \underline{W}^*_i \underline w_i^\intercal - q(t)
			\Big)
		\Big\rangle_{n,t,\epsilon}
		\nonumber \\ &
		\qquad \qquad + \frac12{\rm Tr}[r(t) (q(t) -\rho)]  -\frac{A_n}{2}\,,
		\label{eq:der_f_t_raw}	
\end{align}
where
\begin{align}
A_n = \mathbb{E} \Big[ {\rm Tr}\Big[
			\frac{1}{\sqrt{n}} \sum_{\mu=1}^{m} \frac{\nabla \nabla^\intercal P_{\rm out}(Y_{t,\mu} | \underline{S}_{t,\mu})}
			{P_{\rm out}(Y_{t,\mu} | \underline{S}_{t,\mu})} 
			\Big( \frac{1}{\sqrt{n}} \sum_{i=1}^{n} (\underline W^*_i (\underline W_i^*)^\intercal - \rho) \Big)\Big] \frac{1}{n} \ln \mathcal{Z}_{n,\epsilon}(t) \label{An}
		\Big]	\,.
\end{align}
Once this is done, we show that $A_n$ goes to $0$ as $n \to \infty$ uniformly in $t \in [0,1]$ in order to conclude the proof. 

The Hamiltonian \eqref{SI-interpolating-ham} $t$-derivative evaluated at the ground-truth matrices is given by
	\begin{align}
	\frac{d{\cal H}_t}{dt}(W^*, U^*; &Y_t, Y_t', X,V) 
	=
	- \sum_{\mu=1}^{m}
	 \nabla^\intercal u_{Y_{t,\mu}}( \underline{S}_{t,\mu}) \frac{d{\underline S}_{t,\mu}}{dt}
	-  \sum_{i=1}^{n} \Big(\frac{dR_1(t)^{1/2}}{dt}\underline{W}_i^{*}\Big)^\intercal (\underline{Y}'_{t,i}  - R_1(t)^{1/2} \underline{W}_i^*)
	\nonumber \\ &
	=
	- \sum_{\mu=1}^{m} {\rm Tr}\Big[ \frac{d{\underline S}_{t,\mu}}{dt} \nabla^\intercal u_{Y_{t,\mu}}( \underline{S}_{t,\mu})\Big]-  \sum_{i=1}^{n} {\rm Tr}\Big[\Big(\frac{dR_1(t)^{1/2}}{dt}\Big)^\intercal(\underline{Y}'_{t,i}  - R_1(t)^{1/2} \underline{W}_i^*)
	\underline{W}_i^{*\intercal}\Big]
	 \label{117_}
	\end{align}
	(where we used that $R_1(t)$ is symmetric).
	The $t$-derivative of $f_{n,\epsilon}(t)$ thus reads, for $0 < t < 1$,
	\begin{align}
		\frac{df_{n,\epsilon}(t)}{dt} = -\underbrace{\frac{1}{n}  \mathbb{E} \Big[\frac{d{\cal H}_t}{dt}(W^*, U^*; Y_t, Y'_t, X,V)\ln \mathcal{Z}_{n,\epsilon}(t)
		 \Big] }_{T_1}
		- \underbrace{\frac{1}{n} \mathbb{E} \Big\langle \frac{d{\cal H}_t}{dt}(w,u;Y_t,Y'_t,X,V) \Big\rangle_{n,t,\epsilon} }_{T_2}\label{106}.
	\end{align}

First, we note that $T_2=0$. This is a direct consequence of the Nishimori identity Proposition \ref{prop:nishimori}:
\begin{align}
		T_2 = \frac{1}{n} \mathbb{E} \Big\langle \frac{d{\cal H}_t}{dt}(w,u;Y_t,Y_t',X, V) \Big\rangle_{n,t,\epsilon} 
		= \frac{1}{n} \mathbb{E} \,\frac{d{\cal H}_t}{dt}(W^*,U^*;Y_t,Y'_t,X, V) = 0\,.
	\end{align}

We now compute $T_1$. 
Starting from \eqref{117_} and considering the first term only (recall also the expression \eqref{bigs} for ${\underline S}_{t,\mu}$),
	\begin{align}
	& \mathbb{E}\Big[ {\rm Tr}\Big[ \frac{d{\underline S}_{t,\mu}}{dt} \nabla^\intercal u_{Y_{t,\mu}}( \underline{S}_{t,\mu})\Big]  
	\ln \mathcal{Z}_{n,\epsilon}(t)  \Big] 
	=
	\mathbb{E} \Big[{\rm Tr}\Big[\Big\{
			- \frac{\sum_{i=1}^n X_{\mu i} \underline{W}_i^*}{2\sqrt{n (1-t)}}
			\nonumber \\ &
			\qquad+ \frac{d}{dt}\sqrt{R_{2}(t)} \underline V_{\mu}
			+ \frac{d}{dt}\sqrt{t \rho - R_2(t) + 2s_n I_{K\times K}}\, \underline U^*_{\mu}
	\Big\}\nabla^\intercal u_{Y_{t,\mu}}(\underline S_{t,\mu})\Big] \ln \mathcal{Z}_{n,\epsilon}(t)  \Big]\,. 
	\label{107}
	\end{align}
We then compute the first line of the right-hand side of \eqref{107}. By Gaussian integration by parts w.r.t. $X_{\mu i}$ (recall hypothesis \ref{hyp:3}), and using the identity \eqref{iden-matrix},
we find after some algebra
	\begin{align} 
		&-\frac{1}{2\sqrt{n(1-t)}}\mathbb{E} \Big[{\rm Tr}\Big[
			\sum_{i=1}^n X_{\mu i} \underline{W}_i^*
			\nabla^\intercal u_{Y_{t,\mu}} ( \underline S_{t,\mu} )\Big]
			\ln \mathcal{Z}_{n,\epsilon}(t) 
		\Big] 
\nonumber \\ 
&\qquad\qquad=
		-\frac{1}{2}\mathbb{E}\Big[{\rm Tr}\Big[
			\frac{1}{n} \sum_{i=1}^{n} 
			\underline W_i^* \underline W_i^\intercal
			\frac{\nabla \nabla^\intercal P_{\rm out}(Y_{t,\mu} | \underline S_{t,\mu})}{P_{\rm out}(Y_{t,\mu} | \underline S_{t,\mu})}\Big]
			\ln \mathcal{Z}_{n,\epsilon}(t) 
		\Big]
		\nonumber \\ &
		\qquad \qquad\qquad
		- \frac{1}{2}
		\mathbb{E}\Big\langle {\rm Tr}\Big[
			\frac{1}{n} \sum_{i=1}^{n} 
			\underline W_i^* \underline w_i^\intercal 
			\nabla u_{Y_{t,\mu}} ( \underline S_{t,\mu} )
			\nabla^\intercal u_{Y_{t,\mu}} ( \underline s_{t,\mu} )
			\Big]
		\Big\rangle_{n,t,\epsilon}\,.
		\label{eq:compA1}
	\end{align}
Similarly for the second line of the right-hand side of \eqref{107}, we use again Gaussian integrations by parts but this 
time w.r.t. $\underline{V}_\mu, \underline{U}_\mu^*$ which have i.i.d. $\mathcal{N}(0,1)$ entries. This calculation has to be done carefully with the help of
the matrix identity 
\begin{align}\label{SI-deriv}
 \frac{d}{dt} M(t) = \sqrt{M(t)}\, \frac{d \sqrt{M(t)}}{dt} + \frac{d \sqrt{M(t)}}{dt} \sqrt{M(t)}
\end{align}
for any $M(t)\in \mathcal{S}_K^+$, and the cyclicity and linearity of the trace. Applying \eqref{SI-deriv} to $M(t)$ equal to $\int_0^t q(s)ds$ and $\int_0^t (\rho - q(s))ds$, as well as the identity \eqref{iden-matrix}, we reach after some algebra
	\begin{align}\label{eq:compA2}
		 & \,\mathbb{E}\Big[{\rm Tr}\Big[
			\Big( \frac{d}{dt}\sqrt{R_{2}(t)} \underline V_{\mu}
			+ \frac{d}{dt}\sqrt{t \rho - R_2(t) + 2s_n I_{K\times K}}\, \underline U^*_{\mu}\Big)
			\nabla^\intercal u_{Y_{\mu}} ( \underline S_{\mu,t} )\Big]
			\ln \mathcal{Z}_{n,\epsilon}(t) 
		\Big]
\nonumber \\ 
		=
		&\,\mathbb{E}\Big[{\rm Tr}\Big[
			\rho \frac{\nabla \nabla^\intercal P_{\rm out}(Y_{t,\mu} | \underline S_{\mu,t})}{P_{\rm out}(Y_{t,\mu} | \underline S_{\mu,t})} \Big]
			\ln \mathcal{Z}_{n,\epsilon}(t) 
		\Big]
		+\mathbb{E}\Big\langle {\rm Tr}\Big[q(t) \nabla u_{Y_{t,\mu}}(\underline S_{\mu,t}) \nabla^\intercal u_{Y_{t,\mu}}(\underline s_{\mu,t})\Big]
		\Big\rangle_{n,t,\epsilon}\,.
	\end{align}
As seen from \eqref{117_}, \eqref{106} it remains to compute 
$\mathbb{E}[{\rm Tr}[(\frac{d}{dt}\sqrt{R_1(t)})^\intercal(\underline{Y}'_{t,i}  - \sqrt{R_1(t)} \underline{W}_i^*)
	\underline{W}_i^{*\intercal}]\ln \mathcal{Z}_{n,\epsilon}(t)]$.
Recall that $\underline{Y}'_{t,i}  - \sqrt{R_1(t)} \underline{W}_i^* = \underline Z'_i \sim\mathcal{N}(0,I_{K\times K})$. Using Gaussian integration by parts as well as the identity \eqref{SI-deriv} one obtains 
\begin{align}\label{finalT1}
	\mathbb{E}\Big[{\rm Tr}\Big[\Big(\frac{d}{dt}\sqrt{R_1(t)}\Big)^\intercal(\underline{Y}'_{t,i}  - \sqrt{R_1(t)} \underline{W}_i^*)
	\underline{W}_i^{*\intercal}\Big]\ln \mathcal{Z}_{n,\epsilon}(t)\Big]
	= - {\rm Tr} \Big[\sqrt{R_1(t)} \big(\rho - \mathbb{E} \langle W^*_j w_j \rangle_{n,t,\epsilon} \big)\Big] \,.
\end{align}
Finally, the term $T_1$ is obtained by putting together \eqref{107}, \eqref{eq:compA1}, \eqref{eq:compA2} and \eqref{finalT1}.

It now remains to check that $A_n\to 0$ as $n\to +\infty$ uniformly in $t\in [0,1]$. The proof from  \cite{barbier2017phase} (Appendix C.2) can easily be adapted, so we give here 
just a few indications for the ease of the reader. First one notices that 
\begin{align}
			\mathbb{E} \Big[ \frac{\nabla \nabla^\intercal P_{\rm out}(Y_{t,\mu} | \underline S_{t,\mu})}{P_{\rm out}(Y_{\mu} | \underline S_{t,\mu})}		
		\, \Big| \, W^*, \{\underline S_{t, \mu}\}_{\mu=1}^m \Big]
		=
		\int dY_\mu \nabla \nabla^\intercal P_{\rm out}(Y_{t,\mu}|\underline S_{t,\mu}) = 0\, , \label{117}
\end{align}
so that by the tower property of the conditional expectation one gets 
\begin{align}\label{tower}
 \mathbb{E} \Big[ {\rm Tr}\Big[
			\frac{1}{\sqrt{n}} \sum_{\mu=1}^{m} \frac{\nabla \nabla^\intercal P_{\rm out}(Y_{t,\mu} | \underline{S}_{t,\mu})}
			{P_{\rm out}(Y_{t,\mu} | \underline{S}_{t,\mu})} 
			\Big( \frac{1}{\sqrt{n}} \sum_{i=1}^{n} (\underline W^*_i (\underline W_i^*)^\intercal - \rho) \Big)\Big]\Big] = 0\,.
\end{align}
Next, one shows by standard second moment methods that $\mathbb{E}[(\ln\mathcal{Z}_{n,\epsilon}(t)/n - f_{n,\epsilon}(t))^2] \to 0$ as $n\to +\infty$ uniformly in $t\in [0,1]$ (see \cite{barbier2017phase} for the proof at $K=1$, that generalizes straightforwardly for any finite $K$). Then, using this last fact together with \eqref{tower}, and under hypotheses \ref{hyp:1}, \ref{hyp:2}, \ref{hyp:3}, an easy application of the Cauchy-Schwarz 
inequality implies $A_n\to 0$ as $n\to +\infty$ uniformly in $t\in [0,1]$. This ends the proof. $\QED$\\

\subsection{Technical lemmas} \label{sec:techLemmas}
\begin{lemma}[Cauchy-Lipschitz Theorem and Liouville Formula]\label{lemma:CL-liouville}
Let
	$$
	F : 
	\left|
	\begin{array}{ccc}
		[0,1] \times (0,+\infty)^d & \to & [0, + \infty)^d \\
		(t,z) & \mapsto & F(t,z)
	\end{array}
	\right.
	$$
	be a continuous, bounded function.
	Assume that $F$ admits continuous partial derivatives $\frac{\partial F}{\partial z_i}$ ($i=1,\ldots,d$) on its domain of definition.
	Then, for all $\epsilon \in (0,+\infty)^d$, the Cauchy problem
	\begin{equation} \label{eq:equadiff}
	y(0) = \epsilon
\qquad
\text{and}
\qquad
y'(t) = F\big(t,y(t)\big)
\end{equation}
admits a unique solution $t \mapsto y(t,\epsilon)$. For all $t \in [0,1]$, the mapping $z_t: \epsilon \mapsto y(t,\epsilon)$ is a diffeomorphism of class $\cC^1$, from $(0,+\infty)^d$ to $z_t\big((0,+\infty)^d\big)$. Moreover, the determinant $J(z_t)(\epsilon)$ of the Jacobian of $z_t$ at $\epsilon$ verifies
	\begin{equation}\label{eq:jac}
	J(z_t)(\epsilon)={\rm{det}}\Big(\Big(\frac{\partial y_i}{\partial \epsilon_j}\Big)_{i,j}\Big)
	= 
	\exp\Big(
		\int_0^t \sum_{i=1}^d\frac{\partial F_i}{\partial z_i}\big(s,y(s,\epsilon)\big) ds \Big)\,.
\end{equation}
Thus, in particular, if in addition $\sum_{i=1}^d\frac{\partial F_i}{\partial z_i} \geq 0$ then $J(z_t)(\epsilon) \geq 1$ for all $\epsilon$.
\end{lemma}
\begin{proof}
	The existence and uniqueness of the solution of \eqref{eq:equadiff} follows from the classical Cauchy-Lipschitz Theorem. The solution is indeed defined on all the segment $[0,1]$ because $F$ is bounded. 

	Theorem~3.1 from Chapter 5 in \cite{hartman1982ordinary} gives that $y$ admits continuous partial derivatives $\frac{\partial y}{\partial \epsilon_i}$ for $i=1,\ldots,d$, and Corollary~3.1 from Chapter 5 in the same reference states the Liouville formula \eqref{eq:jac}.

	By the Cauchy-Lipschitz Theorem, two solutions of $y'(t) = F\big(t,y(t)\big)$ that are equal at some $t \in [0,1]$ are equal everywhere. This implies that the mapping $z_t: \epsilon \mapsto y(t,\epsilon)$ is injective, for all  $t \in [0,1]$.
Since $y$ admits continuous partial derivatives in $\epsilon_i$, $i=1,\ldots,d$, we obtain that $z_t$ is of class $\cC^1$ on $(0,+\infty)^d$. Now, the equation \eqref{eq:jac} gives that $J(z_t)(\epsilon) > 0$ for all $\epsilon \in (0,+\infty)^d$. The local inversion Theorem gives then that $z_t$ is a $\cC^1$ diffeomorphism.
\end{proof}
\begin{lemma}[Boundedness of an overlap fluctuation]\label{Bounded-fluctuation}
Under hypothesis \ref{hyp:2} one can find a constant $C(\varphi, K,  \Delta) < +\infty$ (independent of $n, t, \epsilon$) such that for any $R_n\in \mathcal{S}_K^+$ we have  
\begin{align}
 \mathbb{E}\Big\langle  \Big\Vert \frac{1}{n} \sum_{\mu=1}^{m}\nabla u_{Y_{t,\mu}}(\underline{s}_{t,\mu}) \nabla u_{Y_{t,\mu}}(\underline{S}_{t,\mu})^\intercal
        - R_n
      \Big\Vert_{\rm F}^2
    \Big\rangle_{n,t,\epsilon} \leq 2{\rm Tr}(R_n^2) + \alpha^2 C(\varphi, K, \Delta).
\end{align}
We note that the constant remains bounded as $\Delta\to 0$ and diverges as $K\to +\infty$.
\end{lemma}

\begin{proof}
It is easy to see that for symmetric matrices $A$, $B$ we have ${\rm Tr}(A-B)^2 \leq 2({\rm Tr} A^2 + {\rm Tr} B^2)$. Therefore,
\begin{align}
 \mathbb{E}\Big\langle  & \Big\Vert  
        \frac{1}{n} \sum_{\mu=1}^{m}\nabla u_{Y_{t,\mu}}(\underline{s}_{t,\mu}) \nabla u_{Y_{t,\mu}}(\underline{S}_{t,\mu})^\intercal
        - R_n
      \Big\Vert_{\rm F}^2
    \Big\rangle_{n,t,\epsilon} 
    \nonumber \\ & 
    \qquad\qquad\leq 
    2{\rm Tr} (R_n^2) 
    + 
    2 \mathbb{E}\Big\langle  {\rm Tr }
        \Big(\frac{1}{n} \sum_{\mu=1}^{m}\nabla u_{Y_{t,\mu}}(\underline{s}_{t,\mu}) \nabla u_{Y_{t,\mu}}(\underline{S}_{t,\mu})^\intercal\Big)^2
        \Big\rangle_{n,t,\epsilon} .
\end{align}
In the rest of the argument we bound the second term of the r.h.s.
Using the triangle inequality and then Cauchy-Schwarz we obtain 
\begin{align}
 &\mathbb{E}\Big\langle  \Big\Vert \frac{1}{n} \sum_{\mu=1}^{m}\nabla u_{Y_{t,\mu}}(\underline{s}_{t,\mu}) \nabla u_{Y_{t,\mu}}(\underline{S}_{t,\mu})^\intercal \Big\Vert_{\rm F}^2
        \Big\rangle_{n,t,\epsilon}
        \leq \mathbb{E}\Big\langle \frac{1}{n^2}\Big(\sum_{\mu=1}^m \Vert \nabla u_{Y_{t,\mu}}(\underline{s}_{t,\mu}) \nabla u_{Y_{t,\mu}}(\underline{S}_{t,\mu})^\intercal\Vert_{\rm F}\Big)^2\Big\rangle_{n,t,\epsilon}
        \nonumber \\ 
        &\qquad\qquad\leq\mathbb{E}\Big\langle \frac{1}{n^2}
        \Big(\sum_{\mu=1}^m \Vert \nabla u_{Y_{t,\mu}}(\underline{s}_{t,\mu})\Vert_2 \Vert\nabla u_{Y_{t,\mu}}(\underline{S}_{t,\mu})^\intercal\Vert_2\Big)^2\Big\rangle_{n,t,\epsilon}.
        \label{bound-over-fluct}
\end{align}
From the random representation of the transition kernel, 
\begin{align}
u_{Y_{t, \mu}}(\underline s) & = \ln P_{\rm out}(Y_{t, \mu}|\underline x) = \ln \int dP_A(a_\mu) \frac{1}{\sqrt{2\pi \Delta}}e^{-\frac{1}{2\Delta} (Y_{t,\mu} - \varphi(\underline x, a_\mu))^2}
\end{align}
and thus 
\begin{align}
\nabla u_{Y_{t, \mu}}(\underline x) & = \frac{\int dP_A(a_\mu) (Y_{t,\mu} - \varphi(\underline x, a_\mu)) \nabla \varphi(\underline x, a_\mu) e^{-\frac{1}{2\Delta} (Y_{t,\mu} - \varphi(\underline x, a_\mu))^2}}
{\int dP_A(a_\mu) e^{-\frac{1}{2\Delta} (Y_{t,\mu} - \varphi(\underline x, a_\mu))^2}}
\end{align}
where $\nabla \varphi$ is the $K$-dimensional gradient w.r.t. the first argument $\underline x\in \mathbb{R}^K$. From the observation model we get 
$\vert Y_{t, \mu} \vert \leq \sup\vert\varphi\vert +\sqrt{\Delta}\vert Z_\mu\vert$, where the supremum is taken over 
both arguments of $\varphi$, and thus we immediately obtain for all $\underline s\in \mathbb{R}^K$
\begin{align}\label{bound-u}
\Vert\nabla u_{Y_{t, \mu}}(\underline x)\Vert \leq (2 \sup\vert \varphi\vert + \sqrt{\Delta}\vert Z_\mu\vert) \sup\Vert\nabla \varphi\Vert\,.
\end{align}
From \eqref{bound-u} and \eqref{bound-over-fluct} we see that it suffices to check that 
\begin{align*}
\frac{m^2}{n^2}\mathbb{E}\big[\big((2\sup\vert\varphi\vert + |Z_\mu|)^2(\sup\Vert\nabla \varphi\Vert)^2\big)^2\big] \leq C(\varphi, K, \Delta)
\end{align*}
where $C(\varphi, K, \Delta) < +\infty$ is a finite constant depending only on $\varphi$, $K$, and $\Delta$. This is easily seen by expanding all squares and using 
that $m/n \to \alpha$.
This ends the proof of Lemma \ref{Bounded-fluctuation}. 
\end{proof}

\begin{lemma}[Properties of $\psi_{P_0}$]\label{convexity_0}
$\psi_{P_0}$ is defined as the free entropy of the first auxiliary channel ~\eqref{aux-model-1}. We have, for any $r \in \mathcal{S}_K^+$:
\begin{align*}
\psi_{P_0}(r) &\equiv \mathbb{E} \ln \int_{\bbR^K} \mathrm{d}w P_0(w) e^{Y_0^\intercal r^{1/2} w - \frac{1}{2} w^\intercal r w} .
\end{align*}
Then $\psi_{P_0}$ is convex and differentiable on $\mathcal{S}_K^+$, with $\nabla \psi_{P_0}(r) \in \mathcal{S}_K^+$ for any $r \in \mathcal{S}_K^+$.
\end{lemma}

\begin{proof}
Note that $\psi_{P_0}$ is related to the mutual information $I(W_0;Y_0)$ via the relation $I(W_0;Y_0) = -\psi_{P_0}(r) + \frac{K}{2} + \frac{1}{2} \text{Tr}[r \rho]$. It is then a known result (see \cite{reeves2018mutual,payaro2011yet,lamarca2009linear}) that the derivative $\nabla_r I(W_0;Y_0)$ is given by the matrix-MMSE, i.e. $\nabla_r I(W_0;Y_0) = \frac{1}{2} \mathbb{E} \left[\braket{w} \braket{w}^\intercal\right]$. This implies that $\nabla_r \psi_{P_0}(r) = \frac{1}{2} (\rho - \mathbb{E} [\braket{w} \braket{w}^\intercal])$. Using the Nishimori identity Prop.\ref{prop:nishimori}, we can write it as $\nabla_r \psi_{P_0}(r) = \frac{1}{2} \mathbb{E} \left[(w-\braket{w})(w-\braket{w})^\intercal\right]$, which is clearly a positive matrix. It is also known (see for instance Lemma 4 of \cite{reeves2018mutual}), that $I(W_0;Y_0)$ is a concave function of $r$, which implies that $\psi_{P_0}$ is convex, which ends the proof.
\end{proof}

\begin{lemma}[Properties of $\Psi_{P_{\rm out}}$]\label{convexity}
Recall that $\Psi_{P_{\rm out}}$ is defined as the free entropy of the second auxiliary channel ~\eqref{aux-model-2}. More precisely, for $q \in \mathcal{S}_K^+(\rho)$, we have:
\begin{align*}
\Psi_{P_{\rm out}}(q) \equiv \mathbb{E} \ln \int_{\bbR^K} {d}w\frac{e^{-\frac{1}{2} \|w\|^2}}{(2 \pi)^{K/2}}  P_{\rm out} \big(\widetilde{Y}_0 | q^{1/2} V + (\rho - q)^{1/2}w\big).
\end{align*}
Then $\Psi_{P_{\rm out}}$ is continuous and convex on  $\mathcal{S}_K^+(\rho)$, and twice differentiable inside ${\mathcal{S}_K^+}(\rho)$. Also, $\nabla\Psi_{P_{\rm out}}(q)\in \mathcal{S}_K^+$.
\end{lemma}

\begin{proof}
The continuity and differentiability of $\Psi_{P_{\rm out}}$ is easy, and exactly similar to the first part of the proof of Proposition 18 of \cite{barbier2017phase}; it just follows from the hypothesis \ref{hyp:2} which allows to use continuity and differentiation under the expectation, because all the domination hypotheses are easily verified.  

One can compute the gradient and Hessian matrix of $\Psi_{P_{\rm out}}(q)$, for $q$ inside $\mathcal{S}_K^+(\rho)$, using Gaussian integration by parts and the Nishimori identity. The calculation is tedious and essentially follows the steps of Proposition 11 of \cite{barbier2017phase}. Recall that $u_{\widetilde Y_0}(x) \equiv \ln P_{\rm out}(\widetilde Y_0|x)$. We define the average $\braket{-}_{\rm sc}$ (where ${\rm sc}$ stands for ``scalar channel'') as
\begin{align}
\braket{g(w)}_{\rm sc} &\equiv \frac{\int_{\bbR^K} \mathcal{D}w P_{\rm out}(\widetilde Y_0|(\rho-q)^{1/2} w + q^{1/2} V) g(w)}{\int_{\bbR^K} \mathcal{D}w P_{\rm out}(\widetilde Y_0|(\rho-q)^{1/2} w + q^{1/2} V)},
\end{align}
for any continuous bounded function $g$. One arrives at:
\begin{align}\label{eq:derivative}
\nabla \Psi_{P_{\rm out}}(q) &= \frac{1}{2} \EE \Big<\nabla u_{\widetilde Y_0} \left((\rho-q)^{1/2} W^* + q^{1/2} V\right) \nabla u_{\widetilde Y_0} \left((\rho-q)^{1/2} w + q^{1/2} V\right)^\intercal \Big>_{\rm sc}.
\end{align}
Note that this gradient is actually a symmetric matrix of size $K \times K$, as it is a gradient w.r.t. $q$, which is itself a matrix of size $K$. The Hessian $\nabla \nabla^\intercal \Psi_{P_{\rm out}}$ with respect to $q$ is thus a $4$-tensor. One can compute in the same way:
\begin{align}
\nabla \nabla^\intercal \Psi_{P_{\rm out}}(q) = \frac{1}{2} \EE &\Big[ \Big( \Big<\frac{\nabla \nabla^\intercal P_{\rm out}(\widetilde Y_0|(\rho-q)^{1/2} w + q^{1/2} V)}{P_{\rm out}(\widetilde Y_0|(\rho-q)^{1/2} w + q^{1/2} V)}\Big>_{\rm sc}  \\
&  -\Big<\nabla u_{\widetilde Y_0} \left((\rho-q)^{1/2} W^* + q^{1/2} V\right) \nabla u_{\widetilde Y_0} \left((\rho-q)^{1/2} w + q^{1/2} V\right)^\intercal \Big>_{\rm sc} \Big)^{\otimes 2} \Big] \nonumber .
\end{align}
In this expression, $\otimes 2$ means the ``tensorized square'' of a matrix, i.e. for any matrix $M$ of size $K \times K$, $M^{\otimes 2}$ is a $4$-tensor with indices $M^{\otimes 2}_{l_0 l_1 l_2 l_3} = M_{l_0 l_1}M_{l_2 l_3}$. From this expression, it is clear that the Hessian of $\Psi_{P_{\rm out}}$ is always positive, when seen as a matrix with rows and columns in $\mathcal{S}_K$, and thus $\Psi_{P_{\rm out}}$ is convex, which ends the proof of Lemma~\ref{convexity}.
\end{proof}

%% file: sections/supplementary/replicas.tex
\section{Replica calculation}
\label{sec:replicacomputation}

Our goal here is to provide a heuristic derivation of the replica formula of Theorem~\ref{main-thm} using the replica method, a powerful non-rigorous tool from statistical physics of disordered systems \cite{mezard1987spin,mezard2009information}. This computation is necessary to properly ``guess'' the formula that we then prove using the adaptive interpolation method. The reader interested in the replica approach to neural networks and the committee machine is invited to look as well to some of the classical papers \cite{gardner1988optimal,mezard1989space,schwarze1992generalization,schwarze1993generalization,schwarze1993learning,monasson1995learning}. 

The replica trick makes use of the formula, for a random variable $x \in \bbR^n$ and a strictly positive function $f_n : \bbR^n \to \bbR$ that depends on $n$:
\begin{align}\label{eq:replicas_method}
\lim_{n \to \infty} \frac{1}{n} \mathbb{E} \ln f_n = \lim_{p \to 0^+} \lim_{n \to \infty} \frac{1}{n p} \ln \mathbb{E}f_n^{p}.
\end{align}

Note that the inversion of the two limits here is non-rigorous. Computing the moments $\mathbb{E}f^{p}$ can often be done for integers $p \in \bbN$, and one can conjecture from it its value for every $p > 0$, before taking the limit $p \to 0^+$ in \eqref{eq:replicas_method} by analytical continuation of the value for integer $p$. 

In our calculation, we will use this formula to compute the \emph{free entropy} of our system, $f \equiv \lim_{n \to \infty} f_n$. We will thus need the moments of the partition function, for integer $p$:
\begin{align*}
\EE \mathcal{Z}_n^p &= \EE \left[\int_{\bbR^{n} \times \bbR^K} dw \prod_{i=1}^n P_0\left(\right\{w_{il}\left\}_{l=1}^K\right) \prod_{\mu=1}^m P_{\rm out} \left(Y_\mu \Big|\left\{ \frac{1}{\sqrt{n}} \sum_{i=1}^n X_{\mu i} w_{il} \right\}_{l=1}^K\right)\right]^p, \\
&= \EE \left[\prod_{a=1}^p \int_{\bbR^{n} \times \bbR^K}  dw^a \prod_{i=1}^n P_0\left(\right\{w^a_{il}\left\}_{l=1}^K\right) \prod_{\mu=1}^m P_{\rm out} \left(Y_\mu \Big|\left\{ \frac{1}{\sqrt{n}} \sum_{i=1}^n X_{\mu i} w^a_{il} \right\}_{l=1}^K\right) \right].
\end{align*}

The outer expectation is done over $X_{\mu i} \sim \mathcal{N}(0,1)$, $w^\star$ and $Y$. Writing $w^\star$ as $w^0$ we have: 
\begin{align*}
\EE \mathcal{Z}_n^p  &= \EE_{X} \int_{\bbR^m} dY\prod_{a=0}^p\Bigg[ \int_{\bbR^{n} \times \bbR^K}  dw^a \prod_{i=1}^n P_0\left(\{w^a_{il}\}_{l=1}^K\right) \nonumber\\
&\qquad\qquad \qquad \qquad \times \prod_{\mu=1}^m P_{\rm out} \left(Y_\mu \Big|\left\{ \frac{1}{\sqrt{n}} \sum_{i=1}^n X_{\mu i} w^a_{il} \right\}_{l=1}^K\right) \Bigg]. 
\end{align*}

To perform the average over $X$, we notice that, since it is an i.i.d. standard Gaussian matrix, then for every $a,\mu,l$, $Z^a_{\mu l} \equiv n^{-1/2} \sum_{i=1}^n X_{\mu i} w_{i l}^a$ follows a Gaussian multivariate distribution, with zero mean. This naturally leads to introduce its covariance tensor, which is equal to:
\begin{align}
\EE Z^{a}_{\mu l} Z^{b}_{\nu l'} &= \delta_{\mu \nu} \Sigma_{\substack{a l \\ b l'}} =  \delta_{\mu \nu} Q^{a l}_{b l'}, \\
Q^{a l}_{b l'} &\equiv \frac{1}{n} \sum_{i=1}^n w_{i l}^a w_{i l'}^b.
\end{align}
For every $a,b$, $Q^a_b \in \bbR^{K \times K}$ is the \emph{overlap} matrix, and $\Sigma$ is of size $(p+1)K \times (p+1)K$. Introducing $\delta$ functions for fixing $Q$, we arrive at : 
\begin{align}
\mathbb{E} \left[\mathcal{Z}_n^p \right] &= \prod_{(a,r)} \int_{\mathbb{R}} d Q^{a r}_{a r} \prod_{\{(a,r);(b,r')\}} \int_{\mathbb{R}} d Q^{a r}_{b r'}  \left[I_{\text{prior}}(\{Q^{a r}_{b r'}\}) \times I_{\text{channel}}(\{Q^{a r}_{b r'}\}) \right],
\end{align}
with:
\begin{align}
I_{\text{prior}} (\{Q^{a r}_{b r'}\}) &= \prod_{a=0}^p \left[\int_{\mathbb{R}^{n \times K}} d w^a  P_0(w^a) \right] \left[\prod_{\{(a,l);(b,l')\}} \delta \left(Q^{a l}_{b l'} - \frac{1}{n} \sum_{i=1}^n w_{i l}^a w_{i l'}^b \right)\right], \\
I_{\text{channel}} (\{Q^{a r}_{b r'}\}) &= \int_{\mathbb{R}^m} d Y \prod_{a=0}^p \int_{\mathbb{R}^{m \times K}} dZ^a\prod_{a=0}^p P_{\rm out}(Y | Z^a) e^{- \frac{m}{2} \ln \det \Sigma - \frac{m K(p+1)}{2} \ln 2\pi} \nonumber \\
&\exp \left[- \frac{1}{2} \sum_{\mu =1}^m \sum_{a,b} \sum_{l,l'} Z^{a}_{\mu l} Z^b_{\mu l'}(\Sigma^{-1})_{\substack{a l \\ b l'}} \right]. 
\end{align}

By Fourier expanding the delta functions in $I_{\rm prior}$, and performing a saddle-point method, one obtains:
\begin{align}
\lim_{n \to \infty} \frac{1}{n} \ln \mathbb{E} \left[\mathcal{Z}_n^p \right] = \text{extr}_{Q,\hat{Q}} \left[H(Q,\hat{Q})\right],
\end{align}
in which (recall $\alpha \equiv \lim_{n \to \infty} m/n$) :
\begin{align}
H(Q,\hat{Q}) &\equiv \frac{1}{2} \sum_{a=0}^p \sum_{l,l'} Q^{a l}_{a l} \hat{Q}^{a l}_{a l} - \frac{1}{2} \sum_{a \neq b} \sum_{l, l'} Q^{a l}_{b l'} \hat{Q}^{a l}_{b l'} + \ln I + \alpha \ln J, 
\end{align}
in which we defined:
\begin{align}
I &\equiv \prod_{a=0}^p \int_{\mathbb{R}^{K}} dw ^a P_0(w^a) \exp \left[-\frac{1}{2} \sum_{a=0}^p \sum_{l,l'}  \hat{Q}^{a l}_{a l'} w_l^a w_{l'}^a  + \frac{1}{2} \sum_{a \neq b} \sum_{l,l'} \hat{Q}^{a l}_{b l'} w_l^a w_{l'}^b \right], \\
J &\equiv \int_{\mathbb{R}} d y \prod_{a=0}^p \int_{\mathbb{R}^{K}} \frac{d Z^a}{\left(2 \pi\right)^{K (p+1)/2}} \frac{P_{\rm out}(y |Z^a)}{\sqrt{\det \Sigma}} \exp \left[- \frac{1}{2} \sum_{a,b=0}^p \sum_{l, l'=1}^K Z^{a}_{l} Z^b_{l'}(\Sigma^{-1})_{\substack{a l \\ b l'}} \right].
\end{align}

Our goal is to express $H(Q,\hat{Q}) $ as an analytical function of $p$, in order to perform the replica trick. To do so, we will assume that the extremum of $H$ is attained at a point in $Q,\hat{Q}$ space such that a \emph{replica symmetry} property is verified. More concretely, we assume: 
\begin{align}
\exists Q^0 \in \mathbb{R}^{K \times K} \text{ s.t. }  \quad \forall a \in [|0,p|] \quad \forall (l,l') \in [|1,K|]^2 \quad Q^{a l}_{a l'} &= Q^0_{l l'}, \\
\exists q \in \mathbb{R}^{K \times K} \text{ s.t. } \quad \forall (a < b) \in [|0,p|]^2 \quad \forall (l,l') \in [|1,K|]^2 \quad Q^{a l}_{b l'} &= q_{l l'},
\end{align}
and similarly for $\hat{Q}^0$ and $\hat{q}$. Note that $Q^0$ is by definition a symmetric matrix, while $q$ is also symmetric by our assumption of replica symmetry. Under this ansatz, we obtain: 
\begin{align}\label{eq:H}
H(Q^0,\hat{Q}^0,q,\hat{q}) &= \frac{p+1}{2} \text{Tr} [Q^0 \hat{Q}^0 ] - \frac{p(p+1)}{2} \text{Tr} [q \hat{q}] + \ln I + \alpha \ln J. 
\end{align} 

Remains now to compute an expression for $I$ and $J$ that is analytical in $p$, in order to take the limit $p \to 0^+$. This can be done easily, using the identity, for any symmetric positive matrix $M \in \bbR^{K \times K}$ and any vector $x \in \bbR^K$:
$\exp\left(x^\intercal  (M/2) x\right) = \int_{\bbR^K} \mathcal{D}\xi \, \exp\left(\xi^\intercal  M^{1/2} x\right)$, in which $ \mathcal{D}\xi$ is the standard Gaussian measure on $\bbR^K$. We obtain:
\begin{align}
I &= \int_{\mathbb{R}^K} \mathcal{D}\xi \left[\int_{\mathbb{R}^{K}} dw \, P_0(w) \, \exp \left[-\frac{1}{2} w^\intercal  (\hat{Q}^0 + \hat{q}) w + \xi^\intercal  \hat{q}^{1/2} w \right] \right]^{p+1}, \\
J &= \int_{\mathbb{R}} d y \int_{\mathbb{R}^K} \mathcal{D}\xi  \left[ \int_{\mathbb{R}^{K}} dZ  P_{\rm out}\left\{y | (Q^0-q)^{1/2}Z + q^{1/2} \xi \right\} \right]^{p+1}.
\end{align}
  Our assumptions must be consistent in the sense that $ \text{extr}_{Q,\hat{Q}} \left[\lim_{p \to 0^+} H(Q,\hat{Q})\right] = 0$ (because $\EE \mathcal{Z}_n^0 = 1$). In the $p \to 0^+$ limit, one easily gets $J = 1$ and $I = \int_{\mathbb{R}^{K}} d w \, P_0(w) \exp\left[-\frac{1}{2} w^\intercal  \hat{Q}^0 w^0 \right]$.  This implies that the optimal overlap parameters satisfy $\hat{Q}^0 = 0$ and $Q^0_{ll'} = \EE_{P_0}\left[w_l w_{l'}\right]$. In the end, we obtain the final formula for the free entropy: 
  
\begin{align}\label{eq:replica_solution}
\lim_{n \to \infty} f_n &= \text{extr}_{q,\hat{q}}\left\{- \frac{1}{2} \text{Tr} [q \hat{q}] + I_P + \alpha I_C \right\}, \\ 
I_P &\equiv \int_{\mathbb{R}^K} \mathcal{D}\xi \int_{\mathbb{R}^{K}} d w^0 P_0(w^0) \exp \left[-\frac{1}{2} (w^0)^\intercal   \hat{q} w^0 +\xi^\intercal  \hat{q}^{1/2} w^0 \right] \nonumber\\
&\qquad\qquad\qquad\qquad\times\ln \left[\int_{\mathbb{R}^{K}} d w P_0(w) \exp \left[-\frac{1}{2} w^\intercal   \hat{q} w  +\xi^\intercal  \hat{q}^{1/2} w \right] \right] \nonumber, \\
I_C &\equiv \int_{\mathbb{R}} d y \int_{\mathbb{R}^K} \mathcal{D}\xi \int_{\mathbb{R}^{K}} \mathcal{D}Z^0 P_{\rm out}\left\{y |(Q^0 - q)^{1/2}Z^0 +q^{1/2}\xi\right\}\nonumber\\
&\qquad\qquad\qquad\qquad\times\ln \left[\int_{\mathbb{R}^{K}} \mathcal{D}Z P_{\rm out}\left\{y | (Q^0 - q)^{1/2}Z +q^{1/2}\xi\right\} \right] \nonumber.
\end{align}  

A known ambiguity of the replica method is that its result is given as an extremum, here over the set $\mathcal{S}_K^+(Q_0)$ of positive symmetric matrices, such that $(Q^0 - q)$ is also a positive matrix. It is easy to show that this form gives back the form given in Theorem~\ref{main-thm}, by assuming that this extremum is realized as a $\sup_{\hat{q}} \inf_{q}$. Note that in the notations of Theorem~\ref{main-thm}, $Q^0$ is denoted $\rho$ and $\hat{q}$ is denoted $R$.

%% file: sections/supplementary/generalization_error.tex
\section{Generalization error}\label{sec:generalization}

We detail here two different possible definitions of the generalization error, and how they are related in our system. Recall that we wish to estimate $W^*$ from the observation of $\varphi_{\rm out}(XW^*)$. In the following, we denote $\EE$ for the average over the (quenched) $W^*$ and the data $X$, and $\braket{-}$ for the Gibbs average over the posterior distribution of $W$. One can naturally define the \emph{Gibbs generalization error} as:
\begin{align}\label{eq:Gibbs_gen_error}
\epsilon_g^{\rm Gibbs} &\equiv \frac{1}{2} \EE_{W^*,X} \big\langle\left[\varphi_{\rm out}\left(X W\right) - \varphi_{\rm out}\left(X W^* \right) \right]^2\big\rangle,
\end{align}
and define the \emph{Bayes-optimal generalization error 
} as:
\begin{align}\label{eq:def_eg_Bayes}
\epsilon_g^{\rm Bayes} &\equiv \frac{1}{2} \EE_{W^*,X} \big[\big(\braket{\varphi_{\rm out}\left(X W\right)} - \varphi_{\rm out}\left(X W^* \right)\big)^2 \big].
\end{align}
Using the Nishimori identity \ref{prop:nishimori}, one can show that:
\begin{align*}
\epsilon_g^{\rm Bayes} &= \frac{1}{2} \EE_{X,W^*} \left[ \varphi_{\rm out}\left(X W^* \right)^2\right] + \frac{1}{2} \EE_{X,W^*} \left[ \braket{ \varphi_{\rm out}\left(X W \right)}^2\right] \nonumber\\
&\qquad\qquad\qquad\qquad- \EE_{X,W^*} \braket{ \varphi_{\rm out}\left(X W^* \right) \varphi_{\rm out}\left(X W \right)} , \\
&= \frac{1}{2} \EE_{X,W^*} \left[ \varphi_{\rm out}\left(X W^* \right)^2\right]  - \frac{1}{2} \EE_{X,W^*} \braket{ \varphi_{\rm out}\left(X W^* \right) \varphi_{\rm out}\left(X W \right)}.
\end{align*}
Using again the Nishimori identity one can write:
\begin{align*}
\epsilon_g^{\rm Gibbs} &= \EE_{X,W^*} \left[ \varphi_{\rm out}\left(X W^* \right)^2\right]  -  \EE_{X,W^*}\braket{ \varphi_{\rm out}\left(X W^* \right) \varphi_{\rm out}\left(X W \right)},
\end{align*}
which shows that $\epsilon_g^{\rm Gibbs} = 2 \epsilon_g^{\rm Bayes}$. Note finally that since the distribution of $X$ is rotationally invariant, the quantity $\EE_{X} \left[\varphi_{\rm out}\left(X W^* \right) \varphi_{\rm out}\left(X W \right)\right]$ only depends on the \emph{overlap} $q \equiv W^\intercal W^*$. As the overlap is shown to concentrate under the Gibbs measure by Proposition \ref{concentration}, and as we expect that the value it concentrates on is the optimum $q^*$ of the replica formula (such fact is proven, e.g., for random linear estimation problems in \cite{barbier_ieee_replicaCS}), the generalization error can itself be evaluated as a function of $q^*$. Examples where it is done include  \cite{opper1996statistical,seung1992statistical,schwarze1993learning,barbier2017phase}.

\subsection{The generalization error at \texorpdfstring{$K = 2$}{K=2}}

In this subsection alone, we go back to the $K=2$ case, instead of the $K \to \infty$ limit. From the definition of the generalization error (see sec.~\ref{sec:generalization}), one can directly give an explicit expression of this error in the $K = 2$ case. Recall our committee-symmetric assumption on the overlap matrix, which here reads 
\begin{align*}
q &= \begin{pmatrix}
q_d + \frac{q_a}{2} & \frac{q_a}{2} \\
 \frac{q_a}{2} & q_d + \frac{q_a}{2} 
\end{pmatrix}.
\end{align*} 
For concision, we denote here $\sign(x) = \sigma(x)$. One obtains from \eqref{eq:def_eg_Bayes}:
\begin{align}
\frac{1}{2}& -2 \epsilon_g^{\rm Bayes,K=2} = \int_{\bbR^4} \mathcal{D}x\, \sigma\left[\sigma(x_1) + \sigma(x_2) \right] \\
& \times \sigma\left\{\sigma\left[(\frac{q_a}{2} + q_d)x_1  +  \frac{q_a}{2} x_2 + x_3 \sqrt{1-\frac{q_a^2}{2} - q_a q_d - q_d^2}\right] \right. \nonumber \\
& \left. \quad + \sigma\left[\frac{q_a}{2} x_1 +(\frac{q_a}{2} + q_d)x_2  - x_3 \frac{q_a (q_d + \frac{q_a}{2})}{\sqrt{1-\frac{q_a^2}{2} - q_a q_d - q_d^2}} +x_4 \sqrt{\frac{(1-q_d^2)(1-(q_a+q_d)^2)}{1-\frac{q_a^2}{2} - q_a q_d - q_d^2}}\right] 
\right\} .\nonumber
\end{align}
Note that one could possibly simplify this expression by using an appropriate orthogonal transformation on $x$. These integrals were then computed using Monte-Carlo methods to obtain the generalization error in the left and middle plots of Fig.~\ref{fig:phaseDiagramK2}.

%% file: sections/supplementary/large_K.tex
\section{The large \texorpdfstring{$K$}{K} limit in the committee symmetric setting} 
\label{sec:largeK}

We consider the large $K$ limit\footnote{A similar limit has been derived in the context of coding with sparse superposition codes \cite{Barbier2017ApproximateMD}. There the large input alphabet limit of the mutual information is considered {\it after} the thermodynamic limit $n\to\infty$ corresponding to the large codeword limit in this coding context.} for a sign activation function, and for different priors on the weights. Since the output is a sign, the channel is simply a delta function. We assume a committee symmetric solution, i.e. the matrices $q$ and $\hat{q}$ ($q$ and $R$ in the notations of Theorem~\ref{main-thm}) are of the type $q = q_d \mathds{1}_K + \frac{q_a}{K} \textbf{1}_K \textbf{1}_K^\intercal$, with the unit vector $\textbf{1}_K = (1)_{l=1}^K$, and similarly for $\hat{q}$. In the large $K$ limit, this scaling of the order parameters is natural. Indeed, assume that the covariance of the prior is $Q^0 = \mathds{1}_K $ ($Q^0 = \rho$ in the notations of Theorem~\ref{main-thm}). Since both $q$ and $(Q^0-q)$ are assumed to be positive matrices, it is easily shown to imply that $0\leq q_d \leq 1$ and $0 \leq q_a + q_d \leq 1$.

\subsection{Large \texorpdfstring{$K$}{K} limit for sign activation function}\label{sec:large_K_channel}

In the following, we consider $Q^0 = \sigma^2 \mathds{1}_K $. We are interested here in computing the leading order term in $I_C$ of \eqref{eq:replica_solution}. Note that replacing $\sigma^2$ by $1$ in this equation only amounts to replacing $q$ by $q/\sigma^2$, so we can assume $\sigma^2 = 1$ without loss of generality. We (abusively) write $I_C$ in \eqref{eq:replica_solution} as $I_C = \sum_{y = \pm 1} \int_{\bbR^K} \mathcal{D}\xi \, I_C(y,\xi) \log I_C(y,\xi)$, with the definition
\begin{align}
I_C(y,\xi) \equiv \int_{\mathbb{R}^{K}} \mathcal{D}Z P_{\rm out}\left\{y | (Q^0 - q)^{1/2}Z +q^{1/2}\xi\right\}.
\end{align} 

Here, we assumed a sign activation function and no noise, as well as a particular form for $Q_0$ and $q$ (see the remarks above). Note that this implies that $q^{1/2} = \sqrt{q_d} \mathds{1}_K  + \frac{\sqrt{q_a+q_d}-\sqrt{q_d}}{K} \textbf{1}_K \textbf{1}_K^\intercal$ and that $(Q_0-q)^{1/2} = \sqrt{1-q_d} \mathds{1}_K  + \frac{\sqrt{1-q_a-q_d}-\sqrt{1-q_d}}{K} \textbf{1}_K \textbf{1}_K^\intercal$. All together, this gives the following explicit expression for $I_C(y,\xi)$ : 
\begin{align*}
&I_C(y,\xi) \equiv \int_{\mathbb{R}^{K}} \mathcal{D}Z \nn
&\qquad\times\delta\left\{y - \text{sign}\left[\frac{1}{\sqrt{K}} \sum_{l=1}^K \text{sign}\left[\sqrt{1-q_d} Z_l + \left(\sqrt{1-q_a-q_d}-\sqrt{1-q_d}\right) \frac{\textbf{1}_K^\intercal Z}{K} + (q^{1/2} \xi)_l \right] \right]\right\}.
\end{align*} 
Introducing a new variable $w \equiv \frac{\textbf{1}_K^\intercal Z}{\sqrt{K}}$ and a Fourier-transform of the then-introduced delta function, as well as another variable $u$ being the argument of the outer sign function in the previous equations, one obtains:

\begin{align*}
I_C(y,\xi) &= \int_\bbR \frac{dw d\hat{w}}{2 \pi} \frac{du d\hat{u}}{2 \pi} e^{i w \hat{w} + i u \hat{u}} \delta_{y,\text{sign}(u)} \nonumber\\
&\qquad\qquad\times\prod_{l=1}^K  \int_{\bbR} \mathcal{D}z e^{- i \hat{w} \frac{z}{\sqrt{K}}} e^{- \frac{i \hat{u}}{\sqrt{K}} \text{sign} \left[ z +  \left[\sqrt{\frac{1-q_a-q_d}{1-q_d}} - 1\right]  \frac{w}{\sqrt{K}}  + \frac{1}{\sqrt{1-q_d}}(q^{1/2} \xi)_l\right]} .
\end{align*}

Denote \begin{align*}
\lambda_l(w,\xi) \equiv \left[\sqrt{\frac{1-q_a-q_d}{1-q_d}} - 1\right]  \frac{w}{\sqrt{K}}  + \frac{1}{\sqrt{1-q_d}}(q^{1/2} \xi)_l,
\end{align*}
such that 
\begin{align*}
I_C(y,\xi) &= \int_\bbR \frac{dw d\hat{w}}{2 \pi} \frac{du d\hat{u}}{2 \pi} e^{i w \hat{w} + i u \hat{u}} \delta_{y,\text{sign}(u)} \prod_{l=1}^K  \int_{\bbR} \mathcal{D}z e^{- i \hat{w} \frac{z}{\sqrt{K}}} e^{- \frac{i \hat{u}}{\sqrt{K}} \text{sign} \left[ z +  \lambda_l(w,\xi)\right]} .
\end{align*}
For $1 \leq l \leq K$, one can rewrite the factorized integral in the last expression of $I_C(y,\xi)$ as:
\begin{align}
\label{eq:IC_decomposition}
I_C(y,\xi) &= \int_\bbR \frac{dw d\hat{w}}{2 \pi} \frac{du d\hat{u}}{2 \pi} e^{i w \hat{w} + i u \hat{u}} \delta_{y,\text{sign}(u)} \prod_{l=1}^K  J\left(\lambda_l(w,\xi),\hat{w},\hat{u}\right),
\\
J\left(\lambda_l(w,\xi),\hat{w},\hat{u}\right) &\equiv e^{-\frac{\lambda_l^2}{2} + i \lambda_l \frac{\hat{w}}{\sqrt{K}}} \int_{\bbR} \mathcal{D}z e^{z (\lambda_l - i \frac{\hat{w}}{\sqrt{K}})} e^{- \frac{i \hat{u}}{\sqrt{K}} \text{sign} \left[ z \right]}.
\end{align}
We abusively dropped the dependency of $\lambda_l$ on $(w,\xi)$. Note the following identity:
\begin{align}
F(\alpha,i\beta) \equiv \int_{\bbR}  \mathcal{D} z e^{\alpha z + i\beta \sign(z)} &= e^{\alpha^2/2} \left[\cos \beta + i \sin \beta \hat{H}(\alpha)\right],
\end{align}
with $\hat{H}(x) = \text{erf}(x/\sqrt{2})$. Using it in our previous expressions, we obtain:
\begin{align*}
J(\lambda_l,\hat{w},\hat{u}) &= e^{-\frac{1}{2K} \hat{w}^2} \left[\cos \left(\frac{\hat{u}}{\sqrt{K}}\right) - i\sin \left(\frac{\hat{u}}{\sqrt{K}}\right) \hat{H} \left(\lambda_l - i \frac{\hat{w}}{\sqrt{K}}\right) \right]. 
\end{align*}
Note that by our committee-symmetry assumption, we have $\lambda_l(w,\xi) = \lambda_{l,0}(\xi) + \frac{1}{\sqrt{K}} \lambda_{1}(w,\xi)$ with $\lambda_{l,0}$ and $\lambda_{1}$ typically of order $1$ when $K \to \infty$:
\begin{align}\label{eq:lambda_l0_dec}
\lambda_{l,0}(\xi) &\equiv \sqrt{\frac{q_d}{1-q_d}} \xi_l, \\
\label{eq:lambda_l1_dec}
\lambda_{1}(w,\xi) &\equiv \left[\sqrt{\frac{1-q_a-q_d}{1-q_d}} - 1\right]  w + \left[\sqrt{\frac{q_a+q_d}{1-q_d}} - \sqrt{\frac{q_d}{1-q_d}}\right] \frac{\textbf{1}_K^\intercal \xi}{\sqrt{K}}. 
\end{align}
Expanding $J(\lambda_l,\hat{w},\hat{u})$ as $K \to \infty$, we obtain using the known development of the error function:
\begin{align*}
J(\lambda_l,\hat{w},\hat{u}) &= e^{-\frac{1}{2K} \hat{w}^2} \left[1 - \frac{\hat{u}^2}{2 K} - i  \hat{H} \left[\lambda_{l,0}(\xi) \right]\frac{\hat{u}}{\sqrt{K}} - i \frac{\hat{u}\left[\lambda_1(w,\xi) - i \hat{w}\right]}{K} \sqrt{\frac{2}{\pi}} e^{-\frac{\lambda_{l,0}(\xi)^2}{2}} + \mathcal{O}(K^{-3/2}) \right].
\end{align*}
This yields (putting back the $(w,\xi)$ dependency): 
\begin{align}\label{eq:expansion_J}
\prod_{l=1}^K J\left[\lambda_l(w,\xi),\hat{w},\hat{u}) \right]&= e^{-\frac{1}{2} \hat{w}^2} \exp \left[- \frac{\hat{u}^2}{2} - i \hat{u} S_1 - i \sqrt{\frac{2}{\pi}}\hat{u} (\lambda_1-i\hat{w})\Gamma_0 + \frac{1}{2} \hat{u}^2 S_2 +  \mathcal{O}(K^{-1/2})\right],
\end{align}
in which we defined the following quantities, that only depend on $\xi$ (recall \eqref{eq:lambda_l0_dec})
\begin{align*}
w_\xi(\xi) &\equiv \frac{1}{\sqrt{K}} \sum_{l=1}^K \xi_l,   & \Gamma_0(\xi) &\equiv \frac{1}{K} \sum_{l=1}^K e^{-\frac{1}{2} \lambda_{l,0}(\xi)^2}, \\ \quad S_1(\xi) &\equiv \frac{1}{\sqrt{K}} \sum_{l=1}^K \hat{H}(\lambda_{l,0}(\xi)),     & S_2(\xi) &\equiv \frac{1}{K} \sum_{l=1}^K \hat{H}(\lambda_{l,0}(\xi))^2.
\end{align*}
A detailed calculation actually shows that the previous expansion of \eqref{eq:expansion_J} is valid up to $\mathcal{O}(K^{-1})$, and not only $\mathcal{O}(K^{-1/2})$. Recall also \eqref{eq:IC_decomposition}, in which one can now readily perform the integration over all variables $w,\hat{w},u,\hat{u}$ to obtain (dropping the $\xi$ dependency in $w_\xi, \Gamma_0, S_1, S_2$):
\begin{align}\label{eq:IC_expanded}
I_C(y,\xi) &= H \left[- y \frac{S_1+ \sqrt{\frac{2}{\pi}}  w_\xi \Gamma_0 \frac{\sqrt{q_d+q_a}-\sqrt{q_d}}{\sqrt{1-q_d}}}{\sqrt{1-S_2 -\frac{2}{\pi} \Gamma_0^2 \frac{q_a}{1-q_d}}}\right] + \mathcal{O}(K^{-1}),
\end{align}
in which $H(x) \equiv \int_x^\infty \mathcal{D}z = \frac{1}{2} \left[1-\text{erf}(x/\sqrt{2})\right]$. Note that all quantities $w_\xi, \Gamma_0, S_1, S_2$ only depend on $\xi$ via its empirical measure, which implies that the integration over $\xi \in \bbR^K$ will be tractable. We compute it in the following, using theoretical physics methods. We denote the quantity that appears in \eqref{eq:IC_expanded} as a function of $w_\xi, \Gamma_0, S_1, S_2$:
\begin{align*}
G(y,w_\xi,\Gamma_0,S_1,S_2) &\equiv  H \left[- y \frac{S_1+ \sqrt{\frac{2}{\pi}}  w_\xi \Gamma_0 \frac{\sqrt{q_d+q_a}-\sqrt{q_d}}{\sqrt{1-q_d}}}{\sqrt{1-S_2 -\frac{2}{\pi} \Gamma_0^2 \frac{q_a}{1-q_d}}}\right].
\end{align*}
Introducing once again delta functions and their Fourier transforms for $w_\xi, \Gamma_0, S_1, S_2$, we write, starting from \eqref{eq:IC_expanded}:
\begin{align}\label{eq:eq_I_expansion_step}
I_C &= \sum_{y = \pm 1}\int_{\mathbb{R}^K} \mathcal{D}\xi I_C (y,\xi) \log  I_C(y,\xi) \nonumber\\
&= \sum_{y = \pm 1} \int \frac{dw_\xi d\hat{w}_\xi}{2 \pi} \frac{d\Gamma_0 d\hat{\Gamma}_0}{2 \pi} \frac{dS_1 d\hat{S_1}}{2 \pi} \frac{dS_2 d\hat{S_2}}{2 \pi} e^{i w \hat{w} + i \Gamma_0 \hat{\Gamma}_0 + i S_1 \hat{S}_1 + i S_2 \hat{S}_2}  \, G(y,w_\xi,\Gamma_0,S_1,S_2) \nonumber \\
& \qquad\qquad   \times\log G(y,w_\xi,\Gamma_0,S_1,S_2) \left[\int_{\bbR^K} \mathcal{D} \xi e^{- i \hat{w} w_\xi(\xi) - i \hat{\Gamma}_0 \Gamma_0(\xi) - i \hat{S}_1 S_1(\xi) - i \hat{S}_2 S_2(\xi)}\right]  + \mathcal{O}(K^{-1}).
\end{align}
The integral over $\xi$ in  \eqref{eq:eq_I_expansion_step} can be computed in the limit $K \to \infty$: 
\begin{align*}
\Lambda &\equiv \int_{\bbR^K} \mathcal{D} \xi e^{- i \hat{w} w_\xi(\xi) - i \hat{\Gamma}_0 \Gamma_0(\xi) - i \hat{S}_1 S_1(\xi) - i \hat{S}_2 S_2(\xi)} \\
&= \left[\int_{\bbR} \mathcal{D} \xi \exp\left[- i \frac{\hat{w} \xi}{\sqrt{K}} - i \frac{\hat{\Gamma}_0 e^{- \frac{q_d}{2 (1-q_d)} \xi^2}}{K} - i \frac{\hat{S}_1 \hat{H}\left[\sqrt{\frac{q_d}{1-q_d}} \xi\right]}{\sqrt{K}}  - i \frac{\hat{S}_2 \hat{H}\left[\sqrt{\frac{q_d}{1-q_d}} \xi\right]^2}{K}\right]\right]^K
\end{align*}
The large $K$ expansion yields
\begin{align*}
\Lambda &= \exp\Bigg\{- \frac{1}{2} \hat{w}^2 - i \hat{\Gamma} \sqrt{1-q_d} - \hat{S_1}\hat{w}  \EE \left[\xi \hat{H}\left(\sqrt{\frac{q_d}{1-q_d}}\xi\right)\right]\nn
&\qquad\qquad\qquad\qquad\qquad- \left(\frac{1}{2} \hat{S}_1^2 + i \hat{S}_2\right) \EE \left[\hat{H}\left(\sqrt{\frac{q_d}{1-q_d}}\xi\right)^2\right ]\Bigg\} + \mathcal{O}(K^{-1})\,.
\end{align*}
The expectations are taken with respect to a real variable $\xi \sim \mathcal{N}(0,1)$. These expectations are known by properties of the error function:
\begin{align*}
\EE \left[\hat{H}\left(\sqrt{\frac{q_d}{1-q_d}}\xi\right)^2\right ] &= \frac{2}{\pi} \arcsin{q_d} \, , \\
\EE\left[\xi \hat{H}\left(\sqrt{\frac{q_d}{1-q_d}}\xi\right)\right ] &=  \sqrt{\frac{2 q_d}{\pi}}.
\end{align*}
One can now compute the integrals over the ``hat'' variables in \eqref{eq:eq_I_expansion_step}. Denote $\Gamma_0^f \equiv \sqrt{\frac{2(1-q_d)}{\pi}}$, and $S_2^f \equiv \frac{2}{\pi} \arcsin q_d $. This yields:
\begin{align}\label{eq:step_I}
I_C = \int_{\bbR^2} \mathcal{D}w  \mathcal{D}S_1 \, &G\left(y,w,\Gamma_0^f, \sqrt{\frac{2 (\arcsin q_d - q_d)}{\pi}} S_1 + w \sqrt{\frac{2q_d}{\pi}},S_2^f \right) \nonumber \\
&\quad \log G\left(y,w,\Gamma_0^f, \sqrt{\frac{2 (\arcsin q_d - q_d)}{\pi}} S_1 + w \sqrt{\frac{2q_d}{\pi}},S_2^f \right).
\end{align}
Note that 
\begin{align*}
G\left(y,w,\Gamma_0^f, \sqrt{\frac{2 (\arcsin q_d - q_d)}{\pi}} S_1 + w \sqrt{\frac{2q_d}{\pi}},S_2^f \right)  &=  H \left[- y \sqrt{\frac{2}{\pi}} \frac{\sqrt{\arcsin q_d - q_d} S_1+w\sqrt{q_d+q_a}}{\sqrt{1-\frac{2}{\pi}(q_a + \arcsin q_d)}}\right].
\end{align*}
Making the change of variable $S_1^{new} = S_1+w \frac{\sqrt{q_d+q_a}}{\sqrt{\arcsin q_d - q_d}}$ in \eqref{eq:step_I}, and defining $\gamma \equiv \frac{2}{\pi}(q_a + \arcsin q_d)$, one reaches:
\begin{align*}
I_C &= \sum_{y=\pm 1} \int_{\bbR} \mathcal{D}x H \left[y x \sqrt{\frac{\gamma}{1-\gamma}}\right]\log H \left[y x \sqrt{\frac{\gamma}{1-\gamma}}\right]  + \mathcal{O}(K^{-1}).
\end{align*}
The two values of $y$ contribute in the same way, which finally yields:
\begin{align}\label{eq:channel_large_K}
I_C &= 2 \int_{\bbR} \mathcal{D}x  H \left[x \sqrt{\frac{\gamma}{1-\gamma}}\right]\log H \left[x \sqrt{\frac{\gamma}{1-\gamma}}\right] + \mathcal{O}(K^{-1}).
\end{align}
Note that the parameter $\gamma$ is naturally bounded to the interval $[0,1]$ by the conditions $0 \leq q_d \leq 1$ and $0 \leq q_a + q_d \leq 1$.  

\subsection{The Gaussian prior}

The prior part $I_P$ of the free entropy of \eqref{eq:replica_solution} is very easy to evaluate in the Gaussian prior setting. We consider a prior with covariance matrix $Q_0 = I_{K}$ (we can simply rescale $q$ by $q/\sigma^2$ in the final expression for a finite variance $Q_0 = \sigma^2 I_{K}$ as we already described). Performing the Gaussian integration in $I_P$ in \eqref{eq:replica_solution} yields:
\begin{align}\label{eq:prior_large_K}
I_P &= \frac{K}{2} \hat{q}_d + \frac{1}{2} \hat{q}_a - \frac{K-1}{2} \log (1+\hat{q}_d) - \frac{1}{2} \log \left(1+ \hat{q}_d + \hat{q}_a \right).
\end{align}

\subsection{The fixed point equations}
\label{supp:large_K_SP}
From the definition of the free entropy \eqref{eq:replica_solution} and the expansions for $I_P$ and $I_C$ obtained in \eqref{eq:channel_large_K} and \eqref{eq:prior_large_K}, one obtains the fixed point equations after having extremized over $\hat{q}_d$ and $\hat{q}_a$ (recall that $\alpha \equiv \lim \frac{m}{n}$):
\begin{align}
\label{eq:fixed_point_formal_1}
\partial_{q_a} \left[I_G(q_d,q_a) + \alpha I_C(q_d,q_a)\right] &= 0, \\
\label{eq:fixed_point_formal_2}
\partial_{q_d} \left[I_G(q_d,q_a) + \alpha I_C(q_d,q_a)\right] &= 0,
\end{align}
with $I_G(q_d,q_a), I_C(q_d, q_a)$ defined as:
\begin{align}
  \label{eq:def_IG_largeK}
I_G(q_d,q_a) &\equiv \frac{1}{2} \left[q_a + K q_d\right] - \frac{K-1}{2} \log\left[\frac{1}{1-q_d}\right] - \frac{1}{2} \log \left[\frac{1}{1 - q_a - q_d}\right],\\
  \label{eq:def_IC_largeK}
I_C(q_d,q_a) &= \underbrace{2 \int_{\bbR} \mathcal{D}x  H \left[x \sqrt{\frac{\gamma}{1-\gamma}}\right]\log H \left[x \sqrt{\frac{\gamma}{1-\gamma}}\right]}_{\equiv J(\gamma)} + \mathcal{O}(K^{-1}),
\end{align}
and recall that $\gamma = \gamma(q_d,q_a) \equiv \frac{2}{\pi}(q_a + \arcsin q_d)$.
Notice that since $0 \leq q_a + q_d \leq 1$ and $0 \leq q_d \leq 1$, we have $0 \leq \gamma(q_d,q_a)  \leq \frac{2}{\pi}(1 - q_d + \arcsin q_d) \leq 1$.
The fixed point equations \eqref{eq:fixed_point_formal_1}, \eqref{eq:fixed_point_formal_2} have different behaviors depending on the scaling of $\alpha$ with the hidden layer size $K$. We detail these different behaviors in the following paragraphs. 

\subsubsection{Regime $\alpha = o_{K \to \infty}(K)$}

In this regime (which in particular contains the case in which $\alpha$ stays of order $1$ when $K \to \infty$), the fixed point equations \eqref{eq:fixed_point_formal_1}, \eqref{eq:fixed_point_formal_2} can be simplified as 
(recall the definition of $\gamma$ above):
\begin{equation}\label{eq:fixed_point_alpha_1}
	\begin{cases}
		q_d &= 0,  \\
		q_a &= \frac{4}{\pi} \alpha (1-q_a) J'(\gamma).
	\end{cases}
\end{equation}

\subsubsection{Regime $\alpha = \Theta_{K \to \infty}(K)$}

In this regime, we naturally define $\widetilde \alpha \equiv \alpha/K $, such that $\widetilde \alpha$ will remain of order $1$.
  From eq.~\eqref{eq:def_IG_largeK} we have: 
  \begin{align}\label{eq:derivative_IG}
    \begin{dcases}
      \frac{\partial I_G}{\partial q_d} &= - \frac{K q_d}{2 (1-q_d)} + \frac{1}{2 (1-q_d)} - \frac{1}{2(1-q_a-q_d)}, \\
      \frac{\partial I_G}{\partial q_a} &= - \frac{q_a+q_d}{2(1-q_a-q_d)}.
    \end{dcases}
  \end{align}
  Denoting the expansion of $I_C$ in eq.~\eqref{eq:def_IC_largeK} 
  as $I_C(q_d,q_a) = J(\gamma) + K^{-1} \Delta(q_d,q_a)$, with $\Delta(q_d, q_a) = \mcO(1)$ as $K \to \infty$, we have from 
  eqs.~\eqref{eq:fixed_point_formal_1} and \eqref{eq:derivative_IG}:
  \begin{align}\label{eq:fixed_point_chi_largeK}
      \frac{q_a+q_d}{2(1-q_a-q_d)} =\frac{2 \widetilde \alpha}{\pi} K J'(\gamma) + \widetilde \alpha \partial_{q_a} \Delta(q_d,q_a).
  \end{align}
  This implies that $q_a + q_d \sim 1 - \chi / K$, with $\chi$ remaining finite as $K \to \infty$.
  From eqs.~\eqref{eq:fixed_point_formal_2}, \eqref{eq:derivative_IG}, and \eqref{eq:fixed_point_chi_largeK}, we then have (at leading order as $K \to \infty$):
  \begin{align*}
    \begin{dcases}
      \frac{q_d}{2(1-q_d)} &= -\frac{\chi^{-1}}{2} + \frac{2\widetilde{\alpha}}{\pi\sqrt{1-q_d^2}} J'(\gamma), \\
      \chi^{-1} &= \frac{4 \widetilde{\alpha}}{\pi} J'(\gamma).
    \end{dcases}
  \end{align*}
We can finally simplify the equations above as:
\begin{equation}\label{eq:fixed_point_alpha_K}
\begin{cases}
	q_d = \frac{4 \widetilde{\alpha}}{\pi}(1-q_d)\left( \frac{1}{\sqrt{1-q_d^2}}-1\right) J'(\gamma),  \\
			\chi^{-1} = \frac{4 \widetilde{\alpha}}{\pi} J'(\gamma).
			\end{cases}
	\end{equation}
  Notice that here $\gamma = \gamma(q_d, q_a) \simeq \frac{2}{\pi}(1 - q_d + \arcsin(q_d))$ as $K \to \infty$.
The State Evolution (SE) computation of Figure \ref{fig:phaseDiagramKlarge} was performed by solving the fixed point equations \eqref{eq:fixed_point_alpha_1} and \eqref{eq:fixed_point_alpha_K} (depending on the regime of $\alpha$). 

\paragraph{The stability of the $q_d = 0$ solution:}
It is easy to show that \eqref{eq:fixed_point_alpha_K} always admit what we call a \emph{non-specialized solution}, i.e. a solution with $q_d = 0$. This solution stops to be optimal in terms of the free energy at a finite $\widetilde \alpha_{\rm
  spec} \simeq 7.65 $. However, one can show that this solution will remain \emph{linearly} stable for every $\widetilde \alpha$. Actually, it is linearly stable in the much broader regime $\alpha = o(K^2)$. Going back to the initial formulation of the fixed point equations \eqref{eq:fixed_point_formal_1},\eqref{eq:fixed_point_formal_2}, and adding the correct time indices to iterate them, one obtains:  
  \begin{align}
  \label{eq:fixed_point_iterate_qd}
  q_d^{t+1} &= \frac{F(q_d^t,q_a^t)}{1+F(q_d^t,q_a^t)}, \\
  q_a^{t+1} &= \frac{G(q_d^t,q_a^t)}{\left(1+F(q_d^t,q_a^t)\right) \left(1+F(q_d^t,q_a^t)  G(q_d^t,q_a^t)\right)},
  \end{align}
  with $F$ and $G$ defined as:
  \begin{align}
  \label{eq:def_F}
  F(q_d,q_a) &\equiv \frac{2 \alpha}{K-1}\left[\partial_{q_d} I_C-\partial_{q_a} I_C\right], \\
   G(q_d,q_a) &\equiv \frac{2 \alpha K}{K-1}\left[\partial_{q_a} I_C-\frac{1}{K}
  \partial_{q_d} I_C\right].
  \end{align}
  
We focus on the behavior of \eqref{eq:fixed_point_iterate_qd} around $q_d = 0$. Given our previous expansion of $I_C$ in the $K \to \infty$ limit, and \eqref{eq:def_F}, one easily sees that for $\alpha = o_{K \to \infty}(K^2)$,  $\frac{\partial F}{\partial q_d}|_{q_d = 0} \to_{K \to \infty} 0$, which means the $q_d = 0$ solution always remains linearly stable.

However, assume now that $\alpha = \Theta(K^2)$. Performing a similar calculation to the one shown in sec. \ref{sec:large_K_channel}, one can show the following expansion:
\begin{align*}
I_C(q_d,q_a) &= I_C^{(0)}(q_d,q_a) + \frac{1}{K} I_C^{(1)}(q_d,q_a) + \mathcal{O}\left(\frac{1}{K^2}\right).
\end{align*}
The term of $\frac{\partial F}{\partial q_d}|_{q_d = 0}$ arising from $I_C^{(1)}$ will thus have a possibly non-zero contribution in the $K \to \infty$ limit, as seen from \eqref{eq:def_F}. 

To summarize, the non-specialized solution always remains linearly stable in the large $K$ limit at least for $\alpha \ll K^2$. This implies that in this regime, Approximate Message Passing can not escape the non-specialized fixed point to find the specialized solution, as seen in Fig.~\ref{fig:phaseDiagramKlarge}. For $\alpha$ of order larger than $K^2$, one would have to explicitly compute $I_C^{(1)}$ in order to check that $\frac{\partial F}{\partial q_d}|_{q_d = 0} \neq 0$ to show that the non-specialized solution is indeed linearly unstable. This tedious calculation is left for future work.

\subsection{The generalization error at large \texorpdfstring{$K$}{K}}

Recall the definition of the generalization error in \eqref{eq:def_eg_Bayes}. From the remarks of section \ref{sec:generalization}, one can compute it
at large $K$ by applying the same techniques used to compute  the channel integral $I_C$ in sec.~\ref{sec:large_K_channel}. One obtains after a tedious, yet straightforward, calculation:
\begin{align}
\epsilon_g^{\rm Bayes} = \frac{1}{2}\epsilon_g^{\rm Gibbs}  = \frac{1}{\pi} \arccos \left[\frac{2}{\pi} \left(q_a + \arcsin q_d\right)\right] + \mathcal{O}(K^{-1}).
\end{align}
This expression is the one used in the computation of the generalization error in the left panel of Fig.~\ref{fig:phaseDiagramKlarge}.

%% file: sections/supplementary/linear_network.tex
\section{Linear networks show no specialization}\label{sec:linear_net}

An easy yet interesting case is a linear network with identical weights in the second layer and a final output function $\sigma : \bbR \to \bbR$, i.e a network in which $\varphi_{\rm out}(\bh) = \sigma \big(\frac{1}{\sqrt{K}} \sum_{l=1}^K h_l\big)$. For clarity, in this section, we decompose the channel as $P_{\rm out}(y|\varphi_{\rm out}(Z))$ for $Z \in \bbR^K$ instead of $P_{\rm out}(y|Z)$. We will compute the channel integral $I_C$ of the replica solution \eqref{eq:replica_solution}. For simplicity, we assume that $Q^0 = \mathds{1}_K$ the identity matrix (i.e.\ $w$ has identity covariance matrix under $P_0$). Note that \eqref{eq:replica_solution} gives $I_C$ as $I_C = \int_\bbR dy \int_{\bbR^K} \mathcal{D}\xi I_C(y,\xi) \log I_C(y,\xi)$. One can easily derive:
\begin{align*}
I_C(y,\xi) &=  e^{-\frac{1}{2} \xi^\intercal  (\mathds{1}_K - q)^{-1} q \xi} \int_{\bbR^2} \frac{du d\hat{u}}{2 \pi} e^{i u \hat{u}} P_{\rm out}(y|\sigma(u)) \nonumber\\
&\qquad\qquad\qquad\times\int_{\bbR^K} \frac{d Z}{\sqrt{(2 \pi)^K \det (\mathds{1}_K - q)}} e^{-\frac{1}{2} Z^\intercal  (\mathds{1}_K - q)^{-1} Z + Z^\intercal  X(\hat{u},xi)},
\end{align*}
in which we denoted $X(\hat{u},xi) \triangleq (\mathds{1}_K - q)^{-1} q^{1/2} \xi - \frac{i \hat{u}}{\sqrt{K}} \textbf{1}_K$, with the unit vector $\textbf{1}_K = (1)_{l=1}^K$. The inner integration over $Z$ can be done, as well as the integration over $\hat{u}$:
\begin{align*}
I_C(y,\xi) &=\frac{1}{\sqrt{1 - \frac{1}{K} \textbf{1}_K^\intercal  q \textbf{1}_K}}\int_{\bbR} \frac{du }{\sqrt{2 \pi}} P_{\rm out}(y|\sigma(u)) \exp\left[- \frac{\left(u - \frac{1}{\sqrt{K}} \textbf{1}_R^\intercal  q^{1/2} \xi \right)^2}{2  \left(1 - \frac{1}{K} \textbf{1}_K^\intercal  q \textbf{1}_K\right)} \right].
\end{align*}

So we can formally write the total dependency of $I_C(y,\xi)$ on $\xi$ and on $q$ as $$I_C(y,\xi) = I_C\left(y,\frac{1}{\sqrt{K}} \textbf{1}_K^\intercal  q^{1/2} \xi,\frac{1}{K} \textbf{1}_K^\intercal  q \textbf{1}_K \right).$$ Note that we have the following identity, for any fixed vector $x \in \bbR^K$ and smooth real function $F$:
\begin{align}
\int_{\bbR^K} \mathcal{D} \xi F(x^\intercal  \xi) = \frac{1}{\sqrt{2 \pi x^\intercal  x}} \int_\bbR d u F(u) e^{- \frac{u^2}{2 x^\intercal  x}}.
\end{align} 
In the end, if we denote $\Gamma(q) \triangleq \frac{1}{K} \textbf{1}_K^\intercal  q \textbf{1}_K$, we have:
\begin{align}
I_C &= \int_\bbR d y \frac{1}{\sqrt{2 \pi \Gamma(q)}} \int_\bbR d ve^{- \frac{v^2}{2 \Gamma(q)}}  I_C(v,y) \log I_C(v,y) , \\
I_C(v,y) &\equiv \frac{1}{\sqrt{ 2 \pi (1-\Gamma(q))}} \int_\bbR du \, P_{\rm out}(y|\sigma(u)) \exp\left[- \frac{1}{2  \left(1 - \Gamma(q)\right)} \left(u -v\right)^2\right].
\end{align}
Note that by hypothesis, both $q$ and $\mathds{1}_K-q$ are positive matrices, so $0 \leq \Gamma(q) \leq 1$. As these equations show, $I_C$ only depends on $\Gamma(q) = K^{-1} \sum_{l,l'} q_{ll'}$. From this one easily sees that extremizing over $q$ implies that the optimal $\hat{q}$ satisfies $\hat{q}_{l l'} = \hat{q}/K$ for some real $\hat{q}$. Subsequently, all $q_{ll'}$ are also equal to a single value, that we can denote $\frac{q}{K}$. This shows that this network never exhibits a specialized solution.

%% file: sections/supplementary/AMP_derivation.tex
\section{Update functions and AMP derivation}
\label{sec:AMP}

AMP can be seen as Taylor expansion of the loopy belief-propagation (BP)
approach \cite{mezard1987spin,mezard2009information,wainwright2008graphical},
similar to the so-called Thouless-Anderson-Palmer equation in spin
glass theory \cite{thouless1977solution}. While the behavior of AMP
can be rigorously studied
\cite{bayati2011dynamics,javanmard2013state,bayati2015universality},
it is useful and instructive to see how the derivation can be
performed in the framework of belief-propagation and the cavity
method, as was pioneered in \cite{mezard1989space,Kaba} for the single
layer problem. The derivation uses the Generalized AMP notations of
\cite{rangan2011generalized} and follows closely the one of \cite{REVIEWFLOANDLENKA}.

\subsection{Definition of the update functions}
\label{supp:update_functions}
Let's consider the distributions probabilities $Q_{\rm out}$ and $Q_0$, closely related to the inference problems of eq.~\eqref{aux-model-1} and eq.~\eqref{aux-model-2}:
\begin{align*}
	Q_{\rm out}(z;\omega,y,V) &\equiv \frac{1}{\mathcal{Z}_{P_{\rm out}}}e^{-\frac{1}{2}(z-\omega)^\intercal V^{-1}(z-\omega)} P_{\rm out}(y|z); \hspace{0.5cm} Q_0 (W; \Sigma, T) \equiv \frac{1}{\mathcal{Z}_{P_0}} P_0(W) e^{ -\frac{1}{2} W^\intercal \Sigma^{-1}W + T^\intercal \Sigma^{-1}W }\,.
\end{align*}
 We define the update functions $g_{\rm out}$, $\partial_{\omega} g_{\rm out}$, $f_w$ and $f_c$, which will be useful later in the algorithm:
\begin{align*}	
		g_{\rm out}(\omega,y,V) &\equiv \partial_\omega \log( \mathcal{Z}_{P_{\rm out}} ) =	V^{-1} \mathbb{E}_{Q_{\rm out}} \left[ z-\omega \right] \,, \vspace{0.5cm} \\
		\partial_{\omega} g_{\rm out}(\omega,y,V) &= 	V^{-1} \mathbb{E}_{Q_{\rm out}} \left[ (z-\omega)(z-\omega)^\intercal \right] - V^{-1} -  g_{\rm out}g_{\rm out}^\intercal \,, \vspace{0.5cm} \\
		f_w( \Sigma, T) &\equiv  \partial_{\Sigma^{-1}T} \log \mathcal{Z}_{P_{\rm 0}} = \mathbb{E}_{Q_0}[W]\,, \vspace{0.5cm} \\
		f_c ( \Sigma, T) & \equiv \partial_{\Sigma^{-1}T} f_w =  \mathbb{E}_{Q_0} [W W^\intercal ]  - f_w  f_w^\intercal \,.
\end{align*}
Note that $g_{\rm out}$ is the mean of $V^{-1}(z-\omega)$ with respect tor $Q_{\rm out}$ and $f_w$ the mean of $Q_0$.

\subsection{Derivation of the Approximate Message Passing algorithm}
\label{supp:AMP_derivation}

\subsubsection{Relaxed BP equations}

Let us consider a set of messages $\{m_{i\to \mu},\tilde{m}_{\mu \to i}\}_{i=1..n,\mu=1..m}$ on the bipartite factor graph corresponding to our problem  Fig.~\ref{fig_factG}. These messages correspond to the marginal probabilities of $W_i$ if we remove the edges $i \to \mu$ or $\mu \to i$. The belief propagation (BP) equations (or sum-product equations) can be formulated as the following \cite{mezard2009information,wainwright2008graphical}, where $W_i=(w_{il})_{l=1..K} \in \bbR^K$:
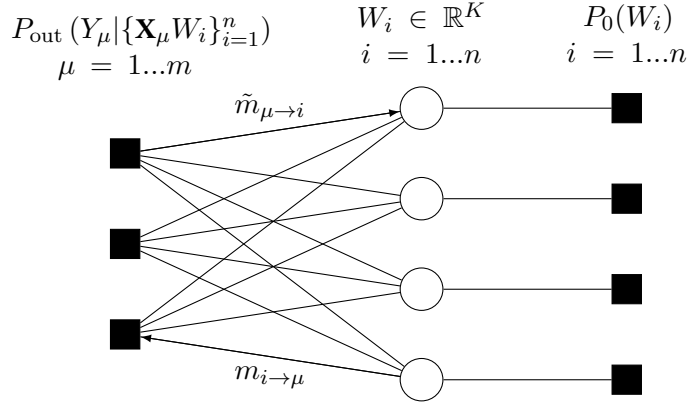
\begin{figure}[htb!]
\centering
	\begin{tikzpicture}[scale=1.2]
    \tikzstyle{annot} = [text width=3cm, text centered]
       \tikzstyle{factor}=[rectangle,minimum size=11pt,draw=black, fill opacity=1.,fill=black]
    \tikzstyle{latent}=[circle,minimum size=16pt,draw=black, fill opacity=1.,fill=white]     	
	\def\Nv{4}
	\def\Nu{3}
	\def\Nf{12}        
         \foreach \name / \y in {1,...,\Nu}
        \node[factor] (F2-\name) at (0,-0.5-\y) {};
        
        \foreach \name / \y in {1,...,\Nv}
        \node[latent] (W-\name) at (3.25,-\y) {};
         \foreach \name / \y in {1,...,\Nv}
        \node[factor] (P-\name) at (5.5,-\y) {};

\path[-latex] (F2-1) edge node[above]{$\td{m}_{\mu \to i}$} (W-1);
\path[-latex] (W-4) edge node[below]{$m_{i \to \mu}$} (F2-3);        	
    \foreach \i in {1,...,\Nu}
    	 \foreach \j in {1,...,\Nv}
        	\edge[-]{F2-\i}{W-\j};  
    	\foreach \j in {1,...,\Nv}
        	\edge[-]{P-\j}{W-\j};     	

\node[annot,above of=F2-1, node distance=1.4cm] {$P_{\rm out}\(Y_{\mu}| \{\tbf{X}_{\mu} W_i\}_{i=1}^n  \)$ \\ $\mu =1...m$};
\node[annot,above of=W-1, node distance=1cm] (hl) {$W_i \in \bbR^K$ \\ $i=1...n$};
\node[annot,above of=P-1, node distance=1cm] (hl) {$P_0(W_i)$ \\ $i=1...n$};
\end{tikzpicture}
	 \caption{Factor graph representation of the committee machine (for $n=4$ and $m=3$). The variable (circle) $W_i \in \bbR^{K}$ needs to satisfy a prior constraint (square) $P_0$ and a constraint accounting for the fully connected layer, that correlates all the variables together.}
\label{fig_factG}
\end{figure}

\begin{equation}
	\begin{cases}
		m_{i\to \mu}^{t+1} (W_i) = \displaystyle \frac{1}{\mathcal{Z}_{i\to \mu}} P_0 (W_i) \prod\limits_{k \neq \mu}^m \tilde{m}_{\nu \to i}^t (W_i)\,, \vspace{0.1cm}\\
		\tilde{m}_{\mu \to i}^t (W_i) =  \displaystyle \frac{1}{\mathcal{Z}_{\mu \to i}} \int \prod\limits_{j\neq i}^n dW_j P_{\rm out}\left (Y_{\mu} | \displaystyle  \frac{1}{\sqrt{n}} \sum_{j=1}^n  X_{\mu j}W_{j} \right)  m_{j \to \mu}^t (W_j ) \,.
	\end{cases}
	\label{supp:BPEquations}
\end{equation}
The term inside $P_{\rm out}$ can be decouple using its $K$-dimensional Fourier transform
\begin{equation*}
P_{\rm out}\left (Y_{\mu} | \displaystyle \frac{1}{\sqrt{n}} \sum_{j=1}^n  X_{\mu j}W_{j} \right) = \frac{1}{(2\pi)^{K/2}}
\int_{\bbR^K} d \xi \exp\left( i \xi^\intercal \left( \displaystyle \frac{1}{\sqrt{n}} \sum_{j=1}^n  X_{\mu j}W_{j}\right) \hat{P}_{\rm out}(Y_{\mu} , \xi )    \right) \,.	
\end{equation*}
Injecting this representation in the BP equations, \eqref{supp:BPEquations} becomes
\begin{align*}
&\tilde{m}_{\mu \to i}^t (W_i ) = 
\frac{1}{(2\pi)^{K/2}\mathcal{Z}_{\mu\to i} }
\int_{\bbR^K} d\xi \hat{P}_{\rm out}(Y_{\mu} , \xi)  
\exp\left(i  \xi^\intercal \frac{1}{\sqrt{n}} X_{\mu i} W_i \right)\nonumber\\
 &\qquad\qquad\qquad\qquad\times\prod\limits_{j\neq i}^n  \underbrace{\int_{\bbR^K} dW_j 
		 m_{j \to \mu}^t (W_j ) \exp\left( i  \xi^\intercal \frac{1}{\sqrt{n}} X_{\mu j} W_j ) \right)}_{\equiv I_j}
		 \label{mtilde} \,,
\end{align*}
and we define the mean and variance of the messages
\begin{equation}
	\begin{cases}
		\hat{W}_{j\to \mu}^t \equiv  \displaystyle \int_{\bbR^K} dW_j
		 m_{j\to \mu}^t (W_j ) W_j  \,, \vspace{0.2cm} \\
		 \hat{C}_{j\to \mu}^t \equiv \displaystyle  \int_{\bbR^K} dW_j
		 m_{j\to \mu}^t (W_j ) W_j W_j^\intercal - \hat{W}_{j\to \mu}^t(\hat{W}_{j\to \mu}^t)^\intercal \,.
	\end{cases}
\label{What_Chat}
\end{equation}

In the limit $n\to \infty$ the term $I_j$ can be easily expanded and expressed using $\hat{W}$ and $\hat{C}$
\begin{align*}
 I_j &= \int_{{\bbR}^K} dW_j
		 m_{j \to \mu}^t (W_j ) \exp\left( i  \xi^\intercal \frac{X_{\mu j}}{\sqrt{n}} W_j ) \right) \simeq  \exp\left( i \frac{X_{\mu j}}{\sqrt{n}}\xi^\intercal  \hat{W}_{j\to \mu}^t -  \frac{1}{2}\frac{X_{\mu j}^2}{n} \xi^\intercal  \hat{C}_{j\to \mu}^t  \,, \xi\right) \,,
\end{align*} 
and finally using the inverse Fourier transform, we obtain
\begin{align*}
&\tilde{m}_{\mu \to i}^t (W_i ) \simeq 
\frac{1}{(2\pi)^K\mathcal{Z}_{\mu \to i}}
\int_{\bbR^K} dz P_{\rm out}(Y_\mu , z ) 
\int_{\bbR^K} d \xi  
e^{-i \xi^\intercal z}
e^{ i X_{\mu i} \xi^\intercal W_i} \\
&\qquad\qquad\qquad\qquad\qquad\times\prod\limits_{j\neq i}^n  \exp\left( i \frac{X_{\mu j}}{\sqrt{n}} \xi^\intercal \hat{W}_{j\to \mu}^t -  \frac{1}{2}\frac{X_{\mu j}^2}{n} \xi^\intercal \hat{C}_{j\to \mu}^t \xi\right) \\
&= \frac{1}{(2\pi)^K\mathcal{Z}_{\mu\to i}}
\int_{\bbR^K} dz
P_{\rm out}(Y_\mu , z )
\int_{\bbR^K} d \xi  
e^{-i \xi^\intercal z}
e^{ i X_{\mu i} \xi^\intercal W_i} e^{i \xi^\intercal \sum\limits_{j\neq i}^n \frac{X_{\mu j}}{\sqrt{n}} \hat{W}_{j\to \mu}^t } e^{-  \frac{1}{2} \xi^\intercal \sum\limits_{j\neq i}^n \frac{X_{\mu j}^2}{n}   \hat{C}_{j \to \mu }^t\xi} \\
&= \frac{1}{(2\pi)^K\mathcal{Z}_{\mu\to i}} \int_{\bbR^K} dzP_{\rm out}(Y_\mu , z)\sqrt{\frac{(2\pi)^K}{\det(V_{i \mu }^t)}} \underbrace{e^{-\frac{1}{2} \left( z -\frac{X_{\mu i}}{\sqrt{n}} W_i -\omega_{i \mu}^t \right)^\intercal (V_{i \mu }^t)^{-1} \left( z -\frac{X_{\mu i}}{\sqrt{n}} W_i -\omega_{i \mu}^t \right)}}_{\equiv H_{i\mu}} \,,
\end{align*}
where we defined the mean and variance, depending on the node $i$
\begin{equation}
	\omega_{i \mu}^t \equiv \displaystyle  \frac{1}{\sqrt{n}} \sum\limits_{j\neq i}^n X_{\mu j}  \hat{W}_{j\to \mu}^t \,, \hspace{0.5cm} 
		V_{i\mu}^t \equiv \displaystyle  \frac{1}{n} \sum\limits_{j\neq i}^n X_{\mu j}^2  \hat{C}_{j \to \mu}^t \,.
		\label{appendix:amp:omega_V}
\end{equation}

Again, in the limit $n\to \infty$, the term $H_{i\mu}$ can be expanded:
\begin{align*}
	H_{i\mu} &\simeq   e^{-\frac{1}{2} \left( z -\omega_{i \mu}^t \right)^\intercal (V_{i\mu}^t)^{-1} \left( z -\omega_{i \mu}^t \right) } 
	\left( 1 + \frac{X_{\mu i}}{\sqrt{n}} W_i^\intercal (V_{i\mu}^t)^{-1} (z -\omega_{i \mu}^t) -\frac{1}{2}\frac{X_{\mu i}^2}{n} W_i^\intercal (V_{i\mu}^t)^{-1} W_i \right.\\
& \left. + \frac{1}{2} \frac{X_{\mu i}^2}{n} W_i^\intercal (V_{i\mu}^t)^{-1} (z -\omega_{i \mu}^t)(z -\omega_{i \mu}^t)^\intercal  (V_{i\mu}^t)^{-1} W_i \right).
\end{align*}
Gathering all pieces, the message $\tilde{m}_{\mu \to i}$ can be expressed using definitions of $g_{\rm out}$ and $\partial_{\omega} g_{\rm out}$
\begin{align*}
\tilde{m}_{\mu  \to i}^t (W_i ) &\sim \frac{1}{\mathcal{Z}_{\mu \to i}} \left \{1 +  \frac{X_{\mu i}}{\sqrt{n}} W_{i}^\intercal  g_{\rm out} (\omega_{i\mu}^t, Y_{\mu}, V_{i\mu}^t) +
\frac{1}{2} \frac{X_{\mu i}^2}{n} W_{i}^\intercal g_{\rm out}  g_{\rm out}^\intercal (\omega_{i\mu}^t, Y_{\mu}, V_{i\mu}^t) W_{i} +\right. \\
& \left. \frac{1}{2} \frac{X_{\mu i}^2}{n} W_{i}^\intercal  \partial_\omega g_{\rm out}(\omega_{i\mu}^t, Y_{\mu}, V_{i\mu}^t)  W_{i}
\right\}\\
&= \frac{1}{\mathcal{Z}_{\mu \to i}} \left\{ 1 + W_{i}^\intercal  B_{\mu \to i}^t +\frac{1}{2}  W_{i}^\intercal  B_{\mu \to i}^t (B_{\mu \to i}^t)^\intercal  (W_{i}) -\frac{1}{2} W_{i}^\intercal  A_{\mu \to i}^t W_{i} \right\} \\
&=\sqrt{\frac{\det(A_{\mu \to i}^t)}{(2\pi)^K}} \exp\left(-\frac{1}{2}\left(W_{i}^\intercal  - (A_{\mu \to i}^t)^{-1}B_{\mu \to i}^t \right)^\intercal  A_{\mu \to i}^t\left(W_{i}^\intercal  - (A_{\mu \to i}^t)^{-1}B_{\mu \to i}^t \right) \right) \,,
\label{supp:mtilde}
\end{align*}
with the following definitions of $A_{\mu \to i}$ and $B_{\mu \to i}$:
\begin{equation}
\label{appendix:amp:A_B}
		B_{\mu \to i}^t \equiv  \frac{X_{\mu i}}{\sqrt{n}} g_{\rm out} (\omega_{i\mu}^t, Y_{\mu}, V_{i\mu}^t), \hspace{0.5cm}
		A_{\mu \to i}^t \equiv - \frac{X_{\mu i}^2}{n}  \partial_\omega g_{\rm out}(\omega_{i\mu}^t, Y_{\mu}, V_{i\mu}^t)
\end{equation}
Using the set of BP equations \eqref{supp:BPEquations}, we can finally close the set of equations only over $\{ m_{i\to \mu}\}_{i\mu}$:
\begin{equation*}
	 m_{i\to \mu}^{t+1} (W_i) = \frac{1}{\mathcal{Z}_{i\to \mu}} P_0 (W_i) \prod\limits_{\nu \neq \mu}^m \sqrt{\frac{\det(A_{\nu \to i}^t)}{(2\pi)^K}} e^{-\frac{1}{2}\left(W_{i} - (A_{\nu \to i}^t)^{-1}B_{\nu \to i}^t \right)^\intercal  A_{\nu \to i}^t\left(W_{i} - (A_{\nu \to i}^t)^{-1}B_{\nu \to i}^t \right) }.
	 \label{mil}
\end{equation*}

In the end, computing the mean and variance of the product of Gaussians, the messages are updated using $f_w$ and $f_c$:
\begin{equation}
	\begin{cases}
		\hat{W}_{i\to \mu}^{t+1} = f_w( \Sigma_{\mu \to i}^t , T_{\mu \to i}^t )\,,  \vspace{0.3cm} \\
		\hat{C}_{i \to \mu}^{t+1} = f_c( \Sigma_{\mu \to i}^t , T_{\mu \to i}^t )\,, \vspace{0.3cm} \\
	\end{cases}
	\hspace{1cm}
	\begin{cases}
	\Sigma_{\mu \to i}^t \equiv \left( \sum\limits_{\nu \ne \mu}^m  A_{\nu \to i}^t \right )^{-1}\,,  \vspace{0.1cm} \\
	T_{\mu \to i}^t \equiv \Sigma_{\mu \to i}^t  \left( \sum\limits_{\nu \ne \mu}^m  B_{\nu \to i}^t \right) \,.
	\end{cases} 
	\label{appendix:amp:Sigma_T}
\end{equation}

\paragraph{Summary of the Relaxed BP set of equations:\\}
In the end, using eq~.(\ref{What_Chat},\ref{appendix:amp:omega_V},\ref{appendix:amp:A_B}, \ref{appendix:amp:Sigma_T}), relaxed BP equations can be written as the following set of equations:\\
\begin{minipage}[c]{.46\linewidth}
\begin{align*}
	\begin{cases}
		\omega_{i\mu}^t &= \sum\limits_{j\neq i}^n \frac{X_{\mu j}}{\sqrt{n}}   \hat{W}_{j\to \mu}^t\vspace{0.3cm} \\
		V_{i\mu}^t &= \sum\limits_{j\neq i}^n \frac{X_{\mu j}^2}{n} \hat{C}_{j\to \mu}^t \vspace{0.3cm} \\
		B_{\mu \to i}^t &=  \frac{X_{\mu i}}{\sqrt{n}} g_{\rm out} (\omega_{i\mu}^t, Y_{\mu}, V_{i\mu}^t)\vspace{0.3cm}  \\
		A_{\mu \to i}^t &= - \frac{X_{\mu i}^2}{n}  \partial_\omega g_{\rm out}(\omega_{i\mu}^t, Y_{\mu}, V_{i\mu}^t)\vspace{0.3cm} \\
	\end{cases}
\end{align*}
      
\end{minipage} \hfill
\begin{minipage}[c]{.46\linewidth}
\begin{align}
	\begin{cases}
		\Sigma_{\mu \to i}^t &= \left( \sum\limits_{\nu \ne \mu}^m  A_{\nu \to i}^t \right )^{-1} \vspace{0.3cm} \\
	T_{\mu \to i}^t &= \Sigma_{\mu \to i}^t  \left( \sum\limits_{\nu \ne \mu}^m  B_{\nu \to i}^t \right)\vspace{0.3cm} \\
	\hat{W}_{i\to \mu}^{t+1} &= f_w( \Sigma_{\mu \to i}^t , T_{\mu \to i}^t ) \vspace{0.3cm} \\
		\hat{C}_{i \to \mu}^{t+1} &= f_c( \Sigma_{\mu \to i}^t , T_{\mu \to i}^t ) \vspace{0.3cm} \\
	\end{cases}
	\label{supp:relaxed_BP}
\end{align}
\end{minipage}

\subsubsection{Approximate Message Passing algorithm}
The relaxed BP algorithm uses $\mathcal{O}(n^2)$ messages. However, all the messages depend weakly on the target node. On a tree, the missing message is negligible, that allows us to expand the previous relaxed BP equations \eqref{supp:relaxed_BP} to make appear the Onsager term at a previous time step, and reduce the number of messages to $\mathcal{O}(n)$. We define the following estimates and parameters based on the complete set of messages:

\begin{minipage}[c]{.46\linewidth}
\begin{equation*}
	\begin{cases}
		\omega_{\mu}^t \equiv \sum\limits_{j = 1}^n \frac{X_{\mu j}}{\sqrt{n}}   \hat{W}_{j\to \mu}^t\vspace{0.1cm} \\
		V_{\mu}^t \equiv  \sum\limits_{j=1}^n \frac{X_{\mu j}^2}{n}   \hat{C}_{j\to \mu}^t\vspace{0.1cm} \\
	\end{cases}
\end{equation*}
\end{minipage} \hfill
\begin{minipage}[c]{.46\linewidth}
\begin{equation}
	\begin{cases}
		\Sigma_{i  }^t \equiv  \left( \sum\limits_{\nu =1}^m  A_{\nu \to i}^t \right )^{-1} \vspace{0.1cm} \\
	T_{i }^t \equiv  \Sigma_{i}^t  \left( \sum\limits_{\nu =1}^m  B_{\nu \to i}^t \right)\vspace{0.1cm} \\
	\end{cases}
\end{equation}
\end{minipage}
Let's now expand the previous messages of eq.~\eqref{supp:relaxed_BP}, making appear these new target-independent messages:
\paragraph{$\bullet$ $\Sigma_{\mu \to i}^t$}
\begin{align*}
	\Sigma_{\mu \to i}^t &= \left( \sum\limits_{\nu \ne \mu}^m  A_{\nu \to i}^t \right )^{-1} = \left( \sum\limits_{\nu =1 }^m  A_{\nu \to i}^t - A_{\mu \to i}^t \right )^{-1} 
	=  \left( \sum\limits_{\nu =1 }^m  A_{\nu \to i}^t \left( I_{K\times K} -  \left( \sum\limits_{\nu =1 }^m  A_{\nu \to i}^t \right)^{-1} A_{\mu \to i}^t \right) \right )^{-1} \\
	&=  \left( I_{K\times K} -  \left( \sum\limits_{\nu =1 }^m  A_{\nu \to i}^t \right)^{-1} A_{\mu \to i}^t \right)^{-1}  \left( \sum\limits_{\nu =1 }^m  A_{\nu \to i}^t  \right )^{-1} = \underbrace{\left( I_{K\times K} -  \Sigma_{i  }^t A_{\mu \to i}^t \right)^{-1}}_{\simeq I_{K\times K} + \Sigma_{i  }^t A_{\mu \to i}^t + {\cal O}(n^{-1})}  \Sigma_{i  }^t \simeq  \Sigma_{i  }^t + {\cal O}\(\frac{1}{n} \)
\end{align*}

\paragraph{$\bullet$ $T_{\mu \to i}^t$}
\begin{align*}
	T_{\mu \to i}^t &= \Sigma_{\mu \to i}^t  \left( \sum\limits_{\nu \ne \mu}^m  B_{\nu \to i}^t \right) 
	= \left( \Sigma_{i  }^t + {\cal O}\left(\frac{1}{n}\right) \right) \left( \sum\limits_{\nu =1}^m  B_{\nu \to i}^t -  B_{\mu \to i}^t \right) \\
	&= T_{i}^t - \Sigma_{i  }^t B_{\mu \to i}^t + {\cal O}\left(\frac{1}{n}\right)
\end{align*}

\paragraph{$\bullet$ $\hat{W}_{i\to \mu}^{t+1}$}
\begin{align*}
	\hat{W}_{i\to \mu}^{t+1} &= f_w( \Sigma_{\mu \to i}^t , T_{\mu \to i}^t ) = f_w\left( \Sigma_{i  }^t  , T_{i}^t - \Sigma_{i  }^t B_{\mu \to i}^t  \right) + {\cal O}\left(\frac{1}{n}\right)\\
	&\simeq f_w\left( \Sigma_{i  }^t  , T_{i}^t  \right) -      \left.\frac{d f_w}{d T}  \right|_{\left(\Sigma_{i  }^t ,T_{i}^t\right)}\Sigma_{i  }^t B_{\mu \to i}^t \\
	&= \underbrace{f_w\left( \Sigma_{i  }^t  , T_{i}^t  \right)}_{=\hat{W}_{i}^{t+1}} - \left(\Sigma_{i  }^t\right)^{-1} \underbrace{f_c \left(\Sigma_{i  }^t, T_{i}^t \right) \Sigma_{i  }^t}_{=\hat{C}_{i}^{t+1}} \underbrace{B_{\mu \to i}^t}_{\simeq \frac{X_{\mu i}}{\sqrt{n}} g_{\rm out} (\omega_{\mu}^t, Y_{\mu}, V_{\mu}^t)}\\
	&= \hat{W}_{i}^{t+1} - \frac{X_{\mu i}}{\sqrt{n}} \left(\Sigma_{i  }^t\right)^{-1}\hat{C}_{i}^{t+1}\Sigma_{i  }^tg_{\rm out} (\omega_{\mu}^t, Y_{\mu}, V_{\mu}^t) + {\cal O}\left( \frac{1}{n} \right)
\end{align*}

\paragraph{$\bullet$ $\hat{C}_{i\to \mu}^{t+1}$\\}
Let's denote for convenience, $\mathcal{E} = \left(\Sigma_{i  }^t\right)^{-1}\hat{C}_{i}^{t+1}\Sigma_{i  }^tg_{\rm out} (\omega_{\mu}^t, Y_{\mu}, V_{\mu}^t)$. Then

\begin{align*}
\hat{C}_{i\to \mu}^{t+1} &=\mathbb{E}_{Q_0} \[ \hat{W}_{i\to \mu}^t (\hat{W}_{i\to \mu}^t)^\intercal  \] -\mathbb{E}_{Q_0} \[\hat{W}_{i\to \mu}^t  \] \mathbb{E}_{Q_0} \[ \hat{W}_{i\to \mu}^t \]^\intercal \\
&= \mathbb{E}_{Q_0} \[ \left(\hat{W}_{i}^t - \frac{X_{\mu i}}{\sqrt{n}} \mathcal{E}\right)\left(\hat{W}_{i}^t - \frac{X_{\mu i}}{\sqrt{n}} \mathcal{E}\right)^\intercal  \] - \mathbb{E}_{Q_0} \[ \hat{W}_{i}^t - \frac{X_{\mu i}}{\sqrt{n}} \mathcal{E}\] \mathbb{E}_{Q_0} \[ \hat{W}_{i}^t - \frac{X_{\mu i}}{\sqrt{n}} \mathcal{E}\]^\intercal \\
&=  \mathbb{E}_{Q_0} \[ \hat{W}_{i}^t (\hat{W}_{i}^t)^\intercal  \] - \mathbb{E}_{Q_0} \[ \hat{W}_{i}^t \]\mathbb{E}_{Q_0} \[ \hat{W}_{i}^t \]^\intercal  + {\cal O}\left( \frac{1}{\sqrt{n}} \right)  = \hat{C}_{i}^{t+1} + {\cal O}\left( \frac{1}{\sqrt{n}} \right) 
 \end{align*} 
 
\paragraph{$\bullet$ $g_{\rm out} (\omega_{i\mu}^t, Y_{\mu}, V_{i\mu}^t)$}
\begin{align*}
	g_{\rm out} (\omega_{i\mu}^t, Y_{\mu}, V_{i\mu}^t) &= g_{\rm out} \left(\omega_{\mu}^t - \frac{X_{\mu i}}{\sqrt{n}}   \hat{W}_{i\to \mu}^t, Y_{\mu}, V_{\mu}^t - \frac{X_{\mu i}^2}{n}   \hat{C}_{i \to l}^t \right)\\
	&= g_{\rm out} \left(\omega_{\mu}^t , Y_{\mu}, V_{\mu}^t \right) - \frac{X_{\mu i}}{\sqrt{n}} \frac{\partial g_{\rm out} }{ \partial \omega}\left(\omega_{\mu}^t , Y_{\mu}, V_{\mu}^t \right)   \underbrace{\hat{W}_{i\to \mu}^t}_{=\hat{W}_{i}^t + {\cal O}\left( \frac{1}{\sqrt{n}}\right)} + {\cal O}\left( \frac{1}{n}\right)\\
	&= g_{\rm out} \left(\omega_{\mu}^t , Y_{\mu}, V_{\mu}^t \right)-\frac{X_{\mu i}}{\sqrt{n}}\frac{\partial g_{\rm out} }{ \partial \omega}\left(\omega_{\mu}^t , Y_{\mu}, V_{\mu}^t \right)   \hat{W}_{i}^t +{\cal O}\left( \frac{1}{n}\right)
\end{align*}

\paragraph{$\bullet$ $V_{\mu}^t$}
\begin{align*}
	V_{\mu}^t &= \sum\limits_{i=1}^n\frac{X_{\mu i}^2}{n}   \hat{C}_{i \to l}^t = \sum\limits_{i=1}^n \frac{X_{\mu i}^2}{n}   \hat{C}_{i }^t + {\cal O}\left( \frac{1}{n^{3/2}}\right)
\end{align*}

\paragraph{$\bullet$ $\omega_{\mu}^t$}
\begin{align*}
	\omega_{\mu}^t &= \sum\limits_{i = 1}^n \frac{X_{\mu i}}{\sqrt{n}}   \hat{W}_{i\to \mu}^t = \sum\limits_{i = 1}^n \frac{X_{\mu i}}{\sqrt{n}} \left(\hat{W}_{i}^t - X_{\mu i}\left(\Sigma_{i  }^{t-1}\right)^{-1}\hat{C}_{i}^t\Sigma_{i  }^{t-1}g_{\rm out} (\omega_{\mu}^{t-1}, Y_{\mu}, V_{\mu}^{t-1}) + {\cal O}\left( \frac{1}{n} \right)  \right) \\
	&= \sum\limits_{i = 1}^n \frac{X_{\mu i}}{\sqrt{n}} \hat{W}_{i}^t -   \sum\limits_{i = 1}^n \frac{X_{\mu i}^2}{n} \left(\Sigma_{i  }^{t-1}\right)^{-1}\hat{C}_{i}^t\Sigma_{i  }^{t-1}g_{\rm out} (\omega_{\mu}^{t-1}, Y_{\mu}, V_{\mu}^{t-1}) + {\cal O}\left( \frac{1}{n^{3/2}}\right)
\end{align*}

\paragraph{$\bullet$ $\left(\Sigma_{i}^t\right)^{-1}$}
\begin{align*}
\left(\Sigma_{i}^t\right)^{-1} &= \sum\limits_{\mu =1}^m  A_{\mu \to i}^t
	= - \sum\limits_{\mu =1}^m X_{\mu i}^2  \partial_\omega g_{\rm out}(\omega_{i\mu}^t, Y_{\mu}, V_{i\mu}^t) = - \sum\limits_{\mu =1}^m X_{\mu i}^2  \partial_\omega g_{\rm out}(\omega_{\mu}^t, Y_{\mu}, V_{\mu}^t) + {\cal O}\left( \frac{1}{n^{3/2}}\right)
\end{align*}

\paragraph{$\bullet$ $T_{i}^t$}
\begin{align*}
T_{i }^t &= \Sigma_{i}^t  \left( \sum\limits_{\mu =1}^m  B_{\mu \to i}^t \right)  =  \Sigma_{i}^t   \sum\limits_{\mu =1}^m \frac{X_{\mu i}}{\sqrt{n}} g_{\rm out} (\omega_{i\mu}^t, Y_{\mu}, V_{i\mu}^t)\\
&= \Sigma_{i}^t   \sum\limits_{\mu =1}^m \frac{X_{\mu i}}{\sqrt{n}} \left(  g_{\rm out} \left(\omega_{\mu}^t , Y_{\mu}, V_{\mu}^t \right)-\frac{X_{\mu i}}{\sqrt{n}}\frac{\partial g_{\rm out} }{ \partial \omega}\left(\omega_{\mu}^t , Y_{\mu}, V_{\mu}^t \right)   \hat{W}_{i}^t +{\cal O}\left( \frac{1}{n}\right) \right)\\
&= \Sigma_{i}^t \left(  \sum\limits_{\mu =1}^m \frac{X_{\mu i}}{\sqrt{n}} g_{\rm out} \left(\omega_{\mu}^t , Y_{\mu}, V_{\mu}^t \right)  -\frac{X_{\mu i}^2}{n} \frac{\partial g_{\rm out} }{ \partial \omega}\left(\omega_{\mu}^t , Y_{\mu}, V_{\mu}^t \right)   \hat{W}_{i}^t \right) + {\cal O}\left( \frac{1}{n^{3/2}}\right)
\end{align*}

The AMP algorithm follows naturally the rBP updates \eqref{supp:relaxed_BP} using the expanded estimates of the mean and variance $\omega_{\mu}$, $V_{\mu}$, $T_{i }$ and $\Sigma_{i}$, and finally reads in pseudo language:
\begin{algorithm}
\caption{Approximate Message Passing for the committee machine}  
\begin{algorithmic}
    \STATE {\bfseries Input:} vector $Y \in \bbR^m$ and matrix $X\in \bbR^{m \times n}$:
    \STATE \emph{Initialize}: $g_{\rm out,\mu} = 0, \Sigma_i = I_{K\times K} $ for $ 1 \leq i \leq n $ and $ 1 \leq \mu \leq m $ at $t=0$.
    \STATE \emph{Initialize}: $\hat{W}_i \in \bbR^K$ and $\hat{C}_i$, $\partial_{\omega} g_{\rm out , \mu}$ $\in \mathcal{S}_K^+$ for $ 1 \leq i \leq n $ and $ 1 \leq \mu \leq m $ at $t=1$.
    \REPEAT   
    \STATE Update of the mean $\omega_{\mu} \in \bbR^K$ and covariance $V_{\mu}\in \mathcal{S}_K^+$: \\
    \hspace{0.5cm} $\omega_{\mu}^t = \sum\limits_{i = 1}^n \big(\frac{X_{\mu
      i}}{\sqrt{n}}\hat{W}_{i}^t -    \frac{X_{\mu
      i}^2}{n}
    \left(\Sigma_{i}^{t-1}\right)^{-1}\hat{C}_{i}^t \Sigma_{i
    }^{t-1}g_{\rm out,\mu}^{t-1} \big)    \hspace{0.5cm}|\hspace{0.5cm}  V_{\mu}^t = \sum\limits_{i=1}^n\frac{X_{\mu
      i}^2}{n} \hat{C}_{i}^t $\vspace{0.1cm}
    \STATE Update of $g_{\rm out, \mu} \in \bbR^K$ and $\partial_{\omega} g_{\rm out , \mu} \in \mathcal{S}_K^+$: \\
    \hspace{0.5cm}$g_{\rm out, \mu}^t = g_{\rm out} (\omega_{\mu}^t , Y_{\mu}, V_{\mu}^t) \hspace{0.5cm}|\hspace{0.5cm}  \partial_{\omega} g_{\rm out, \mu}^t = \partial_{\omega}  g_{\rm out} (\omega_{\mu}^t , Y_{\mu}, V_{\mu}^t)  $ \vspace{0.1cm}
    \STATE Update of the mean $T_i \in \bbR^K$ and covariance $\Sigma_i \in \mathcal{S}_K^+$:\\
    \hspace{0.5cm}$T_i^t = \Sigma_{i}^t \Big(  \sum\limits_{\mu =1}^m
      \frac{X_{\mu
      i}}{\sqrt{n}}g_{\rm out,\mu}^t  -\frac{X_{\mu
      i}^2}{n}  \partial_{\omega} g_{\rm out , \mu}^t \hat{W}_{i}^t \Big) \hspace{0.5cm}|\hspace{0.5cm}  \Sigma_{i}^t = -\Big(\sum\limits_{\mu =1}^m \frac{X_{\mu
      i}^2}{n}  \partial_\omega g_{\rm out,\mu}^t \Big)^{-1} $\vspace{0.1cm}
    \STATE Update of the estimated marginals $\hat{W}_i \in \bbR^K$ and $\hat{C}_i \in \mathcal{S}_K^+$: \\
    \hspace{0.5cm}$\hat{W}_i^{t+1} = f_w( \Sigma_i^t , T_i^t )   \hspace{0.5cm}|\hspace{0.5cm}  \hat{C}_i^{t+1} = f_c( \Sigma_i^t , T_i^t )$\vspace{0.1cm}
    \STATE ${t} = {t} + 1$ 
    \UNTIL{Convergence on
    $\hat{W}$, $\hat{C}$.} 
    \STATE {\bfseries Output:}
    $\hat{W}$ and $\hat{C}$.
\end{algorithmic}
\end{algorithm}

\section{State evolution equations from AMP}
\label{sec:se}
In this section, $W^\star$ denotes the ground truth weights of the teacher, and we define the overlap parameters at time $t$, $m^t$, $\sigma^t$, $q^t$, $Q$ and  that respectively measure the correlation of the AMP estimator with the ground truth, its variance and the norms of student and teacher weights:
\begin{align*}
	\begin{cases}
	m^t \equiv \displaystyle \EE_{W^\star} \lim_{n\to \infty} \frac{1}{n} \sum_{i=1}^n \hat{W}^t_i(W^\star_i)^\intercal \,, \\
	q^t \equiv \displaystyle \EE_{W^\star} \lim_{n\to \infty} \frac{1}{n} \sum_{i=1}^n \hat{W}^t_i(\hat{W}^{t}_i)^\intercal \,,  
	\end{cases}
	\hspace{0.3cm} \textrm{ and } \hspace{0.3cm}
	\begin{cases}
	\sigma^t \equiv \displaystyle \EE_{W^\star} \lim_{n\to \infty} \frac{1}{n} \sum_{i=1}^n \hat{C}_{i}^{t} \,. \\
	Q \equiv \displaystyle \EE_{W^\star} \lim_{n\to \infty} \frac{1}{n} \sum_{i=1}^n W^\star_i(W^\star_i)^\intercal \,, \\ 
	\end{cases}
\end{align*}
The aim is to derive the asymptotic behavior of these overlap parameters, called state evolution. The idea is to compute the overlap distributions starting with the relaxed BP equations of eq.~\eqref{supp:relaxed_BP}.

\subsection{Messages distribution}
In order to get the asymptotic behavior of the overlap parameters, we need first to compute the distribution of $\Sigma_{\mu \to i}^t$ and $T_{\mu \to i}^t$. Besides, we recall that  in our model, the output has been generated by a teacher according to $Y_\mu = \varphi_{\rm out}^0 \(\frac{1}{\sqrt{n}} W^\star X_\mu , A\)  $. We define $z_\mu \equiv \frac{1}{\sqrt{n}} W^\star X_\mu = \frac{1}{\sqrt{n}} \sum_{i=1}^n X_{\mu i} W_i^\star $ and $z_{\mu \to i} \equiv  \frac{1}{\sqrt{n}} \sum_{j\ne i}^n X_{\mu j} W_j^\star $. And it is useful to recall $\EE_X[X_{\mu i}] = 0 $ and $\EE_X[X_{\mu i}^2] = 1$.

\subparagraph{$\bullet$ $\omega_{\mu \to i}^t$\\}
Under belief propagation assumption messages are independent. $\omega_{\mu \to i}^t$ is thus the sum of independent variables and follows a Gaussian distribution. Let's compute the first two moments, using expansions of the relaxed BP equations of eq.~\eqref{supp:relaxed_BP}:
\begin{align*}
	\EE_X \[ \omega_{\mu \to i}^t \] &=   \frac{1}{\sqrt{n}} \sum\limits_{j\neq i}^n \EE_X \[X_{\mu j}\]   \hat{W}_{j\to \mu}^t  = 0 \,, \\
	\EE_X \[ \omega_{\mu \to i}^t (\omega_{\mu \to i}^t)^\intercal \] &=  \frac{1}{n} \sum\limits_{j\neq i, k\ne i }^n \EE_X \[X_{\mu j} X_{\mu k} \]   \hat{W}_{j\to \mu}^t (\hat{W}_{k\to \mu}^t)^\intercal  = \sum\limits_{j\neq i }^n \EE_X \[X_{\mu j}^2 \]   \hat{W}_{j\to \mu} (\hat{W}_{j\to \mu})^\intercal \\
		&= \frac{1}{n} \sum\limits_{j\neq i }^n    \hat{W}_{j\to \mu}^t (\hat{W}_{j\to \mu}^t)^\intercal 
		= \frac{1}{n} \sum\limits_{i=1 }^n    \hat{W}_{i}^t (\hat{W}_{i}^t)^\intercal + \mO\(1/n^{3/2} \) \underlim{n}{\infty} q^t \,.
\end{align*}	

\subparagraph{$\bullet$ $z_{\mu}$}
\begin{align*}
	\EE_X \[z_\mu  \] &= \frac{1}{\sqrt{n}} \sum_{i=1}^n \EE_X\[X_{\mu i}\] W_i^\star  = 0 \,,\\
	\EE_{X,W^\star} \[ z_\mu z_\mu^\intercal \] &=   \EE_{W^\star} \frac{1}{n} \sum\limits_{j =1, k=1 }^n \EE_X \[X_{\mu j} X_{\mu k} \]   W_j^\star (W_k^\star)^\intercal = \EE_{W^\star} \frac{1}{n} \sum\limits_{i =1 }^n   W_i^\star (W_i^\star)^\intercal \underlim{n}{\infty} Q \,.
\end{align*}

\subparagraph{$\bullet$ $z_{\mu}$ and $\omega_{\mu \to i}^t$}

\begin{align*}
	\EE_{X,W^\star} \[ \omega_{\mu \to i}^t z_\mu^\intercal \] &= \EE_{W^\star} \frac{1}{n} \sum\limits_{j\neq i, k=1 }^n \EE_X \[X_{\mu j} X_{\mu k} \]   \hat{W}_{j\to \mu}^t (W_{k}^\star)^\intercal  = \EE_{W^\star} \frac{1}{n} \sum\limits_{j\neq i}^n    \hat{W}_{j\to \mu}^t (W_{j}^\star)^\intercal \\
	&= \EE_{W^\star} \frac{1}{n} \sum\limits_{i=1}^n  \hat{W}_{i}^t (W_{i}^\star)^\intercal + \mO\(1/n^{3/2} \) \underlim{n}{\infty} m^t \,.
\end{align*}

Hence, asymptotically ($z_\mu$, $\omega_{\mu \to i}^t$) follows a Gaussian distribution with covariance matrix
$	\tbf{Q}^t=\begin{bmatrix}
    Q & m^t \\
    m^t & q^t  \\
  \end{bmatrix} $.
 
\subparagraph{$\bullet$ $V_{\mu \to i}$} concentrates around its mean:
\begin{align*}
	\EE_{X,W^\star} \[ V_{ \mu \to i}^t \]&= \EE_{W^\star}  \frac{1}{n} \sum\limits_{j\neq i}^n\EE_X \[X_{\mu j}^2\] \hat{C}_{j\to \mu}^t = \EE_{W^\star}  \frac{1}{n} \sum\limits_{j\neq i}^n \hat{C}_{j\to \mu}^t  = \EE_{W^\star}  \frac{1}{n} \sum\limits_{i}^n \hat{C}_{i}^t  + \mO\(1/n^{3/2} \) \underlim{n}{\infty} \sigma^t \,.
\end{align*}

Let's define other order parameters, that will appear in the following:
\begin{align*}
	\begin{cases}
		\hat{q}^t & \equiv \alpha \EE_{\omega,z,A} \[g_{\rm out}(\omega,\varphi^0_{\rm out}(z,A), \sigma^t )g_{\rm out}(\omega,\varphi^0_{\rm out}(z,A), \sigma^t )^\intercal \]  \,, \\
		 \hat{m}^t &\equiv \alpha \EE_{\omega,z,A} \[\partial_z g_{\rm out}(\omega,\varphi^0_{\rm out}(z,A), \sigma^t )\]  \,, \\
		  \hat{\chi}^t &\equiv \alpha \EE_{\omega,z,A} \[ - \partial_\omega g_{\rm out}(\omega,\varphi^0_{\rm out}(z,A), \sigma^t )\] \,. \\
	\end{cases}
\end{align*}

\subparagraph{ $\bullet$ $T_{\mu \to i}^t$ } can be expanded around $z_{\mu \to i}$:
\begin{align*}
	&\(\Sigma_{\mu \to i}^t\)^{-1} T_{\mu \to i}^t = \left( \sum\limits_{\nu \ne \mu}^m  B_{\nu \to i}^t \) = \left( \sum\limits_{\nu \ne \mu}^m  \frac{1}{\sqrt{n}} X_{\nu i} g_{\rm out} (\omega_{\nu \to i}^t, \varphi_{\rm out}^0 \(\frac{1}{\sqrt{n}} \sum_{j \ne i}^n X_{\nu j} W_j^\star  + X_{\nu i} W_i^\star  , A\), V_{\nu \to i}^t) \)\\
	&= \left( \sum\limits_{\nu \ne \mu}^m  \frac{1}{\sqrt{n}} X_{\nu i} g_{\rm out} (\omega_{\nu \to i}^t, \varphi_{\rm out}^0 \(z_{\nu \to i}, A\), V_{\nu \to i}^t) \) + \left( \sum\limits_{\nu \ne \mu}^m  \frac{1}{n} X_{\nu i}^2 \partial_z g_{\rm out} (\omega_{\nu \to i}^t, \varphi_{\rm out}^0 \(z_{\nu \to i}, A\), V_{\nu \to i}^t) \) W_i^\star\,.  \\
\end{align*}

\subparagraph{ $\bullet$ $\Sigma_{\mu \to i}^t$ }
\begin{align*}
	\(\Sigma_{\mu \to i}^t\)^{-1} &= \sum \limits_{\nu \ne \mu}^m  A_{\nu \to i}^t = - \sum \limits_{\nu \ne \mu}^m   \frac{1}{n} X_{\nu i}^2  \partial_\omega g_{\rm out}(\omega_{\nu \to i}^t, Y_{\nu}, V_{\nu \to i}^t) \\
	&= - \sum \limits_{\nu \ne \mu}^m   \frac{1}{n} X_{\nu i}^2  \partial_\omega g_{\rm out}(\omega_{\nu \to i}^t, \varphi_{\rm out}^0 \(z_{\nu \to i},A\), V_{\nu \to i}^t) + \mO\(1/n^{3/2} \) \,.
\end{align*}

Hence, taking the average and the large size limit, the first moments of the variables $\Sigma_{\mu \to i}^t$ and $T_{\mu \to i}^t$ read:
\begin{align*}
\begin{cases}
	&\EE_{\omega, z , A, X} \[ \(\Sigma_{\mu \to i}^t\)^{-1} T_{\mu \to i}^t \] \underlim{n}{\infty} \hat{m}^tW_i^\star  \,, \\
	&\EE_{\omega, z , A, X} \[ \(\Sigma_{\mu \to i}^t\)^{-1} T_{\mu \to i}^t  \(T_{\mu \to i}^t\)^\intercal \(\Sigma_{\mu \to i}^t\)^{-1}  \] \underlim{n}{\infty}  \hat{q}^t \,, \\
	&\EE_{\omega, z , A, X} \[ \(\Sigma_{\mu \to i}^t\)^{-1}  \] \underlim{n}{\infty} \hat{\chi}^t  \,.
\end{cases}
\end{align*}

And finally $ T_{\mu \to i}^t \sim (\hat{\chi}^t)^{-1} \( \hat{m}^tW_i^\star  + ( \hat{q}^t)^{1/2}\xi \) $ with $\xi \sim \mathcal{N}(0,\id)$ and $ \(\Sigma_{\mu \to i}^t\)^{-1} \sim (\hat{\chi}^t)^{-1}$ \,.

\subsection{State evolution equations - Non Bayes optimal case}
Let's define the following notations:
\begin{align*}
	T^t[W^\star, \xi ] &\equiv (\hat{\chi}^t)^{-1} \(\hat{m}^t W^\star + (\hat{q}^t)^{1/2} \xi   \) \\
	\Sigma^t &\equiv ( \hat{\chi}^t)^{-1} 
\end{align*}
Gathering above results, the state evolution equations read: 
\begin{align*}
\begin{cases}
	&m^{t+1} = \displaystyle \EE_{W^\star} \lim_{n\to \infty} \frac{1}{n} \sum_{i=1}^n \hat{W}^t_i(W^\star_i)^\intercal =  \EE_{W^\star, \xi } \[ f_w\( \Sigma^t,T^t[W^\star, \xi ] \) \(W^\star \)^\intercal\]  \\
	&q^{t+1} = \displaystyle \EE_{W^\star} \lim_{n\to \infty}  \frac{1}{n} \sum_{i=1}^n \hat{W}^{t+1}_i(\hat{W}^{t+1}_i)^\intercal = \EE_{W^\star, \xi }  \[ f_w\( \Sigma^t,T^t[W^\star, \xi ] \) f_w\( \Sigma^t,T^t[W^\star, \xi ] \)^\intercal \]  \\
	&\sigma^{t+1} = \displaystyle \EE_{W^\star} \lim_{n\to \infty}  \frac{1}{n} \sum_{i=1}^n \hat{C}^{t+1}_i = \EE_{W^\star, \xi }  \[ f_c\( \Sigma^t,T^t[W^\star, \xi ] \)\]  \\
\end{cases}
\end{align*}
and 
\begin{align*}
	\begin{cases}
		\hat{q}^t &= \alpha \EE_{\omega,z,A} \[g_{\rm out}(\omega,\varphi^0_{\rm out}(z,A), \sigma^t )g_{\rm out}(\omega,\varphi^0_{\rm out}(z,A), \sigma^t )^\intercal \]  \\
		&= \alpha \displaystyle \int dP_A(A) \int dz d\omega \mathcal{N} \( z , \omega ; 0 , \tbf{Q}^t   \)  g_{\rm out}(\omega,\varphi^0_{\rm out}(z,A), \sigma^t )  g_{\rm out}(\omega,\varphi^0_{\rm out}(z,A), \sigma^t )^\intercal \\   
		 \hat{m}^t &= \alpha \EE_{\omega,z,A} \[\partial_z g_{\rm out}(\omega,\varphi^0_{\rm out}(z,A), \sigma^t )\]  \\
		 &= \alpha \displaystyle \int dP_A(A) \int dz d\omega \mathcal{N} \( z , \omega ; 0 , \tbf{Q}^t   \)  \partial_z g_{\rm out}(\omega,\varphi^0_{\rm out}(z,A), \sigma^t ) \\
		 \hat{\chi}^t &= \alpha \EE_{\omega,z,A} \[ - \partial_\omega g_{\rm out}(\omega,\varphi^0_{\rm out}(z,A), \sigma^t )\]  \\
		 &= - \alpha \displaystyle \int dP_A(A) \int dz d\omega \mathcal{N} \( z , \omega ; 0 , \tbf{Q}^t   \)   \partial_\omega g_{\rm out}(\omega,\varphi^0_{\rm out}(z,A), \sigma^t )\\
	\end{cases}
\end{align*}

\subsection{State evolution equations -- Bayes-optimal case}
In the Bayes-optimal case, the student knows all the parameters of the teacher and then $P_0^\star = P_0$, $\varphi^0_{\rm out} = \varphi_{\rm out}$, $m^t = q^t$ and $ \hat{q}^t  = \hat{m}^t = \hat{\chi}^t $, $\sigma^{t} = Q - q^t  $ and then, naturally 
\begin{align*}
	T^t[W^\star, \xi ] &\equiv W^\star + (\hat{q}^t)^{-1/2} \xi  \,,  \\
	\Sigma^t &\equiv ( \hat{q}^t)^{-1} \, .
\end{align*}

In the Bayes-optimal case, the set of state evolution equations reduces and simplifies to:
\begin{align}
\begin{cases}
	&q^{t+1}  =  \EE_{W^\star, \xi }  \[ f_w\( \Sigma^t,T^t[W^\star, \xi ] \) f_w\( \Sigma^t,T^t[W^\star, \xi ] \)^\intercal \]  \,,\vspace{0.3cm} \\
	&\hat{q}^t = \alpha \EE_{\omega,z,A} \[g_{\rm out}(\omega,\varphi_{\rm out}(z,A), \sigma^t )g_{\rm out}(\omega,\varphi_{\rm out}(z,A), \sigma^t )^\intercal \] \,,
\end{cases}
\label{appendix:se_amp_bayes}
\end{align}
where $(z,\omega) \sim \mathcal{N}_{z,\omega}\(0,0;\tbf{Q}^t\)$ with $	\tbf{Q}^t=\begin{bmatrix}
    Q & q^t \\
    q^t & q^t  \\
  \end{bmatrix} $.

\subsection{State evolution - Consistence between replicas and AMP - Bayes optimal case}

\paragraph{State evolution - AMP\\}
Using the change of variable $\xi \leftarrow \xi + \(\hat{q}^t\)^{1/2} W^\star$, eq.~\eqref{appendix:se_amp_bayes} becomes:
\begin{align*}
q^{t+1}  \displaystyle = \EE_{\xi} \[ \mathcal{Z}_{P_{\rm 0}} \((\hat{q}^t)^{1/2}\xi,(\hat{q}^t)^{-1}\) f_w\((\hat{q}^t)^{1/2}\xi, (\hat{q}^t)^{-1}\) f_w\((\hat{q}^t)^{1/2}\xi, (\hat{q}^t)^{-1}\)^\intercal \]	
\end{align*}

In addition, in the Bayes-optimal case, as: 
\begin{align*}
	\begin{cases}
		\EE_X \[ \omega_{\mu \to i}^t (z_\mu-\omega_{\mu \to i}^t)^\intercal \] = m^t - q^t = 0 \vspace{0.2cm} \\
		\EE_X[\omega_{\mu \to i}^t (\omega_{\mu \to i}^t)^\intercal] = q^t  \vspace{0.2cm} \\
		\EE_X\[(z_\mu^\intercal-\omega_{\mu \to i}^t)(z_\mu-\omega_{\mu \to i}^t)^\intercal\] = Q  - q^t \, ,  \vspace{0.2cm} \\
	\end{cases}
\end{align*}
the multivariate distribution can be written as a product: $\mathcal{N}_{z,\omega}\(0,0;\tbf{Q}^t\) = \mathcal{N}_\omega\(0,q^t\) \mathcal{N}_{z}\(\omega,Q-q^t\)$.
Hence, using $P_{\rm out}(y|z) = \int dP_A(A) \delta\( y - \varphi^0_{\rm out}(z,A)\) $,  eq.~\eqref{appendix:se_amp_bayes} becomes:
\begin{align*}
	\hat{q}^{t}  \displaystyle &= \alpha \EE_{\omega,z,A} \[g_{\rm out}(\omega,\varphi^0_{\rm out}(z,A), Q - q^t )g_{\rm out}(\omega,\varphi^0_{\rm out}(z,A), Q - q^t )^\intercal \] \vspace{0.2cm} \\
	&= \alpha \int dy \int d \omega \frac{e^{-\frac{1}{2} \omega^\intercal (q^t)^{-1} \omega  }}{(2\pi)^{K/2} \det(q^t)^{1/2}} \int dz P_{\rm out}(y|z) \frac{e^{-\frac{1}{2}(z-\omega)^\intercal (Q-q^t)^{-1} (z-\omega) }}{(2\pi)^{K/2} \det (Q-q^t)^{1/2} } g_{\rm out}(\omega,y, Q - q^t )g_{\rm out}(\omega,y, Q - q^t )^\intercal \vspace{0.2cm} \\
	&= \alpha \int dy \int D\xi \int dz P_{\rm out}(y|z) \frac{e^{-\frac{1}{2}(z-\omega)^\intercal (Q-q^t)^{-1} (z-\omega) }}{(2\pi)^{K/2} \det (Q-q^t)^{1/2} } g_{\rm out}((q^t)^{1/2}\xi,y, Q - q^t )g_{\rm out}((q^t)^{1/2}\xi,y, Q - q^t )^\intercal \vspace{0.2cm} \\
	&= \alpha \EE_{y,\xi} \[ \mathcal{Z}_{P_{\rm out}}\((q^t)^{1/2}\xi,y, Q - q^t \) g_{\rm out}\((q^t)^{1/2}\xi,y, Q - q^t \) g_{\rm out}\((q^t)^{1/2}\xi,y, Q - q^t \)^\intercal \]
\end{align*}

To summarize, the state evolution equations can be written as:
\begin{align}
	\begin{cases}
	q^{t+1}  \displaystyle = \EE_{\xi} \[ \mathcal{Z}_{P_{\rm 0}} \((\hat{q}^t)^{1/2}\xi,(\hat{q}^t)^{-1}\) f_w\((\hat{q}^t)^{1/2}\xi, (\hat{q}^t)^{-1}\) f_w\((\hat{q}^t)^{1/2}\xi, (\hat{q}^t)^{-1}\)^\intercal \]		\vspace{0.2cm} \\
		\hat{q}^{t} = \alpha \EE_{y,\xi} \[ \mathcal{Z}_{P_{\rm out}}\((q^t)^{1/2}\xi,y, Q - q^t \) g_{\rm out}\((q^t)^{1/2}\xi,y, Q - q^t \) g_{\rm out}\((q^t)^{1/2}\xi,y, Q - q^t \)^\intercal \]
	\end{cases}
	\label{appendix:se_amp}
\end{align}

\paragraph{State evolution - Replicas\\}
Recall from sec.~\ref{sec:replicacomputation}, the free entropy eq.~\eqref{eq:replica_solution} reads
\begin{align*}
\begin{cases}
\lim_{n \to \infty} f_n &= \text{extr}_{q,\hat{q}}\left\{- \frac{1}{2} \text{Tr} [q \hat{q}] + I_P + \alpha I_C \right\} \,, \vspace{0.2cm} \\
	I_P &\equiv \EE_{\xi} \[ \mathcal{Z}_{P_{\rm 0}}(\hat{q}^{1/2}\xi, \hat{q}^{-1}) \log(\mathcal{Z}_{P_{\rm 0}}(\hat{q}^{1/2}\xi, \hat{q}^{-1})) \]\,, \vspace{0.2cm} \\
	I_C &\equiv \EE_{\xi,y} \[ \mathcal{Z}_{P_{\rm out}}(q^{1/2}\xi,y, Q - q ) \log(\mathcal{Z}_{P_{\rm out}}(q^{1/2}\xi,y, Q - q )) \] \,.
\end{cases}
\end{align*}  

Taking the derivatives with respect to $q$ and $\hat{q}$, using an integration by part and the following identities:
\begin{align*}
\begin{cases}
	\frac{\partial \mathcal{Z}_{P_{\rm out}}}{\partial q } = -\frac{1}{2} q^{-1}  e^{\frac{1}{2} \xi^\intercal \xi }  \partial_{\xi} 
	\[  e^{-\frac{1}{2} \xi^\intercal \xi } \partial_{\xi} \mathcal{Z}_{P_{\rm out}}  \]	\,, \vspace{0.2cm} \\
	\frac{\partial \mathcal{Z}_{P_{\rm 0}}}{\partial \hat{q} } = -\frac{1}{2}\hat{q}^{-1} e^{\frac{1}{2} \xi^\intercal \xi }   \partial_{\xi} 
	\[  e^{-\frac{1}{2} \xi^\intercal \xi } \partial_{\xi} \mathcal{Z}_{P_{\rm 0}}  \]\,,
\end{cases}
\end{align*}
the state evolution equations read:
\begin{equation*}
	\begin{cases}
	q = 2 \frac{\partial I_P }{\partial \hat{q}} \vspace{0.2cm}\\
	\hat{q} = 2 \alpha \frac{\partial I_C }{\partial q} 
	\end{cases}
	\hspace{0.3cm} \textrm{ with } \hspace{0.3cm}
		\begin{cases}
		\frac{\partial I_P }{\partial \hat{q}} = \frac{1}{2} \EE_{\xi} \[ \mathcal{Z}_{P_{\rm 0}}(\hat{q}^{1/2}\xi, \hat{q}^{-1}) f_w(\hat{q}^{1/2}\xi, \hat{q}) f_w(\hat{q}^{1/2}\xi, \hat{q}^{-1})^\intercal \]  \vspace{0.2cm} \\
		\frac{\partial I_C }{\partial q} = \frac{1}{2} \EE_{y,\xi} \[ \mathcal{Z}_{P_{\rm out}}(q^{1/2}\xi,y, Q - q ) g_{\rm out}(q^{1/2}\xi,y, Q - q ) g_{\rm out}(q^{1/2}\xi,y, Q - q )^\intercal \]
	\end{cases}
\end{equation*}
that simplifies and allows to recover the state evolutions equations directly derived from AMP eq.~\eqref{appendix:se_amp}, but without time indices
\begin{align*}
	\begin{cases}
		q =  \EE_{\xi} \[ \mathcal{Z}_{P_{\rm 0}}(\hat{q}^{1/2}\xi, \hat{q}^{-1}) f_w(\hat{q}^{1/2}\xi, \hat{q}) f_w(\hat{q}^{1/2}\xi, \hat{q}^{-1})^\intercal \] \,, \vspace{0.2cm} \\
		\hat{q} =  \alpha \EE_{y,\xi} \[ \mathcal{Z}_{P_{\rm out}}(q^{1/2}\xi,y, Q - q ) g_{\rm out}(q^{1/2}\xi,y, Q - q ) g_{\rm out}(q^{1/2}\xi,y, Q - q )^\intercal \] \,.
	\end{cases}
\end{align*}

%% file: sections/supplementary/parity_machine.tex
\section{Parity machine for \texorpdfstring{$K=2$}{K=2}}
\label{sec:parity}
Although we mainly focused on the committee machine, another classical two-layers neural network is the parity machine \cite{engel2001statistical} and our proof applies to this case as well.  While learning is known to be computationally hard for general $K$, the case $K=2$ is special, and in fact can be reformulated as a committee machine, where the sign activation function has been replaced by $\varphi_1 (z) = \id (z\ne 0) - \id (z = 0)$:  \begin{equation} Y_\mu = {\rm{\sign}}\Big[\prod_{l=1}^K {\sign} \Big( \sum_{i=1}^nX_{\mu i} W_{i l}^* \Big) \Big]\,= {\rm{\varphi_1}}\Big[\sum_{l=1}^K {\sign} \Big( \sum_{i=1}^nX_{\mu i} W_{i l}^* \Big) \Big]\, .  \label{} \end{equation}

We have repeated our analysis for the $K=2$ parity machine and the phase diagram is summarized in Fig.~\ref{fig:phaseDiagramK2_parity} where we show the generalization error and the elements of the overlap matrix for Gaussian (left) and binary weights (right), with the results of the AMP algorithm (points). 

Below the specialization phase transition $\alpha < \alpha_{\rm spec}$, the symmetry of the output imposes the non-specialized fixed point $q_{00}=q_{01}=0$ to be the only solution, with $\alpha_{\rm spec}^G(K=2) \simeq 2.48 $ and $\alpha_{\rm spec}^B(K=2)\simeq 2.49$. Above the specialization transition $\alpha_{\rm spec}$, the overlap becomes specialized with a non-trivial diagonal term. 

Additionally, in the binary case, an information theoretical transition towards a perfect learning occurs at $\alpha_{\rm IT}^B(K=2)\simeq 2.00$, meaning that the perfect generalization fixed point ($q_{00}=1,q_{01}=0$) becomes the global optimizer of the free entropy. It leads to a first order phase transition of the AMP algorithm which retrieves the perfect generalization phase only at $\alpha_{\rm perf}^B(K=2)\simeq 3.03$. This is similar to what happens in single layer neural networks for the symmetric door activation function, see \cite{barbier2017phase}. Again, these results for the parity machine emphasize a gap between information-theoretical and computational performance.

\begin{figure}[!!!t]
\center
\includegraphics[width=0.9\linewidth]{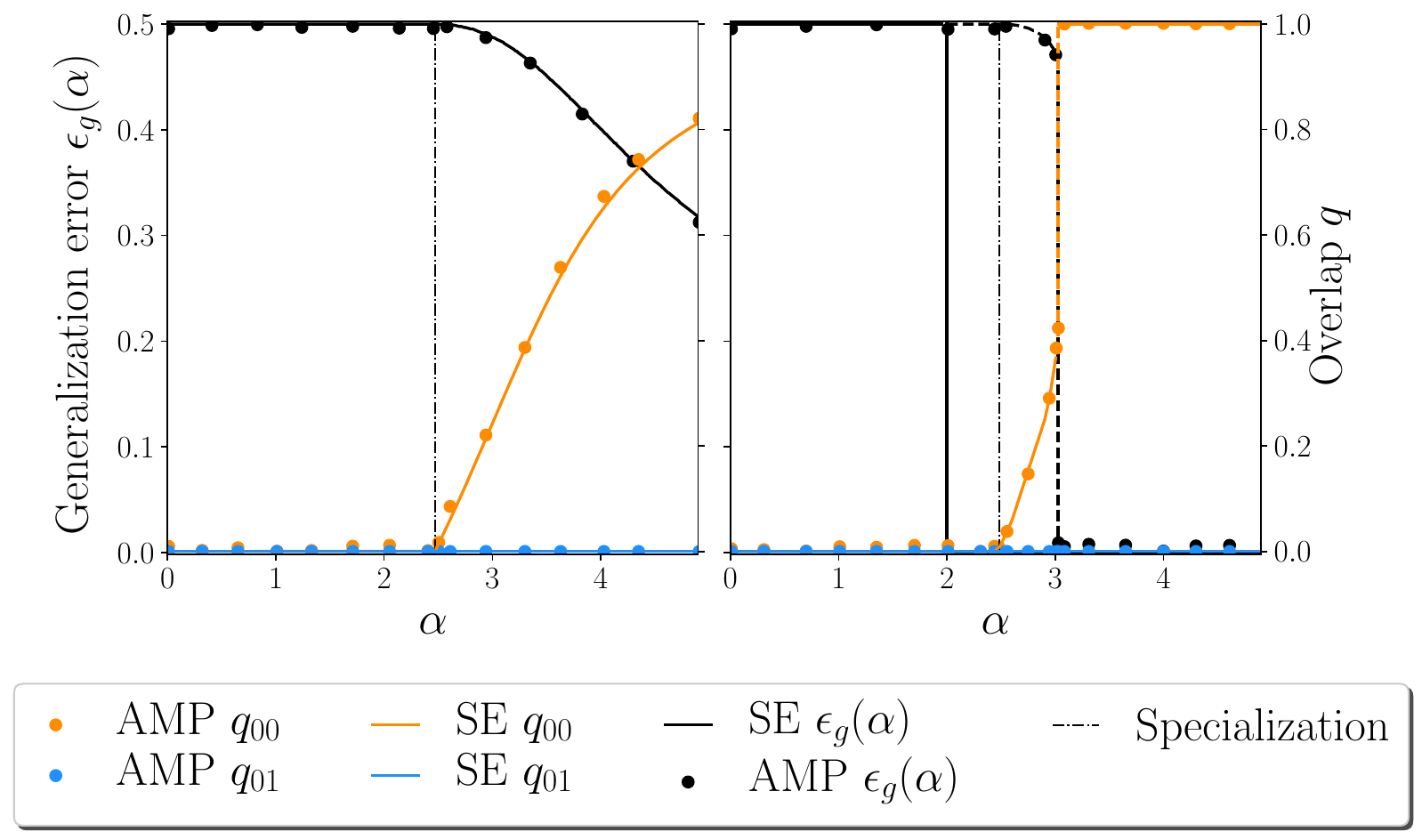}
 \caption{Similar plot as in Fig. \ref{fig:phaseDiagramK2} but for the parity machine with two hidden neurons. Value of the order parameter and the optimal generalization error for a parity
   machine with two hidden neurons with Gaussian weights (left) and
   binary/Rademacher weights (right).  SE and AMP overlaps are respectively represented in full line and points.}
\label{fig:phaseDiagramK2_parity}
\vspace{-0.5cm}
\end{figure}